\documentclass{article}
\usepackage[letterpaper,bindingoffset=0.2in,%
left=0.5in,right=0.5in,top=1in,bottom=1in,%
footskip=.25in]{geometry}

\usepackage{microtype}
\usepackage{graphicx}
\usepackage{subfigure}
\usepackage{booktabs} % for professional tables

\usepackage{amsfonts}       % blackboard math symbols
\usepackage{amsmath,amsthm}

\DeclareMathOperator*{\argmin}{arg\,min}
\usepackage{algorithm}
\usepackage{algorithmic}
\usepackage{float}
\restylefloat{table}
\usepackage{caption}
\usepackage[T1]{fontenc}
\usepackage{tikz}
\usetikzlibrary{arrows}
\usetikzlibrary{calc}
\usetikzlibrary{snakes}
\usepackage{cite}
\usepackage{graphicx}
\usepackage{array}
\usepackage{mdwmath}
\usepackage{mdwtab}
\usepackage{eqparbox}
\usepackage{fixltx2e}
\usepackage{url}
\usepackage{pgf}
\usepackage{lipsum}
\usepackage{multirow}
\usepackage{xspace}
\usepackage{xcolor}
\usepackage{algorithmic}
\usepackage{bm}
\usepackage{bbm}
\usepackage{booktabs}
\usepackage{etoolbox}
\usepackage{makecell}
\usepackage{nicefrac}
\usepackage{textcomp}
\usepackage{stmaryrd}
\usepackage{mathrsfs}
\usepackage{paralist}
\usepackage{comment}

\usepackage[font=footnotesize,labelsep=space,labelfont=bf]{caption}

\usepackage{url}

\usepackage[pdftex,bookmarks=true,pdfstartview=FitH,colorlinks,linkcolor=blue,filecolor=blue,citecolor=blue,urlcolor=blue,pagebackref=false]{hyperref}
\makeatletter

\newcommand{\Rmnum}[1]{\expandafter\@slowromancap\romannumeral #1@}
\makeatother

\newtheorem{theorem}{Theorem}%[chapter]
\newtheorem*{theorem*}{Theorem}
\newtheorem{lemma}{Lemma}%[chapter]
\newtheorem*{lemma*}{Lemma}
\newtheorem{claim}{Claim}%[chapter]
\newtheorem{corollary}{Corollary}%[chapter]
\newtheorem*{cor*}{Corollary}
\newtheorem{remark}{Remark}%[chapter]
\newtheorem{fact}{Fact}%[chapter]

\newcommand{\TODO}[1]{{{\color{red}#1}}}

\newcommand{\namedref}[2]{\hyperref[#2]{#1~\ref*{#2}}}

\definecolor{darkred}{rgb}{0.5, 0, 0} 
\definecolor{darkblue}{rgb}{0,0,0.5} 
\hypersetup{
    colorlinks=true,
    linkcolor=darkred,
    citecolor=darkblue,
    urlcolor=darkblue   
}

\newcommand{\bc}{\ensuremath{{\bf c}}\xspace}
\newcommand{\bd}{\ensuremath{{\bf d}}\xspace}

\newcommand{\bg}{\ensuremath{{\bf g}}\xspace}
\newcommand{\bh}{\ensuremath{{\bf h}}\xspace}

\newcommand{\bx}{\ensuremath{{\bf x}}\xspace}

\newcommand{\by}{\ensuremath{{\bf y}}\xspace}

\newcommand{\bz}{\ensuremath{{\bf z}}\xspace}

\newcommand{\bw}{\ensuremath{{\bf w}}\xspace}

\newcommand{\bu}{\ensuremath{{\bf u}}\xspace}

\newcommand{\bA}{\ensuremath{{\bf A}}\xspace}

\newcommand{\bX}{\ensuremath{{\bf X}}\xspace}
\newcommand{\bG}{\ensuremath{{\bf G}}\xspace}

\newcommand{\R}{\ensuremath{\mathbb{R}}\xspace}

\renewcommand{\paragraph}[1]{\smallskip\noindent{\bf #1}~}

\newcommand{\calC}{\mathcal{C}}
\newcommand{\bbR}{\mathbb{R}}
\usepackage{stfloats}
\usepackage{threeparttable}
\usepackage{comment}
\usepackage{xcolor}

\usepackage{lmodern}
\newtheorem*{claim*}{Claim}
\newtheorem*{corollary*}{Corollary}
\usepackage{hyperref}
\usepackage{natbib}

\title{QuPeL: Quantized Personalization with Applications to Federated Learning}
\author{
	Kaan Ozkara
	\and
	Navjot Singh
	\and
	Deepesh Data
	\and
	Suhas Diggavi
}
\date{University of California, Los Angeles, USA\\
kaan@ucla.edu, navjotsingh@ucla.edu, deepesh.data@gmail.com, suhas@ee.ucla.edu
}
\begin{document}
	
\maketitle

%%------------------------------------------------------------------

%%------------------------------------------------------------------
%% ABSTRACT
\begin{abstract}
Traditionally, federated learning (FL) aims to train a single global model while collaboratively using multiple clients and a server. Two natural challenges that FL algorithms face are heterogeneity in data across clients and collaboration of clients with {\em diverse resources}. In this work, we introduce a \textit{quantized} and \textit{personalized} FL algorithm QuPeL that facilitates collective training with heterogeneous clients while respecting resource diversity. For personalization, we allow clients to learn \textit{compressed personalized models} with different quantization parameters depending on their resources. Towards this, first we propose an algorithm for learning quantized models through a relaxed optimization problem, where quantization values are also optimized over. When each client participating in the (federated) learning process has different requirements of the quantized model (both in value and precision), we formulate a quantized personalization framework by introducing a penalty term for local client objectives against a globally trained model to encourage collaboration. We develop an alternating proximal gradient update for solving this quantized personalization problem, and we analyze its convergence properties. Numerically, we show that optimizing over the quantization levels increases the performance and we validate that QuPeL outperforms both FedAvg and local training of clients in a heterogeneous setting.
\end{abstract}
%%------------------------------------------------------------------

%%------------------------------------------------------------------
%% INTRODUCTION
\section{Introduction} \label{sec:intro}

Federated Learning (FL) is a learning procedure where the aim is to utilize vast amount of data residing in numerous (in millions) edge devices (clients) to train machine learning models without collecting clients' data \cite{mcmahan2017communicationefficient}. Formally, if there are $n$ clients and $f_i$ denotes the local loss function at client $i$, then traditional FL learns a single global model by minimizing the following objective: % function: 
\begin{align} \label{fed1}
	\argmin_{\bw\in\R^d} \Big(f(\bw) := \frac{1}{n} \sum_{i=1}^n f_i(\bw)\Big).
\end{align}
It has been realized lately that a {\em single} model cannot provide good performance to all the clients as they have {\em heterogeneous} data.
This leads to the need for personalized learning, where each client wants to learn a personalized model  \cite{fallah2020personalized,dinh2020personalized}. Since each client may not have sufficient data to learn a good personalized model by itself, in the training process for personalized FL, clients maintain personalized models locally and utilize other clients' data via collaborating through a global model to improve their local models.
As far as we know, the personalized learning in all previous works does not respect {\em resource diversity} of clients, which is inherent to FL as the participating edge devices may vary widely in terms of resources.
%Besides being numerous in quantity, edge devices can hugely vary in terms of resources. 
For instance, there might be scenarios where mobile phones, tablets, laptops all collaborate in a federated setting. 
In such scenarios, where edge devices are constrained to use limited resources, model compression can become critical as it would let devices with resource constraints utilize complex models. % without huge memory cost.

%In federated structures where edge devices are constrained to use low resources, model compression can become critical as it would let devices with resource constraints utilize complex models without huge memory cost. Conventionally, FL setting acknowledges the heterogeneous data distribution among the clients; but it does not explicitly consider the heterogeneity of clients' resources. However, as we mentioned it is highly possible that different devices have access to different resources. To solve the first type of heterogeneity (heterogeneity of data) recent works, unlike the conventional FL, started to use two types of models: personalized and global. In Personalized FL setting clients are allowed to maintain personalized models locally, while still enjoying the collaboration with other clients through a global model. To address the two types of heterogeneity in a single procedure we propose a novel approach. 

In this work, we propose a model compression framework\footnote{Model compression is a process that allows inference time deployment of a model while compressing its size. Though model compression is a general term comprising different methods, we will focus on the quantization aspect of it.} for personalized federated learning that addresses the both aforementioned types of heterogeneity (in data and resources) in a unified manner. 
Our framework lets collaboration among clients (that have different resource requirements in terms of precision of model parameters) through a full precision global model for learning quantized personalized models. 
Note that while training and communication are in full precision, our goal is to obtain compressed models for inference at each client.\footnote{One can combine gradient compression methods for communication efficiency \cite{basu2019qsparse,alistarh2017} with the methods of this paper as they are complementary.}
To achieve this goal in an efficient way, we learn the compression parameters for each client by including quantization levels in the optimization problem itself and carefully formulating the objective.
First we investigate our approach in the {\em centralized} setup, by formulating a relaxed optimization problem and minimizing that through alternating proximal gradient steps, inspired by \cite{Bolte13}.
To extend this to a distributed setup, note that performing multiple local iterations and then synchronizing the local updates with the server is usually performed in FL settings for communication efficiency \cite{kairouz2019advances}. So the idea is to employ our centralized model quantization algorithm locally at clients for updating their personalized models in between synchronization indices. To aid collaboration, we introduce a term in the objective that
penalizes the deviation between personalized and global models, inspired by \cite{dinh2020personalized,hanzely2020federated}.
%Different from other works, we integrate the update for global models into our alternating proximal gradient scheme.
%At the core of this lies the analysis of our method for quantized neural network training in the {\em centralized} setup, and for that we extend the method in \cite{bai2018proxquant} with the modified optimization problem and then minimize the problem through alternating proximal gradient steps that is inspired by \cite{Bolte13}. 

%To introduce personalization, we set up a personalized FL scheme by augmenting our loss function via a term that penalizes the deviation between personalized and global models. Similar approaches were proposed in \cite{dinh2020personalized,hanzely2020federated}. Different from other works, we integrate the update for global models into our alternating proximal gradient scheme.

%\textbf{Our contributions.}
\subsection{Our Contributions}
Our contributions can be summarized as follows: %\\\vspace{-0.55cm}
\begin{itemize}
\item To learn compressed models, we propose a novel relaxed optimization problem that enables optimization over quantization values (centers) as well as the model parameters. We use alternating proximal updates to minimize the objective and analyze its convergence properties. During training we learn both the quantization levels as well as the assignments of the model parameters to those quantization levels, while recovering the convergence rate of $\mathcal{O}(\frac{1}{T})$ in \cite{bai2018proxquant} and \cite{Bolte13}. %\\\vspace{-0.55cm}
%Although the training is in full precision, in the end of training we arrive at a model that can be compressed without significant loss. \\\vspace{-0.55cm}
\item	More importantly, we propose a quantized personalized federated learning scheme that allows different clients to learn models with different quantization precision and values. Besides clients' personalized models, this provides an additional notion of personalization, namely, personalized model compression. We analyze convergence properties, and observe the common phenomenon (in personalized federated learning) that convergence rate depends on an error term related to dissimilarity of global and local gradients. %\\\vspace{-0.55cm}
\item We empirically show that optimizing over centers increases test performance and our personalization scheme outperforms FedAvg and Local Training at individual clients in a heterogeneous setting. We further observe an interesting phenomenon that clients with limited resources have increased performance when they collaborate with resource-abundant clients compared to when they collaborate with only resource-limited clients. 
\end{itemize}

Our work should not be confused with works in distributed/federated learning, where models/gradients are compressed for {\em communication efficiency} \cite{basu2019qsparse,karimireddy2019error} and not to learn compressed/quantized models for inference. On the contrary, our goal is to obtain quantized models for inference, that are suited for each client's resources. Note that we also achieve communication efficiency, but through local iterations, not through gradient/model compression.

\subsection{Related Work} 
%\textbf{Related work.} 
To the best of our knowledge, this is the first work in personalized federated learning where the aim is to learn quantized and personalized models for inference.\footnote{By `learning quantized models' we refer to the fact that compressed models are used during inference time; training may or may not be with quantized models.} As a result, our work can be seen in the intersection of {\em personalized federated learning} and {\em learning quantized models}.

\iffalse
\textit{Personalized federated learning.}
As mentioned before, personalized FL is used in heterogeneous scenarios when a single global model falls short for the learning task. % In such cases, each client would locally retain a personalized model throughout the training phase and collaborate with other clients through a global model. 
Recent work adopted different approaches for learning personalized models: 
{\sf(i)} Combine global and local models throughout the training \cite{deng2020adaptive,mansour2020approaches,hanzely2020federated};
{\sf (ii)} first learn a global model and then personalize it locally by updating it using clients' local data \cite{fallah2020personalized}; 
{\sf (iii)} consider multiple global models to collaborate among only those clients that share similar personalized models \cite{zhang2021personalized}; and
{\sf (iv)} augment the traditional FL  objective via a penalty term that enables collaboration between global and personalized models \cite{hanzely2020federated,hanzely2020lower,dinh2020personalized}.
\fi
\paragraph{Personalized federated learning.}
As mentioned before, personalized FL is used in heterogeneous scenarios when a single global model falls short for the learning task.
Recent work adopted different approaches for learning personalized models: 
{\sf(i)} Combine global and local models throughout the training \cite{deng2020adaptive,mansour2020approaches,hanzely2020federated};
{\sf (ii)} first learn a global model and then personalize it locally by updating it using clients' local data \cite{fallah2020personalized}; 
{\sf (iii)} consider multiple global models to collaborate among only those clients that share similar personalized models \cite{zhang2021personalized,mansour2020approaches,ghosh2020efficient,smith2017federated}; 
{\sf (iv)} augment the traditional FL  objective via a penalty term that enables collaboration between global and personalized models \cite{hanzely2020federated,hanzely2020lower,dinh2020personalized}; and
{\sf (v)} distillation of global model to personalized local models \cite{lin2020ensemble}. 

%\cblue{In an another approach of clustered FL \cite{sattler2019clustered}, clients with similar data distributions collaborate to learn personalized models \citep{mansour2020approaches}.}
%This collaboration can be in different ways. One way is through combining the global and local model parameters to obtain individual personalized models \cite{deng2020adaptive,mansour2020approaches,hanzely2020federated}. 
%Zhang et al.~\cite{zhang2021personalized} considered multiple global models to collaborate among only the clients that share similar personalized models. Fallah et al.~\cite{fallah2020personalized} aimed to find personalized models for heterogeneous clients through performing few updates on clients' local data after learning the common model. %; they studied this problem in the Model-Agnostic Meta-Learning framework \cite{Finn_maml17}. 
%In \cite{hanzely2020federated,hanzely2020lower,dinh2020personalized}, authors augment the traditional FL  objective via a penalty term that enables collaboration between global and personalized models.
%\TODO{Not clear: In Clustered Federated Learning \cite{sattler2019clustered} approaches allow clients with similar data distributions to collaborate \citep{mansour2020approaches}.}

\paragraph{Learning quantized models.}
Training quantized neural networks has been a topic of great interest in the last few years, and extensive research resulted in training quantized networks with precision of as low as 1-bit without significant loss in performance; see \cite{Bin_survey,comprehensive_survey} for surveys. There are two main approaches in obtaining quantized models. Firstly, one can simply train a model and then do a post-training quantization that is agnostic to training procedure; see for example \cite{BannerNS19}. The downside of doing a post-training quantization is that the loss minimization problem for quantization is not related to the empirical loss function. Consequently, there is no guarantee that one will obtain a compressed model that has good performance. As opposed to post-training quantization, the aim of learning quantized models is to learn the quantization itself during the training \cite{courbariaux2016binarized,courbariaux2015binaryconnect}. 
There are two kinds of approaches for training quantized networks that are of our interest. The first approximates the hard quantization function by using a soft surrogate \cite{Yang_2019_CVPR, gong2019differentiable}, while the other one iteratively projects the model parameters onto the fixed set of centers \citep{bai2018proxquant,BinaryRelax}. Each approach has its own limitations; see Section~\ref{sec:problem} for a discussion.

While the initial focus in learning quantized networks was mainly on achieving good empirical performance, there exist some works that analyzed convergence properties. \cite{li2017training} gave the first convergence guarantees by analyzing the algorithm proposed in \cite{courbariaux2015binaryconnect} using convexity assumptions. % and proposing a stochastic rounding scheme which resulted in a better convergence guarantee. 
Later, \cite{BinaryRelax} showed convergence results for non-convex functions but under an orthogonality assumption between the quantized and unquantized weights. More recently, \cite{bai2018proxquant} gave a convergence result for a relaxed/regularized loss function using proximal gradient updates. Note that all of the above works were done in centralized setting.

\subsection{Paper Organization} 
In Section~\ref{sec:problem} we formulate the optimization problem to be minimized. In Sections~\ref{sec:centralized} and~\ref{sec:personalized}, we describe our algorithms along-with the main convergence results for the centralized and personalized settings, respectively, and give the proofs in Section~\ref{sec:proof_centralized} and Section~\ref{sec:proof_personalized}. Section~\ref{sec:experiments} provides numerical results. Omitted proofs/details are in appendices.

%%------------------------------------------------------------------

%%------------------------------------------------------------------
%% PROBLEM SETUP
\section{Problem Formulation} \label{sec:problem}
%We first state an ideal problem for model compressed learning, then formulate our relaxed optimization problem for model compression; and finally, introduce our model quantized and personalized objective function.
%\TODO{Write a preamble.}
As motivated in Section~\ref{sec:intro}, in FL settings where clients with heterogeneous data also have diverse resources, our goal in this paper is for clients to collaboratively learn personalized compressed models (with potentially different precision). To this end, below, first we state our final objective function that we will end up optimizing in this paper for learning personalized compressed models, and then in the rest of this section we will argue and motivate why this particular choice of the objective is appropriate for our purpose.

Recall from \eqref{fed1}, in the traditional FL setting, the local loss function at client $i$ is denoted by $f_i$. For our personalized quantized model training, we define the following augmented loss function to be minimized by client $i$:
\begin{align} 
	F_i(\bx_i,\bc_i,\bw) &:= f_i(\bx_i)+f_i(\widetilde{Q}_{\bc_i}(\bx_i))+\lambda R(\bx_i,\bc_i) + \frac{\lambda_p}{2} \|\bx_i - \bw \|^2.  \label{qpfl}
\end{align}
Here, $\bw \in \mathbb{R}^d$ denotes the global model, $\bx_i \in \mathbb{R}^d$ denotes the personalized model at client $i$, and $\bc_i \in \mathbb{R}^{m_i}$ denotes the model quantization centers at client $i$, where $m_i$ is the number of centers at client $i$, with $\log m_i$ representing the number of bits per parameter, which could be different for each client -- larger the $m_i$, higher the precision. Having different $\bc_i,m_i$'s introduces another layer of personalization, namely, personalization in the quantization itself. 
In \eqref{qpfl}, $\widetilde{Q}_{\bc_i}$ denotes the soft-quantization function with respect to (w.r.t) the fixed set of centers $\bc_i$, $R(\bx_i,\bc_i)$ denotes a distance/regularizer function that encourages quantization (e.g., the $\ell_1$-distance, $R(\bx,\bc)=\min\{ \frac{1}{2} \|\bz-\bx\|_1:z_i \in \{c_1,\cdots,c_m\}, \forall i \}$ ), and $\lambda_p$ is the hyper-parameter controlling the regularization. We will formally define the undefined quantities later in this section.
%\textbf{Notation} We will omit $i$ in $\widetilde{Q}^{(i)}_{\bc}(\bx)$ when it is clear from the context. For example, $\widetilde{Q}_{\bc_i}(\bx_i)$ already implies that the function belongs to client $i$.
Consequently, our main objective becomes
\begin{align} \label{per1}
	\min \Big(F\big(\{\bx_i\},\{\bc_i\},\bw\big) := \frac{1}{n}\sum_{i=1}^n F_i(\bx_i,\bc_i,\bw)\Big), %\sum_{i=1}^n F_i(\bx_i,\bc_i,\bw)
\end{align}
where minimization is taken over $\bx_i\in\R^d,\bc_i\in\R^{m_i}$ for $i\in[n]$, $\bw\in\R^d$, and $F_i(\bx_i,\bc_i,\bw)$ for $i\in[n]$ are from \eqref{qpfl}.

In formulating \eqref{qpfl}, the last term $\frac{\lambda_p}{2} \|\bx_i - \bw \|^2$ penalizes the deviation between the personalized model $\bx_i$ and the global model $\bw$, controlled by the hyper-parameter $\lambda_p$.
Similar augmentations have been proposed in personalized FL settings by \cite{dinh2020personalized} and \cite{hanzely2020federated}, and also in a non-personalized heterogeneous setting by Li et al.~\cite{li2020federated}.
So, it suffices to motivate how we came up with the first three terms in \eqref{qpfl}, which are in fact about a centralized setting because the function $f_i$ and the parameters involved, i.e., $\bx_i,\bc_i$, are local to client $i$.  In the rest of this section, we will motivate this and progressively arrive at this formulation in the centralized setup. Note that as a result of minimizing the relaxed optimization problem we obtain a set of $\bx_i, \bc_i$ for each client; obtained parameters $\bx_i$ are concentrated around $\bc_i's$ due to regularization. Consequently, $\bx_i$ can be quantized using $\bc_i$ without significant loss. Therefore, for inference time deployment what we do is to hard quantize $\bx_i's$ using the respective $\bc_i$.

%In the traditional federated learning setting we have the following objective function, where the goal is to learn a single global model $\bw\in\R^d$:
%\begin{align} \label{fed1}
%	\min_{\bw\in\R^d} \left(f(\bw) := \frac{1}{n} \sum_{i=1}^n f_i(\bw)\right).
%\end{align}
%In \eqref{fed1}, the $f_i$'s are the local loss functions at the clients.

\subsection{Model Compression in the Centralized Setup} 

%\subsection{Ideal Problem}
 %As mentioned in Section~\ref{sec:intro}, we 
% where $f$ is the training loss of the neural network, $\bx\in\mathbb{R}^d$ is a vector denoting model weights, and  $\bc\in\mathbb{R}^m$ is a vector denoting quantization levels (centers). 

Consider a setting where an objective function $f:\R^{d+m}\to\R$ (which could be a neural network loss function) is optimized over both the quantization centers $\bc\in\R^m$ and the assignment of model parameters (or weights) $\bx\in\R^d$ to those centers. There are two ways to approach this problem, either by explicitly putting a constraint that weights belong to the set of centers, or embedding the quantization function into the loss function itself.
\subsubsection{Two Approaches}
The first approach suggests the following optimization problem: 
%Then an ideal optimization problem would be:
%\begin{align} 
%\begin{array}{rl}\label{opt1}
%\displaystyle
minimize $f(\bx)$ (over $\bx\in\R^d,\bc\in\R^m$) subject to $x_j\in\{c_1,\hdots,c_m\}$ for all $j\in[d]$,
%\end{array}
%$\end{align}
where the constraint ensures that every component $x_j$ is in the set of centers.
This can be equivalently written in a more succinct form as:
\begin{align}\label{opt2}
%\begin{array}{rl}\label{opt2}
%\displaystyle
\min_{\bx,\bc}  f(\bx) + \delta_{\bc}(\bx),
%\end{array}
\end{align}
where $\delta_{\bc}$ denotes the indicator function, and for any $\bc\in\R^m,\bx\in\bbR^d$, define $\delta_{\bc}(\bx):=0$ if for every $j\in[d]$, $x_j =c$ for some $c\in\{c_1,\hdots,c_m\}$, otherwise, define $\delta_{\bc}(\bx):=\infty$.
\iffalse
This naturally suggests the following optimization problem: 
%Then an ideal optimization problem would be:
\begin{align} 
%\begin{array}{rl}\label{opt1}
%\displaystyle
\min_{\bx,\bc}\ &\ f(\bx) \notag \\ 
\text { subject to } &\ x_j \in \left\{ c_1, \dots, c_m  \right\}, \ \forall j\in \left\{ 1, \dots, d  \right\}, \label{opt1}
%\end{array}
\end{align}
where the constraint ensures that every component $x_j$ is in the set of centers.

The above problem defined in \eqref{opt1} is equivalent to
\begin{align}\label{opt2}
%\begin{array}{rl}\label{opt2}
%\displaystyle
\min_{\bx,\bc}  f(\bx) + \delta_{\bc}(\bx),
%\end{array}
\end{align}
where $\delta_{\bc}$ is the indicator function of the set of centers $\bc$ and is defined as follows: 
for any $\bc\in\R^m$ and $\bx\in\bbR^d$, define $\delta_{\bc}(\bx)=0$ if for every $j\in[d]$, $x_j =c$ for some $c\in\{c_1,\hdots,c_m\}$; and $\infty$, otherwise. 
\fi

The second approach suggests the following problem:
\begin{align}\label{opt3}
%\begin{array}{rl}\label{opt3}
\min_{\bx,\bc} \big( h(\bx,\bc):=f(Q(\bx,\bc)) \big),
%\end{array}
\end{align}
where %$Q(\bx,\bc)$ is a mapping from centers $\bc$ and full precision weights to quantized weights. In this work, we will consider 
$Q:\mathbb{R}^{d+m}\rightarrow \mathbb{R}^d$ is a quantization function that maps individual weights to the closest centers.
%$Q:\mathbb{R}^{d+m}\rightarrow \mathbb{R}^d$ as a scalar quantizer that maps individual weights to their closest centers. 
In other words, for $\bx\in\mathbb{R}^d, \bc\in\mathbb{R}^m$, we define $Q(\bx,\bc)_i := c_k$, where $k=\argmin_{j\in[m]}\{|x_i-c_j|\}$. From now on, for notational convenience, we will denote $Q(\bx,\bc)$ by $Q_{\bc}(\bx)$. % and assume that in  $Q_{\bc}(\bx)$, both $\bx$ and $\bc$ are inputs to the function, i.e. $Q_{\bc}(\bx):\mathbb{R}^{d+m}\rightarrow \mathbb{R}^d$. 

\iffalse
\subsection{Two Related Problems}

The optimization problem \eqref{opt1} can be closely related to the following two problems. 
\begin{enumerate}
\item The first problem is:
\begin{align}
\begin{array}{rl}\label{opt2}
\displaystyle
\min_{\bc,\bx} & f(\bx) + \delta_{\calC}(\bx),
\end{array}
\end{align}
where $\delta_{\calC}$ is the indicator function of the set $\calC$ and is defined as: for any $\bx\in\bbR^d$, define $\delta_{\calC}(\bx)=0$ if $x_j \in {\calC}, \ \forall j$; and $\infty$, otherwise. 
%\begin{align} \label{delta}
%\delta_C(x)=
%\begin{cases}
%	0, \quad \text{if } x \in C\\
%	\infty, \quad \text{otherwise}
%\end{cases}
%\end{align}
Here, ${\calC}$ denotes the set of centers. This problem is equivalent to \eqref{opt1}.
%
\item The other related problem is the following:
\begin{align}
\begin{array}{rl}\label{opt3}
\min_{\bc,\bx} & h(\bx,\bc):=f(Q(\bx,\bc)),
\end{array}
\end{align}
where $Q(\bx,\bc)$ is a mapping from centers and full precision weights to quantized weights. In this work, we will consider $Q:\mathbb{R}^{d+m}\rightarrow \mathbb{R}^d$ as a scalar quantizer that maps individual weights to their closest centers. In other words, for $\bx\in\mathbb{R}^d, \bc\in\mathbb{R}^m$, we define $Q(\bx,\bc)_i := \argmin_{c_j}\{|x_i-c_j|:c_j \in {\calC}\}$. From now on, we will denote $Q(\bx,\bc)$ as $Q_{\bc}(\bx)$ for notational simplicity and assume that in  $Q_{\bc}(\bx)$, both $\bx$ and $\bc$ are inputs to the function, i.e. $Q_{\bc}(\bx):\mathbb{R}^{d+m}\rightarrow \mathbb{R}^d$. 
\end{enumerate}
\fi

%\textbf{Challenges.}\label{cha}
\paragraph{Limitations of both the approaches.}
Both \eqref{opt2} and \eqref{opt3} have their own limitations:
%\begin{itemize}
%\item 
The discontinuity of $\delta_{\bc}(\bx)$ makes it challenging to minimize the objective function in \eqref{opt2}.
And the hard quantization function $Q_{\bc}(\bx)$ in \eqref{opt3} is actually a staircase function for which the derivative w.r.t.\ $\bx$ is 0 almost everywhere, which makes it impossible to use algorithms that rely on gradients. 
%\end{itemize}

%\subsection{Approximations}
%\textbf{Relaxations.}
\subsubsection{Relaxations}
Inspired by some recent works that addressed the aforementioned challenges for solving \eqref{opt2} and \eqref{opt3} (without optimizing over $\bc$), we propose some relaxations.

%Recently, some approximations were proposed to cope with the above-mentioned challenges. 
Instead of solving the optimization problem in \eqref{opt2} (without optimizing over $\bc$), \cite{bai2018proxquant} and \cite{BinaryRelax} proposed to approximate the indicator function $\delta_{\bc}(\bx)$ using a distance function $R_\bc(\bx)$, which is continuous everywhere. Note that in this case $\bc$ is not an input to the function $R$ but it is a parameterization. In other words, $\bc$ is not a variable that the loss function is optimized over. This suggests the following relaxation of \eqref{opt2}:
\begin{align}
\min_{\bx,\bc} f(\bx) + \lambda R(\bx,\bc) \label{opt4}
\end{align}
On the other hand, instead of using $Q_\bc(\bx)$ in \eqref{opt3} (for fixed $\bc$), \cite{Yang_2019_CVPR} and \cite{gong2019differentiable} proposed to use a soft quantization function $\widetilde{Q}_\bc(\bx)$ that is differentiable everywhere with derivative not necessarily 0. They used element-wise {\em sigmoid} and {\em tanh} functions, respectively. In both cases, there is a parameter $P$ that controls how closely $\widetilde{Q}_\bc(\bx)$ approximates the staircase function $Q_\bc(\bx)$. See Section~\ref{sub:softquant} for an example of the soft quantization function. In general, as $P$ increases, $\widetilde{Q}_\bc(\bx)$ starts to become a staircase-like function, and for low $P$ (i.e., $P \rightarrow 1$), $\widetilde{Q}_\bc(\bx)$ resembles the identity function that maps $\widetilde{Q}_\bc(\bx)_i$ to $x_i$; see \citep[Figure 2]{Yang_2019_CVPR}. Another advantage of using $\widetilde{Q}_\bc(\bx)$ instead of $Q_\bc(\bx)$ that is not exploited by previous works is that it enables the use of Lipschitz properties in the convergence analysis.
This suggests the following relaxation of \eqref{opt3}:
%We formalize these approximations in \eqref{opt2} and \eqref{opt3}, respectively, while using {\em both} $\bx$ and $\bc$ as optimization variables, as follows:
\begin{align}
\min_{\bx,\bc} \big(h(\bx,\bc):=f(\widetilde{Q}_\bc(\bx))\big). \label{opt5}
\end{align}
%\begin{align}
%\begin{array}{rl}\label{opt4}
%\displaystyle
%\min_{c,x} & f(x) + \lambda R(x,c)
%\end{array}
%\end{align}
%\begin{align}\label{opt5}
%\begin{array}{rl}
%\displaystyle
%\min_{c,x} & \cred{h(x,c):=f(\widetilde{Q}_c(x))}
%\end{array}
%\end{align}
\paragraph{Both relaxations have short-comings.}
The first relaxed problem \eqref{opt4} does not capture how the centers should be chosen such that the neural network loss is minimized. The centers are treated as variables that should only be close to weights that are assigned to them. %\cblue{In the second relaxed problem \eqref{opt5}, the partial gradient w.r.t.\ $\bc$ is due to the neural network loss itself.} 
We believe that modeling the direct effect that centers have on the neural network loss is crucial for a complete quantized training, and we verify this through numerics; see Section~\ref{sec:experiments}. In the second relaxed problem \eqref{opt5}, we can observe the effect of centers on neural network loss; however, this time the gradient w.r.t.\ $\bx$ is heavily dependent on the choice of $\widetilde{Q}_\bc$ and optimizing over $\bx$ might deviate too much from optimizing the neural network loss function. For instance, when $P \rightarrow \infty$, i.e., $\widetilde{Q}_\bc(\bx) \rightarrow Q_\bc(\bx)$, gradient w.r.t.\ $\bx$ is $0$ almost everywhere; hence, every point becomes a first order stationary point.

\subsubsection{A Relaxed Problem for Model Quantization}
%\textbf{A relaxed problem for model quantization.}
Our aim is to come up with an objective function that would not diminish the significance of both $\bx$ and $\bc$ in the overall procedure. To leverage the benefits of both, we combine both optimization problems \eqref{opt4} and \eqref{opt5} into one problem:
\begin{align}
%\begin{array}{rl}\label{mainopt}
%\displaystyle
\min_{\bx,\bc} \big(F_{\lambda}(\bx,\bc):=f(\bx) + \lambda R(\bx,\bc) + f(\widetilde{Q}_\bc(\bx))\big). \label{mainopt}
%\end{array}
\end{align}
Here, the first two terms help us preserve the connection of $\bx$ to neural network loss function, and the last term gives us the chance of optimizing centers w.r.t.\ the neural network training loss itself. As a result, we have an objective function  that is continuous everywhere, and for which we can use Lipschitz tools in the convergence analysis.
\begin{remark} \label{remark:centralized_compare}
It is important to note that with this new objective function we are able to optimize not only over weights but also over centers. This allows us to theoretically analyze how the movements of the centers affect the convergence. As far as we know, this has not been the case in the literature of quantized neural network training. Moreover, we observe numerically that optimizing over centers improves performance of the network; see Section \ref{subsec:centralized_expts}.
\end{remark} 

\subsection{An Example of the Soft Quantization Function} \label{sub:softquant}
In this section we give an example of the soft quantization function that can be used in previous sections. In particular, we can define the following soft quantization function: $\widetilde{Q}_\bc(\bx):\mathbb{R}^{d+m}\rightarrow\mathbb{R}^d$ and $\widetilde{Q}_\bc(\bx)_i := \sum_{j=2}^m(c_j-c_{j-1})\sigma(P(x_i-\frac{c_j+c_{j-1}}{2}))+c_1$ where $\sigma$ denotes the sigmoid function and $P$ is a parameter controlling how closely $\widetilde{Q}_\bc(\bx)$ approximates $Q_\bc(\bx)$. Note that as $P \rightarrow \infty$, $\widetilde{Q}_\bc(\bx)\rightarrow Q_\bc(\bx)$. This function can be seen as a simplification of the function that was used in \cite{Yang_2019_CVPR}.\\

\noindent\textbf{Assumption.}  For all $j\in[m]$, $c_j$ is in a compact set. In other words, there exists a finite $c_{\max}$ such that $|c_j| \leq c_{\max}$ for all $j\in[m]$. \\

In addition, we assume that the centers are sorted, i.e., $c_1 < \cdots < c_m$. Now, we state several useful facts.
\begin{fact}\label{sq:fact1}
	 $\widetilde{Q}_\bc(\bx)$ is continuously and infinitely differentiable everywhere.
\end{fact}
\begin{fact}\label{sq:fact2}
	$\sigma(\bx)$ is a Lipschitz continuous function.
\end{fact}
\begin{fact}\label{sq:fact3}
	Sum of Lipschitz continuous functions is also Lipschitz continuous. 
\end{fact}
\begin{fact}\label{sq:fact4}
	Product of bounded and Lipschitz continuous functions is also Lipschitz continuous.
\end{fact}
\begin{fact}\label{sq:fact5}
	Let $g:\mathbb{R}^n \rightarrow \mathbb{R}^m$. Then, the coordinate-wise Lipschitz continuity implies overall Lipschitz continuity. In other words, let $g_i$ be the i'th output then if $g_i$ is Lipschitz continuous for all $i$, then $g$ is also Lipschitz continuous.
\end{fact}
In our convergence analysis, we will require that $\widetilde{Q}_\bc(\bx)$ is Lipschitz continuous as well as smooth with respect to both $\bx$ and $\bc$. We show these properties in the following claims which we prove in Appendix~\ref{appendix:Soft Quantization}.
\begin{claim} \label{claimlQ_1}
	$\widetilde{Q}_\bc(\bx)$ is $l_{Q_1}$-Lipschitz continuous and $L_{Q_1}$-smooth with respect to $\bx$.
\end{claim}
\begin{claim}  \label{claimlQ_2}
	$\widetilde{Q}_\bc(\bx)$ is $l_{Q_2}$-Lipschitz continuous and $L_{Q_2}$-smooth with respect to $\bc$.
\end{claim}
%See Appendix~\ref{appendix:Soft Quantization} for the proofs of the above claims. 
The example we gave in this section is simple yet provides technical necessities we require in the analysis. Other examples can also be used as long as they provide the smoothness properties that we utilize in the next sections. 
\section{Centralized Model Quantization Training} \label{sec:centralized}
%In this section, we discuss our training scheme for leaning quantized models in a centralized setting which aims to minimize the objective \eqref{mainopt}. We first provide steps for the training algorithm and then provide its convergence guarantees for general smooth objectives. 

The goal of this section is to learn a quantized model for inference which can be deployed in a memory-constrained setting where storing a full precision model is not feasible. For this, we propose a training scheme (described in Algorithm \ref{algo:centralized}) for minimizing \eqref{mainopt} by optimizing over $\bx \in \mathbb{R}^d$ (the model parameters) and $\bc \in \mathbb{R}^m$ (quantization values/centers). 
Note that we keep $\bx$ full precision during training and learn the optimal quantization parameters $\bc$ through Algorithm \ref{algo:centralized}.
The learned quantization values are then used to hard-quantize the personalized models to get compressed models for deployment, as noted in Section~\ref{sec:problem}. 

Before we begin with the description of our algorithm we define the proximal mapping. Let $g:\mathbb{R}^n \rightarrow \mathbb{R}$. Given $\bx\in\mathbb{R}^n$ and $t > 0$, the proximal map of $g$ is defined as:
\begin{align}
	\text{prox}_{tg}(\bx):={\argmin_{\bu\in\mathbb{R}^n} }\left\{ tg(\bu) + \frac{1}{2}\|\bu-\bx\|^2 \right\} = {\argmin_{\bu\in\mathbb{R}^n} }\left\{ g(\bu) + \frac{1}{2t}\|\bu-\bx\|^2 \right\}
\end{align}
Particularly, if $g(\bx)=\delta_A(\bx)$ is the indicator function of $A$, the proximal map reduces to the projection:
\begin{align}
	\text{prox}_{tg}(\bx)=\text{proj}_A(\bx)=\argmin_{\bu \in A}\|\bu-\bx\|^2
\end{align}
In our algorithm $g$ will correspond to $\lambda R(\bx,\bc)$. Note that as $\lambda \rightarrow \infty$ we have $\lambda R(\bx,\bc)\rightarrow \delta_{\calC}(\bx)$. As a result, $\text{prox}_{\eta \lambda R}$ can be seen as a soft projection.

As a shorthand notation, for centralized case, we will use  $\text{prox}_{\eta_1\lambda R_{\bc^{t}}}$ to denote $\text{prox}_{\eta_1\lambda R(\cdot,\bc^{t})}$, and $\text{prox}_{\eta_2\lambda R_{\bx^{t+1}}}$ to denote $\text{prox}_{\eta_2\lambda R(\bx^{t+1},\cdot)}$. Similarly for the personalized case, we will use $\text{prox}_{\eta_1\lambda R_{\bc^{t}_i}}$ to denote $\text{prox}_{\eta_1\lambda R(\cdot,\bc^{t}_i)}$, and $\text{prox}_{\eta_2\lambda R_{\bx^{t+1}_i}}$ to denote $\text{prox}_{\eta_2\lambda R(\bx^{t+1}_i,\cdot)}$.

%\subsection{Centralized model quantization training algorithm}
%\subsection{Description of the Algorithm}
%Our proposed scheme is given in Algorithm \ref{algo:centralized}
%and is based on proximal gradient updates where we optimize the model and quantization levels simultaneously. \TODO{Comparison to other centralized quantized training}
\begin{algorithm}[h]
	\caption{Centralized Model Quantization Scheme}
	{\bf Input:} Regularization parameter $\lambda$; initialize the full precision model $\bx^{0}$ and quantization centers $\bc^{0}$; a penalty function enforcing quantization $R(\bx,\bc)$; a soft quantizer $\widetilde{Q}_{\bc}(\bx)$; and learning rates $\eta_1,\eta_2$.\\
	\vspace{-0.3cm}
	\begin{algorithmic}[1] 	\label{algo:centralized}
		\FOR{$t=0$ \textbf{to} $T-1$}
			\STATE Compute $\bg^{t} :=  \nabla_{\bx^{t}} f(\bx^{t})+ \nabla_{\bx^{t}}  f(\widetilde{Q}_{\bc^{t}}(\bx^{t}))$
			\STATE {$\bx^{t+1}=\text{prox}_{\eta_1 \lambda R_{\bc^{t}}}(\bx^{t}-\eta_1 \bg^{t} )$} \\
			\STATE Compute $\bh^{t} = \nabla_{\bc^{t}} f(\widetilde{Q}_{\bc^{t}}(\bx^{t+1})) $
			\STATE $\bc^{t+1}=\text{prox}_{\eta_2 \lambda R_{\bx^{t+1}}}(\bc^{t}-\eta_2 \bh^{t} )$ \\
		\ENDFOR
%		\STATE $\hat{\bx}^{T} = Q_{\bc^{T}}(\bx^{T})$ -- $\forall i\in[d]$, map $x^T_i$ to the nearest $c^T_j$. \\
	\end{algorithmic}
	{\bf Output:} Quantized model $\hat{\bx}^{T}= Q_{\bc^{T}}(\bx^{T})$, %where component $i\in[d]$ of $\hat{\bx}^T$ is defined as $\hat{x}^T_i=\arg\min_{c\in\bc^T}\{|\hat{x}^T_i-c|\}$. % 
	i.e., for $i\in[d]$, map $x^T_i$ to the nearest $c^T_j$, the $j$'th component of $\bc^T$.
\end{algorithm}

%We now formally present the steps of Algorithm \ref{algo:centralized} when running for $T$ iterations and learning rate \TODO{initialize learning rate}, and some regularization coefficient $\lambda$ for quantization. 
%\textbf{Description of the algorithm.}
\subsection{Description of the Algorithm}
We optimize \eqref{mainopt} through alternating proximal gradient descent steps.
The model parameters and the quantization vector are initialized to random vectors $\bx^{0}$ and $\bc^{0}$, respectively. We have two learning rates $\eta_1,\eta_2$ for updating $\bx^t,\bc^t$, respectively.
Note that the objective in \eqref{mainopt} is composed of two parts: the loss function $f(\bx) + f(\tilde{Q}_{\bc}(\bx))$ and a quantization inducing term $R(\bx,\bc)$, which we control by a regularization coefficient $\lambda$.
At each iteration $t$, we first compute the gradient $\bg^{t}$ of the loss function w.r.t.\ $\bx^{t}$ (line 2), and then take the gradient step followed by the $\mathrm{prox}$ step for updating $\bx^t$ to $\bx^{t+1}$ (line 3). 
For the centers, first we compute the gradient $\bh^t$ of the loss function w.r.t.\ $\bc^{t}$ (line 4), and then take the gradient step followed by the $\mathrm{prox}$ step for updating $\bc^t$ to $\bc^{t+1}$ (line 5). 
Thus both the update steps ensure that we simultaneously learn model parameter and quantization vector tied together through proximal mapping of the regularization function $R$. 
%After $T$ iterations, we obtain a full precision model $\bx^{T}$ and quantization vector (centers) $\bc^{T}$. 
Finally, to obtain the compressed model for deployment, in line 7 we quantize the full-precision model $\bx^{T}$ with the set of centers $\bc^{T}$ using $Q_{\bc^T}$, which maps components of $\bx^T$ to the nearest component of $\bc^T$, as explained after \eqref{opt3}. In Algorithm~\ref{algo:centralized}, we consider gradient descent as the update rule in computations of $\bg^t$ and $\bh^t$; however, one can also employ other methods such as SGD and ADAM.

%\textit{Choice of the regularizer $R$:} Though our results do not put any restriction on $R$, in practice, it should be chosen such that its proximal mapping is easy to compute. We discuss a possible choice of $R$ in Section~\ref{sec:experiments}.   

%See Section~\ref{sec:experiments} where we discuss a possible choice of $R$ for which its proximal mapping is easy to compute.

%\subsection{Convergence Result}
%\textbf{Convergence result.}
\subsection{Convergence Result}
We now provide convergence guarantee for Algorithm \ref{algo:centralized} for general smooth objectives.
%In this section, we provide our main theoretical results for convergence rate of Algorithms \ref{algo:centralized} and \ref{algo:personalized}. 
We first state the assumptions required for deriving our results.

%\textbf{Assumptions:}\\
{\bf A.1} \textit{(Finite lower bound of $f$):} 
We assume that $f(\bx)>-\infty$ for all $\bx\in\R^d$, which implies that $F_\lambda(\bx,\bc) > -\infty$ for any $\bx\in\R^d,\bc\in\R^m,\lambda\in\R$. 

%$$F_\lambda(\bx,\bc) \geq 0 \quad \forall \bx \in \mathbb{R}^d, \forall \bc \in \mathbb{R}^m,\forall \lambda \in \mathbb{R}$$ 
%{\sf (ii)}  $f$ is $L$-smooth\footnote{A function $f: \mathbb{R}^d \rightarrow \mathbb{R}$ is $L-smooth$ if for each $x,y \in \mathbb{R}^d$ we have: $f(y) \leq f(x) + \left\langle\nabla f(x), y-x\right\rangle + \frac{L}{2}\|x-y\|^2$} \\
{\bf A.2} \textit{(Smoothness of $f$):} $f$ is $L$-smooth, i.e., for all $\bx,\by \in \mathbb{R}^d$, we have $f(\by) \leq f(\bx) + \left\langle\nabla f(\bx), \by-\bx\right\rangle + \frac{L}{2}\|\bx-\by\|^2$.
 
{\bf A.3} \textit{(Bounded gradients of $f$):} There exists a finite constant $G < \infty$ such that
%\begin{align*}
$\|\nabla f(\bx)\|_2 \leq G$ holds for all $\bx\in\R^d$.
 
%\end{align*} 
{\bf A.4} \textit{(Smoothness of the soft quantizer):} We assume that $\widetilde{Q}_{\bc}(\bx)$ is $l_{Q_1}$-Lipschitz and $L_{Q_1}$-smooth w.r.t.\ $\bx$, i.e., for any $\bc\in\R^m$, the following holds for all $\bx,\by \in \mathbb{R}^d$:
\begin{align*}
\|\widetilde{Q}_{\bc}(\bx) - \widetilde{Q}_{\bc}(\by) \| & \leq l_{Q_1} \|\bx-\by\| \ \\
\|\nabla_{\bx} \widetilde{Q}_{\bc}(\bx) - \nabla_{\by} \widetilde{Q}_{\bc}(\by) \| & \leq L_{Q_1} \|\bx-\by\| \ 
\end{align*}
%that is, for all $\bc,\bd \in \mathbb{R}^m$ we have:
%\begin{align*}
%\|\widetilde{Q}_{\bc}(\bx) - \widetilde{Q}_{\bd}(\bx) \| & \leq l_{Q_1} \|\bc-\bd\| \  \\
%\|\nabla_{\bc} \widetilde{Q}_{\bc}(\bx) - \nabla_{\bd} \widetilde{Q}_{\bd}(\bx) \| & \leq L_{Q_1} \|\bc-\bd\| \ 
%\end{align*}
And that $\widetilde{Q}_{\bc}(\bx)$ is $l_{Q_2}$-Lipschitz and $L_{Q_2}$-smooth w.r.t.\ $\bc$, i.e., for any $\bx\in\R^d$, the following holds for all $\bc,\bd \in \mathbb{R}^m$:  
\begin{align*}
\|\widetilde{Q}_{\bc}(\bx) - \widetilde{Q}_{\bd}(\bx) \| & \leq l_{Q_2} \|\bc-\bd\| \  \\
\|\nabla_{\bc} \widetilde{Q}_{\bc}(\bx) - \nabla_{\bd} \widetilde{Q}_{\bd}(\bx) \| & \leq L_{Q_2} \|\bc-\bd\| \ 
%\|\widetilde{Q}_{\bc}(\bx) - \widetilde{Q}_{\bc}(\by) \| & \leq l_{Q_2} \|\bx-\by\| \ \\
%\|\nabla_{\bx} \widetilde{Q}_{\bc}(\bx) - \nabla_{\by} \widetilde{Q}_{\bc}(\by) \| & \leq L_{Q_2} \|\bx-\by\| \ 
\end{align*}
%{\sf (v)} \textbf{Smoothness of soft quantizer (II):} $\widetilde{Q}(\bc,\bx)$ is $L_{Q,\bx}$-smooth and $l_{Q,\bx}$ Lipschitz with respect to $\bx$:
%\begin{align*}
%\|\widetilde{Q}_{\bc}(\bx) - \widetilde{Q}_{\bc}(\by) \| & \leq l_{Q,\bx} \|\bx-\by\| \ \\
%\|\nabla_{\bx} \widetilde{Q}_{\bc}(\bx) - \nabla_{\by} \widetilde{Q}_{\bc}(\by) \| & \leq L_{Q,\bx} \|\bx-\by\| \ \forall \bx,\by \in \mathbb{R}^d  
%\end{align*}
{\bf A.5} \textit{(Bound on partial gradients of the soft quantizer):} There exists constants $G_{Q_1},G_{Q_2} <\infty $ such that:
\begin{align*}
\|\nabla_{\bx} \widetilde{Q}_{\bc}(\bx) \|_F &= \|\nabla \widetilde{Q}_{\bc}(\bx)_{1:d,:} \|_F \leq G_{Q_1} \\
\|\nabla_{\bc} \widetilde{Q}_{\bc}(\bx) \|_F &= \|\nabla \widetilde{Q}_{\bc}(\bx)_{d+1:d+m,:} \|_F \leq G_{Q_2}
\end{align*}
where $\bX_{p:q,:}$ denotes the sub-matrix of $\bX$ containing rows between $p$ and $q$, and $\|\cdot\|_F$ is the Frobenius norm. 

Now we state our main convergence result of Algorithm \ref{algo:centralized} for minimizing $F_{\lambda}(\bx,\bc)$ in \eqref{mainopt} w.r.t.\ $(\bx,\bc)\in\R^{d+m}$. In the following theorem we provide the first-order guarantees for convergence of $(\bx,\bc)$ to a stationary point.
\begin{theorem}\label{thm:centralized}
Consider running Algorithm \ref{algo:centralized} for $T$ iterations with $\eta_1=\nicefrac{1}{2(L+GL_{Q_1}+G_{Q_1}Ll_{Q_1})}$ and $\eta_2=\nicefrac{1}{2(GL_{Q_2}+G_{Q_2}Ll_{Q_2})}$. For any $t\in[T]$, define $\bG^{t} := [\nabla_{\bx^{t+1}} F_{\lambda}\left(\bx^{t+1},\bc^{t}\right)^T, \nabla_{\bc^{t+1}} F_{\lambda}\left(\bx^{t+1},\bc^{t+1}\right)^T]^T$. Then, under assumptions {\bf A.1-A.5}, we have:
	\begin{align*}%\label{centralized_convergence}
%	\frac{1}{T}\sum_{t=0}^{T-1}\|\bG^{t}\|^2_{2}  & = \mathcal{O} \bigg(\frac{L_{\max}^2\left(F_{\lambda}\left(\bz^{0}\right){-}F_{\lambda}(\bz^{T})\right)}{L_{\min}T}\bigg)
	\frac{1}{T}\sum_{t=0}^{T-1}\|\bG^{t}\|^2_{2}  & = \mathcal{O} \left(\frac{L_{\max}^2\left(F_{\lambda}\left(\bx^{0},\bc^0\right){-}F_{\lambda}(\bx^{T},\bc^T)\right)}{L_{\min}T}\right),
	\end{align*}
	where %$\bz^{t} = (\bx^{t},\bc^{t})$ for $t\in\{0,T\}$, 
	$L_{\min}=\min\{L+GL_{Q_1}+G_{Q_1}Ll_{Q_1}, GL_{Q_2}+G_{Q_2}Ll_{Q_2}\}$ and 
	$L_{\max} = \max\{L+GL_{Q_1}+G_{Q_1}Ll_{Q_1},GL_{Q_2}+G_{Q_2}Ll_{Q_2})\}$.
\end{theorem}
%Note that $\Vert \bG^{t} \Vert_2^2$ corresponds to the sum of gradient norms of the objective w.r.t.\ the centers and the model. Theorem \ref{thm:centralized} shows that the running average of this norm goes to zero with $T$. Thus, this provides a weak notion of convergence of parameters $\bx$ and $\bc$ to stationary points when run for a large value of $T$.
%We give an outline for a proof for Theorem \ref{thm:centralized} in Section \ref{sec:proof_outlines} and 
We prove Theorem~\ref{thm:centralized} in Section~\ref{sec:proof_centralized}. Note that we recover $\frac{1}{T}$ convergence rate in \cite{Bolte13,bai2018proxquant}.

%%------------------------------------------------------------------

%%------------------------------------------------------------------
%% ALGORITHMS
\section{QuPeL: Personalized Quantization for FL} \label{sec:personalized}
In this section, we extend Algorithm~\ref{algo:centralized} to the distributed/federated setting for learning quantized and personalized models for each client. 
As mentioned in Section~\ref{sec:intro}, we do so by performing multiple local iterations at clients and use our centralized scheme locally at clients for updating their local models before synchronizing with the server.
To this end, we propose a new algorithm QuPeL (described in Algorithm \ref{algo:personalized}) for optimizing \eqref{per1} over $\left(\{\bx_i,\bc_i\}_{i=1}^{n},\bw\right)$, where $\bx_i,\bc_i$ respectively denote the model parameters and the quantization vector (centers) for client $i$, and $\bw$ denotes the global model that facilitates collaboration among clients, which is encouraged through a penalty term in the local objectives \eqref{qpfl}.
As discussed in Section~\ref{sec:problem}, $\bc_1,\hdots,\bc_n$ could be different in length, which allows QuPeL to learn models with different precision for different clients based on their memory constraints. Thus, QuPeL simultaneously addresses two important personalization aspects, one for heterogeneous data and the other for resource diversity.

\begin{algorithm}[h]
	\caption{QuPeL: Quantized Personalization Learning}
	{\bf Input:} %Regularization parameters for quantization and personalization  $\lambda,\lambda_p $ respectively, synchronization gap \( \tau \).  For each worker $i$, initial full precision personalized model, initial quantization vectors (centers) and initial local model: $\bz_{i}^{{0}}=(\bx_{i}^{(0)},\bc_{i}^{(0)},\bw_{i}^{(0)})$. A penalty function enforcing quantization $R(\bx,\bc)$, \TODO{a soft quantization function $\widetilde{Q}_{\bc}(\bx)$}. Learning rates $\eta_1,\eta_2,\eta_3$.\\
	Regularization parameters $\lambda,\lambda_p$; synchronization gap $\tau$; for each client $i\in[n]$, initialize full precision personalized model $\bx_i^{0}$, quantization centers $\bc_i^{0}$, and local model $\bw_i^0$; a penalty function enforcing quantization $R(\bx,\bc)$; a soft quantizer $\widetilde{Q}_{\bc}(\bx)$; and learning rates $\eta_1,\eta_2,\eta_3$.\\
	\vspace{-0.3cm}
	\begin{algorithmic}[1] \label{algo:personalized}
		\FOR{$t=0$ \textbf{to} $T-1$}
		\STATE \textbf{On Clients} $i=1$ \textbf{to} $n$ (in parallel) \textbf{do}:
		\IF{$\tau$ does not divide $t$}
		\STATE Compute $\bg_{i}^{t} := \nabla_{\bx_{i}^{t}} f_i(\bx_{i}^{t}) + \nabla_{\bx_{i}^{t}} f_i(\widetilde{Q}_{\bc_{i}^{t}}(\bx_{i}^{t})) $
		\STATE $\bx_{i}^{t+1}=\text{prox}_{\eta_1 \lambda R_{\bc_{i}^{t}}}(\bx_{i}^{t} - \eta_1 (\bg_{i}^{t} + \lambda_p(\bx_{i}^{t}-\bw_{i}^{t}))  )$\\
		\STATE Compute $\bh_i^{t} := \nabla_{\bc_{i}^{t}} f_i(\widetilde{Q}_{\bc_{i}^{t}}(\bx_{i}^{t+1})) $
		\STATE $\bc_{i}^{t+1}=\text{prox}_{\eta_2 \lambda R_{\bx_{i}^{t+1}}}(\bc_{i}^{t}-\eta_2 \bh_i^{t}  )$\\
		\STATE $\bw_{i}^{t+1} = %w_{i,t}-\eta_3 \nabla_{w_{t}} F_i(x_{i,t+1},c_{i,t+1},w_{i,t})=
		\bw_{i}^{t}-\eta_3 \lambda_p (\bx_{i}^{t+1}-\bw_{i}^{t}) $\\
		\ELSE
		\STATE Send $\bw_{i}^{t}$ to \textbf{Server}
		\STATE Receive $\bw^{t}$ from \textbf{Server} and set $\bw_i^{t+1} = \bw^{t}$
		\ENDIF
		\STATE \textbf{On Server do:}
		\IF{$\tau$ divides $t$}
		%		\STATE $\bw^{t+1} = \bw^{t}$
		%		\ELSE
		\STATE Receive $\{\bw_i^{t}\}_{i=1}^n$ and compute $\bw^{t} := \frac{1}{n} \sum_{i=1}^n \bw_i^{t}$
		\STATE Broadcast $\bw^{t}$ to all \textbf{Clients}
		\ENDIF		
		\ENDFOR
		\STATE $\hat{\bx}_{i}^{T} = Q_{\bc_{i}^{T}}(\bx_{i}^{T})$ for all $i \in [n]$
	\end{algorithmic}
	{\bf Output:} Quantized personalized models $\hat{\bx}_i^{T}$ for $i \in [n]$
\end{algorithm}
\subsection{Description of the Algorithm}
%We now formally discuss the update steps of QuPeL when run for $T$ iterations. 
Since clients perform local iterations, apart from maintaining $\bx_i^t,\bc_i^t$ at client $i\in[n]$, it also maintains a model $\bw_i^t$ which helps in utilizing other clients' data via collaboration.
We call set $\{\bw_i^t\}$ local copies of the global model at clients at time $t$; client $i$ updates $\bw_i^t$ in between communication rounds based on its local data and synchronizes that with the server who aggregates all of them to update the global model.
Note that the local objective for each node $i \in [n]$ in \eqref{qpfl} can be split into two terms: the loss function $f_i(\bx_i) + f_i(\widetilde{Q}_{\bc_i}(\bx_i))+ \frac{\lambda_p}{2} \|\bx_i - \bw_i \|^2$ and the term enforcing quantization $\lambda R(\bx_i,\bc_i)$. 
%To optimize over this objective, at each time step $t \in [T]$, Each client $i \in [n]$ maintains three parameters: model $\bx_i^{(t)}$, quantization vector $\bc_i^{(t)}$ and a copy of the global model $\bw_i^{(t)}$ (as we perform local iterations). 
At any step $t$ that is not a communication round (line 3), client $i$ first computes the gradient $\bg_i^t$ of the loss function w.r.t.\ $\bx_i^{t}$ (line 4) and then takes a gradient step followed by the proximal step using $R$ (line 5) to update from $\bx_i^t$ to $\bx_i^{t+1}$. Then it computes the gradient $\bh_i^t$ of the loss function w.r.t.\ $\bc_i^{t}$ (line 6) and takes the update step for the centers followed by the proximal step (line 7). 
Finally, % the local copy of the global model $\bw_i^t$ is updated by taking a gradient step using 
it updates $\bw_i^{t}$ to $\bw_i^{t+1}$ by taking a gradient step of the loss function at $\bw_i^{t}$ (line 8).
%it computes the gradient of the loss function w.r.t.\ $\bw_i^{(t)}$ to $\bw_i^{t+1}$ (line 8). 
Note that, unlike in the centralized case, in QuPeL, the local training of $\bx_i^t,\bc_i^t$ also incorporates knowledge from other clients' data through $\bw_i^{t}$.
In a communication round (when $t$ is divisible by $\tau$), clients upload $\{\bw_i^{t}\}$ to the server (line 10) which aggregates them (line 15) and broadcasts the updated global model to all clients (line 16). 
At the end of training, clients learn their personalized models $\{\bx_i^{T}\}_{i=1}^{n}$ and quantization centers $\{\bc_i^{T}\}_{i=1}^{n}$. Client $i$ then quantizes $\bx_i^{T}$ to the values in $\bc_i^{T}$ using $Q_{\bc_i^T}$ (line 19), as we did in Algorithm~\ref{algo:centralized}. 
%, forming the compressed model $\hat{\bx}_i^{(T)}$ (line 19) which can be deployed in resource constrained environment for inference. 

%The update steps in line 5 and line 7 ensure that the local personal model and the centers are simultaneously updated while being tied together by the quantization inducing term $R$ due to the proximal step. 
%Furthermore, the local training of these parameters also incorporates knowledge from other nodes participating in the training through the parameter $w_i^{(t)}$. 
%In the communication round (when $t$ is divisible by $\tau$), the clients upload $\{\bw_i^{(t)}\}$ to the server (line 10) which aggregates them (line 15) and broadcasts the updated global model to all clients (line 16). 

%\subsection{Theoretical results}
%\textbf{Convergence result.}
\subsection{Convergence Result}
We now discuss the convergence rate for QuPeL for general smooth objectives.
In addition to the assumptions {\bf A.1-A.5} made in Section~\ref{sec:centralized} (for each client) we need one more assumption that bounds heterogeneity in the local datasets across all clients.

%\textbf{Assumptions:}
%Some of these are extensions of assumptions made in Section~\ref{subsec:main-1} to case of multiple workers:
\iffalse
{\sf (vi) } \textbf{Lower bound on local objective:} 
\begin{align*}
F_i(\bx,\bc,\bw)\geq 0\, , \forall i \in [n], \forall \bx\in \mathbb{R}^d, \bc \in \mathbb{R}^m, \bw \in \mathbb{R}^d
\end{align*}
{\sf (vii)} \textbf{Smoothness of \TODO{objectives}}: $f_i$ is $L_i$ smooth for $i \in [n]$. \\
{\sf (viii)}  \textbf{Bounded gradient}: There exists constant $G<\infty$ s.t.:
\begin{align*}
\Vert \nabla f_i \Vert_2 \leq G \, \quad \forall\bx \in \mathbb{R}^d, i\in [n]
\end{align*} 
{\sf (ix)} \textbf{Smoothness of soft quantizers:} For each node $i \in [n], $assumptions {\sf (iv)} hold for the soft quantization function $\tilde{Q}^{(i)}_{\bc}(\bx)$ with constants $L^{(i)}_{Q_1}$, $l^{(i)}_{Q_1}$, $L^{(i)}_{Q_2}$, $l^{(i)}_{Q_2}$ respectively. \\
{\sf (x)} \textbf{Bound on partial gradient of soft quantizers:} For each node $i \in [n]$, assumption {\sf (v)} holds for the soft quantization function $\tilde{Q}^{(i)}_{\bc}(\bx)$ with constants $G^{(i)}_{Q_1}$, $G^{(i)}_{Q_2}$. \\
\fi
{\bf A.6} \textit{(Bounded diversity):}
At any $t \in \{0,\cdots,T-1\}$ and any client $i\in[n]$, the variance of the local gradient (at client $i$) w.r.t. the global gradient is bounded, i.e., there exists $\kappa_i<\infty$, such that for every $\{\bx_i^{t+1}\in\R^d,\bc_i^{t+1}\in\R^{m_i}:i\in[n]\}$ and $\bw^t\in\R^d$ generated according to Algorithm~\ref{algo:personalized}, we have:
\begin{align*}
\Big\| \nabla_{\bw^t} F_i(\bx_{i}^{t+1},\bc_{i}^{t+1},\bw^t) - \frac{1}{n} \sum_{j=1}^n \nabla_{\bw^t} F_j(\bx_{j}^{t+1},\bc_{j}^{t+1},\bw^{t}) \Big\|^2 \leq \kappa_i,
\end{align*}
This assumption is a variant of the bounded diversity assumption in \cite{dinh2020personalized,fallah2020personalized}; the variance is due to the formulation of our objective function. In particular, diversity assumption (Assumption 5) in \cite{fallah2020personalized} is as follows:
\begin{align*}
	\frac{1}{n} \sum_{i=1}^n \Big\| \nabla_{\bw} f_i(\bw) -\nabla_{\bw} f(\bw) \Big\|^2 \leq B,
\end{align*}
where $f_i$ is local function, $B$ is a constant and $\nabla_{\bw} f(\bw) = \frac{1}{n} \sum_{j=1}^n \nabla_{\bw} f_i(\bw)$.
Now we will show the equivalence to our stated assumption {\bf A.6}. Let us define,
\begin{align*}
	x_i(\bw^t) &:= \underset{\bx \in \mathbb{R}^{d}}{\arg \min }\left\{\left\langle \bx-\bx^t_{i}, \nabla f_i\left(\bx^t_{i}\right)\right\rangle+\left\langle \bx-\bx^t_{i}, \nabla_{\bx^t_{i}} f_i(\widetilde{Q}_{\bc^t_{i}}(\bx^t_{i}))\right\rangle +\left\langle \bx-\bx^t_{i}, \lambda_p (\bx^t_{i}-\bw^{t})\right\rangle \right.\\
	& \quad \left. +\frac{1}{2 \eta_1}\left\|\bx-\bx^t_{i}\right\|_{2}^{2}+\lambda R(\bx,\bc^t_{i})\right\} \\
	c_i(\bw^t) &:= \underset{\bc \in \mathbb{R}^{m}}{\arg \min }\left\{\left\langle \bc-\bc^t_{i}, \nabla_{\bc^t_{i}} f_i(\widetilde{Q}_{\bc^t_{i}}(x_i(\bw^t)))\right\rangle  \right.  \left. +\frac{1}{2 \eta_2}\left\|\bc-\bc^t_{i}\right\|_{2}^{2}+\lambda R(x_i(\bw^t),\bc)\right\}
\end{align*}
Then we can define,
\begin{align*}
	\psi_i(x_i(\bw^t),c_i(\bw^t),\bw^t):=F_i(\bx^{t+1}_i,\bc^{t+1}_i,\bw^t)
\end{align*}
as a result, we can further define $g_i(\bw^t):=\psi_i(x_i(\bw^t),c_i(\bw^t),\bw^t)$. Therefore, our assumption \textbf{A.6} is equivalent to stating the following assumption:
At any $t \in [T]$ and any client $i\in[n]$, the variance of the local gradient (at client $i$) w.r.t. the global gradient is bounded, i.e., there exists $\kappa_i<\infty$, such that for every $\bw^t\in\R^d$, we have:
\begin{align*}
	\Big\| \nabla_{\bw^t} g_i(\bw^t) - \frac{1}{n} \sum_{j=1}^n \nabla_{\bw^t} g_j(\bw^t) \Big\|^2 \leq \kappa_i,
\end{align*}
And we also define $\kappa := \frac{1}{n} \sum_{i=1}^n \kappa_i$ and then,
\begin{align*}
	\frac{1}{n} \sum_{i=1}^n \Big\| \nabla_{\bw^t} g_i(\bw^t) -\nabla_{\bw^t} g(\bw^t) \Big\|^2 \leq \kappa,
\end{align*}
here $\nabla_{\bw^t} g(\bw^t) = \frac{1}{n} \sum_{j=1}^n \nabla_{\bw^t} g_j(\bw^t)$. Hence, our assumption is equivalent to assumptions that are found in aforementioned works.

%where $\text{Let } \nabla_{\bw^t} F(\{\bx_{i}^{t}\}^n_i,\{\bc_{i}^{t}\}^n_i,\bw^{t}) = \frac{1}{n} \sum_{i=1}^n \nabla_{\bw^{t}} F_i(\bx_{i}^{t},\bc_{i}^{t},\bw^{t} )$. 
%Define $\kappa:=\frac{1}{n}\sum_{i=1}^n \kappa_i$.
Now we state our main convergence result of Algorithm~\ref{algo:personalized} for optimizing \eqref{per1}.
Since the objective function in \eqref{per1} is (non-convex) smooth, as in Theorem~\ref{thm:centralized}, in the following theorem also we provide the first-order convergence guarantee.

\begin{theorem}\label{thm:personalized}
%	Let $\bG_{i}^{t} := [\nabla_{\bx_{i}^{t+1}} F_i(\bx_{i}^{t+1},\bc_{i}^{t},\bw^{t})]^T, \\ \nabla_{\bc_{i}^{t+1}} F_i(\bx_{i}^{t+1},\bc_{i}^{t+1},\bw^{t})^T, \nabla_{\bw^{t}} F_i(\bx_{i}^{t+1},\bc_{i}^{t+1},\bw^{t})^T]^T$ be the concatenated vector of gradients for any node $i \in [n]$. 
Consider running Algorithm \ref{algo:personalized} for $T$ iterations with $\tau \leq \sqrt{T}$, $\eta_1=\nicefrac{1}{2(2\lambda_p+L+GL_{Q_1}+G_{Q_1}Ll_{Q_1})}$, $\eta_2 = \nicefrac{1}{2(GL_{Q_2}+G_{Q_2}Ll_{Q_2})}$, and $\eta_3 = \nicefrac{1}{4\lambda_p\sqrt{T}}$.
For any $t\in[T]$, define $\bG_{i}^{t} := [\nabla_{\bx_{i}^{t+1}} F_i(\bx_{i}^{t+1},\bc_{i}^{t},\bw^{t})^T, \\ \nabla_{\bc_{i}^{t+1}} F_i(\bx_{i}^{t+1},\bc_{i}^{t+1},\bw^{t})^T, \nabla_{\bw^{t}} F_i(\bx_{i}^{t+1},\bc_{i}^{t+1},\bw^{t})^T]^T$.
%$\bG_{i}^{t}:=[\bG_{i,1}^t\ \bG_{i,2}^t\ \bG_{i,3}^t]^T$, where $\bG_{i,1}^t=\nabla_{\bx_{i}^{t+1}} F_i(\bx_{i}^{t+1},\bc_{i}^{t},\bw^{t})^T$, $\bG_{i,2}^t=\nabla_{\bc_{i}^{t+1}} F_i(\bx_{i}^{t+1},\bc_{i}^{t+1},\bw^{t})^T$, $\bG_{i,3}^t=\nabla_{\bw^{t}} F_i(\bx_{i}^{t+1},\bc_{i}^{t+1},\bw^{t})^T$.
Then, under assumptions {\bf A.1-A.6}, we have:
	%
	%	\begin{align*}
	%	\frac{1}{T}\sum_{t=0}^{T-1}\frac{1}{n}\sum_{i=1}^{n} \Big\|G_{i,t}\Big\|^2 \leq& \frac{54(L^{(i)}_{max})^2\tau\kappa+108(L^{(i)}_{max})^2+288(L^{(i)}_{max})^2\lambda_p \Delta_F}{\sqrt{T}} + \frac{189\tau^2\kappa+\frac{3}{4}\tau^2\kappa^2}{T} + \frac{378\tau^2\kappa}{T^{\frac{3}{2}}}+ 216\kappa
	%	\end{align*}
	\begin{align*}
	\frac{1}{T}\sum_{t=0}^{T-1}\frac{1}{n}\sum_{i=1}^{n} \Big\|\bG_{i}^{t}\Big\|^2 &= \mathcal{O} \left( \frac{\Lambda}{\sqrt{T}} + \frac{L_{\max}^2\tau^2\kappa+\tau^2\kappa^2}{T}+\frac{L_{\max}^2\tau^2\kappa}{T^{\frac{3}{2}}}+ L_{\max}^2\kappa \right) ,
	\end{align*}
	where $L_{\max} = \max\{1,GL_{Q_2}+G_{Q_2}LL_{Q_2},\frac{5}{3}\lambda_p+L+GL_{Q_1}+G_{Q_1}LL_{Q_1}\}$, and $\Lambda = L_{\max}^2\tau\kappa+L_{\max}^2+L_{\max}^2\lambda_p \Delta_F$ with $\Delta_F=\frac{1}{n}\sum_{i=1}^{n}\left(F_i(\bx_{i}^{0},\bc_{i}^{0},\bw_{i}^{0})-F_i(\bx_{i}^{T},\bc_{i}^{T},\bw_{i}^{T})\right)$. %$L_{max} =\max\{\lambda_p+L+GL_{Q_2}+G_{Q_2}Ll_{Q_2},GL_{Q_1}+G_{Q_1}Ll_{Q_1},1\}$, and $\kappa=\frac{1}{n}\sum_{i=1}^n \kappa_i$.
\end{theorem}
%In the above result $\Vert \bG_i \Vert_2^2$ refers to the sum of gradient norm of the local objective $F_i(\bx,\bc,\bw)$ w.r.t. to the parameters $\bx_i$, $\bc_i$ and $\bw_i$. Theorem \ref{thm:personalized} establishes a bound on the running time average of the norm of the gradients averaged across nodes. 
Hence, Algorithm~\ref{algo:personalized} approximately converges to a stationary point with a rate of $\nicefrac{1}{\sqrt{T}}$ with an error that depends on the gradient diversity, induced by the heterogeneity in local data. This matches the convergence rate of \cite{dinh2020personalized}.
%The bound shows convergence of this quantity to a constant dictated by $\kappa$, which is the gradient dissimilarity among clients, matching the rates \TODO{CITE}.

We provide a proof of Theorem \ref{thm:personalized} in Section~\ref{sec:proof_personalized}.
%%------------------------------------------------------------------

%%------------------------------------------------------------------
%% PROOF1
\section{Proof of Theorem \ref{thm:centralized}} \label{sec:proof_centralized}
This proof consists of two parts. First we show the sufficient decrease property by sequentially using Lipschitz properties for each update step in Algorithm~\ref{algo:centralized}. For each variable $\bx$ and $\bc$ we find the decrease inequalities and then combine them to obtain an overall sufficient decrease. Then we bound the norm of the gradient using optimality conditions of the proximal updates in Algorithm~\ref{algo:centralized}. Using sufficient decrease and bound on the gradient we arrive at the result. We leave some of the derivations and proof of the claims to Appendix~\ref{appendix:proof of theorem 1}.

\textbf{Alternating updates.} Remember that for the Algorithm~\ref{algo:centralized} we have the following alternating updates:
\begin{align*}
	\bx^{t+1} &= \text{prox}_{\eta_1\lambda R_{\bc^{t}}}(\bx^{t} - \eta_1 \nabla f(\bx^{t})-\eta_1 \nabla_{\bx^{t}} f(\widetilde{Q}_{\bc^{t}}(\bx^{t})) )\\
	\bc^{t+1} &= \text{prox}_{\eta_2\lambda R_{\bx^{t+1}}}(\bc^{t} - \eta_2 \nabla_{\bc^{t}} f(\widetilde{Q}_{\bc^{t}}(\bx^{t+1})))
\end{align*}
These translate to following optimization problems for $\bx$ and $\bc$ respectively:
\begin{align}
	\bx^{t+1} &=\underset{\bx \in \mathbb{R}^{d}}{\arg \min }\left\{\left\langle \bx-\bx^{t}, \nabla_{\bx^{t}} f\left(\bx^{t}\right)\right\rangle+\left\langle \bx-\bx^{t}, \nabla_{\bx^{t}} f(\widetilde{Q}_{\bc^{t}}(\bx^{t}))\right\rangle+\frac{1}{2 \eta_1}\left\|\bx-\bx^{t}\right\|_{2}^{2}+\lambda R(\bx,\bc^{t})\right\} \label{thm1:optimization prob for x}\\
	\bc^{t+1}&= \underset{\bc \in \mathbb{R}^{m}}{\arg \min }\left\{\left\langle \bc-\bc^{t}, \nabla_{\bc^{t}} f(\widetilde{Q}_{\bc^{t}}(\bx^{t+1}))\right\rangle+\frac{1}{2 \eta_2}\left\|\bc-\bc^{t}\right\|_{2}^{2}+\lambda R(\bx^{t+1},\bc)\right\} \label{thm1:optimization prob for c}
\end{align}
See appendix for derivation.
\subsection{Sufficient Decrease}
This section is divided into two, first we will show sufficient decrease property with respect to $\bx$, then we will show sufficient decrease property with respect to $\bc$. 
\subsubsection{Sufficient Decrease Due to $\bx$}
\begin{claim}\label{claim: lqxsmooth}
	$f(\bx)+ f(\widetilde{Q}_\bc(\bx))$ is $(L+GL_{Q_1}+G_{Q_1}LL_{Q_1})$-smooth with respect to $\bx$.
\end{claim}
Using Claim \ref{claim: lqxsmooth} we have,
\begin{align}
	F_\lambda(\bx^{t+1},\bc^t)\ +\ & (\frac{1}{2\eta_1}-\frac{L+GL_{Q_1}+G_{Q_1}LL_{Q_1}}{2})\|\bx^{t+1}-\bx^t\|^2 \notag \\ 
	&= f(\bx^{t+1})+f(\widetilde{Q}_{\bc^t}(\bx^{t+1}))+\lambda R(\bx^{t+1},\bc^t) + (\frac{1}{2\eta_1}-\frac{L+GL_{Q_1}+G_{Q_1}LL_{Q_1}}{2})\|\bx^{t+1}-\bx^t\|^2 \notag \\ 
	&\leq f(\bx^t)+f(\widetilde{Q}_{\bc^t}(\bx^t))+\lambda R(\bx^{t+1},\bc^t)+\left\langle \nabla f(\bx^t), \bx^{t+1}-\bx^t\right\rangle + \left\langle \nabla_{\bx^t} f(\widetilde{Q}_{\bc^t}(\bx^t)), \bx^{t+1}-\bx^t\right\rangle \notag \\
	&\hspace{5cm} + \frac{1}{2\eta_1}\|\bx^{t+1}-\bx^t\|^2\label{thm1:first-part-interim1}
\end{align}
\begin{claim}\label{claim:quantization lower bound 1}
	Let 
	\begin{align*}
		A(\bx^{t+1}) &:= \lambda R(\bx^{t+1},\bc^{t})+\left\langle \nabla f(\bx^{t}), \bx^{t+1}-\bx^{t}\right\rangle 
		+ \left\langle \nabla_{\bx^{t}} f(\widetilde{Q}_{\bc^{t}}(\bx^{t})), \bx^{t+1}-\bx^{t}\right\rangle \notag + \frac{1}{2\eta_1}\|\bx^{t+1}-\bx^{t}\|^2 \notag \\
		A(\bx^{t}) &:= \lambda R(\bx^{t},\bc^{t}).
	\end{align*} 
	Then $A(\bx^{t+1})\leq A(\bx^{t})$.
\end{claim}
Now we use Claim~\ref{claim:quantization lower bound 1} and get,
\begin{align*} 
	 f\left(\bx^{t}\right)+\left\langle \bx^{t+1}-\bx^{t}, \nabla f\left(\bx^{t}\right)\right\rangle + \frac{1}{2 \eta_1}\left\|\bx^{t+1}-\bx^{t}\right\|_{2}^{2} +\lambda R\left(\bx^{t+1}, \bc^{t}\right)\ +\ & f(\widetilde{Q}_{\bc^{t}}(\bx^{t})) \\
	 +\left\langle \nabla_{\bx^{t}} f(\widetilde{Q}_{\bc^{t}}(\bx^{t})), \bx^{t+1}-\bx^{t} \right\rangle 
	&\leq f\left(\bx^{t}\right)+f(\widetilde{Q}_{\bc^{t}}(\bx^{t}))+\lambda R\left(\bx^{t},\bc^{t}\right) \\
	& = F_{\lambda}\left(\bx^{t},\bc^{t}\right).
\end{align*}
Using \eqref{thm1:first-part-interim1} we have,
\begin{align*}
	F_\lambda(\bx^{t+1},\bc^t)+(\frac{1}{2\eta_1}-\frac{L+GL_{Q_1}+G_{Q_1}LL_{Q_1}}{2})\|\bx^{t+1}-\bx^t\|^2 \leq  F_{\lambda}\left(\bx^{t},\bc^{t}\right).
\end{align*}
Now, we choose $\eta_1 = \frac{1}{2(L+GL_{Q_1}+G_{Q_1}Ll_{Q_1})}$ and obtain the decrease property for $\bx$:
\begin{align} \label{thm1:sufficient decrease 1}
 F_{\lambda}\left(\bx^{t+1},\bc^{t}\right) + \frac{L+GL_{Q_1}+G_{Q_1}Ll_{Q_1}}{2}\|\bx^{t+1}-\bx^{t}\|^2 \leq	F_{\lambda}\left(\bx^{t},\bc^{t}\right). 
\end{align}
\subsubsection{Sufficient Decrease Due to $\bc$}
From Claim~\ref{claim: lqcclaim} we have $f(\widetilde{Q}_\bc(\bx))$ is $(GL_{Q_2}+G_{Q_2}LL_{Q_2})$-smooth with respect to $\bc$.
Using Claim~\ref{claim: lqcclaim},
\begin{align}
	&F_\lambda(\bx^{t+1},\bc^{t+1}) +(\frac{1}{2\eta_2}-\frac{GL_{Q_2}+G_{Q_2}LL_{Q_2}}{2})\|\bc^{t+1}-\bc^t\|^2 \notag \\
	 &\qquad= f(\bx^{t+1})+f(\widetilde{Q}_{\bc^{t+1}}(\bx^{t+1}))+\lambda R(\bx^{t+1},\bc^{t+1}) + (\frac{1}{2\eta_2}-\frac{GL_{Q_2}+G_{Q_2}LL_{Q_2}}{2})\|\bc^{t+1}-\bc^t\|^2 \notag \\ 
	&\qquad\leq f(\bx^{t+1})+f(\widetilde{Q}_{\bc^t}(\bx^{t+1}))+\lambda R(\bx^{t+1},\bc^{t+1}) + \left\langle \nabla_{\bc^t} f(\widetilde{Q}_{\bc^t}(\bx^{t+1})), \bc^{t+1}-\bc^t\right\rangle + \frac{1}{2\eta_2}\|\bc^{t+1}-\bc^t\|^2\label{thm1:first-part-interim2}
\end{align}
Now we state the counterpart of Claim~\ref{claim:quantization lower bound 1} for $\bc$.
\begin{claim}\label{claim:quantization lower bound 2}
	Let 
	\begin{align*}
		B(\bc^{t+1}) &:= \lambda R(\bx^{t+1},\bc^{t+1}) + \left\langle \nabla_{\bc^{t}} f(\widetilde{Q}_{\bc^{t}}(\bx^{t+1})), \bc^{t+1}-\bc^{t}\right\rangle \notag + \frac{1}{2\eta_1}\|\bc^{t+1}-\bc^{t}\|^2 \notag \\
		B(\bc^{t}) &:= \lambda R(\bx^{t+1},\bc^{t}).
	\end{align*} 
	Then $B(\bc^{t+1})\leq B(\bc^{t})$.
\end{claim}
Now using Claim~\ref{claim:quantization lower bound 2},
\begin{align} \nonumber
	f(\bx^{t+1}) + \eta_2\left\|\bc^{t+1}-\bc^{t}\right\|_{2}^{2}+\lambda R(\bx^{t+1}, \bc^{t+1}) + f(\widetilde{Q}_{\bc^{t}}(\bx^{t+1}))+\left\langle \bc^{t+1}-\bc^{t}, \nabla_{\bc^{t}} f(\widetilde{Q}_{\bc^{t}}(\bx^{t+1}))\right\rangle \nonumber \\ \nonumber
	\leq f\left(\bx^{t+1}\right)+f(\widetilde{Q}_{\bc^{t}}(\bx^{t+1}))+\lambda R\left(\bx^{t+1},\bc^{t}\right) = F_{\lambda}\left(\bx^{t+1},\bc^{t}\right)
\end{align}
Setting $\eta_2 = \frac{1}{2(GL_{Q_2}+G_{Q_2}Ll_{Q_2})}$ and using the bound in \eqref{thm1:first-part-interim2}, we obtain the sufficient decrease for $\bc$:
\begin{align} \label{thm1:sufficient decrease 2}
	F_{\lambda}\left(\bx^{t+1},\bc^{t+1}\right) + \frac{GL_{Q_2}+G_{Q_2}Ll_{Q_2}}{2}\|\bc^{t+1}-\bc^{t}\|^2 \leq
	F_{\lambda}\left(\bx^{t+1},\bc^{t}\right)
\end{align}
%\textbf{Overall Decrease.} 
\subsubsection{Overall Decrease} 
Summing the bounds in \eqref{thm1:sufficient decrease 1} and \eqref{thm1:sufficient decrease 2}, we have the overall decrease property:
\begin{align} \label{thm1:overall decrease}
	F_{\lambda}(\bx^{t+1},\bc^{t+1}) + \frac{L+GL_{Q_1}+G_{Q_1}Ll_{Q_1}}{2}\|\bx^{t+1}-\bx^{t}\|^2+\frac{GL_{Q_2}+G_{Q_2}Ll_{Q_2}}{2}\|\bc^{t+1}-\bc^{t}\|^2 \leq 
	F_{\lambda}\left(\bx^{t},\bc^{t}\right) 
\end{align}
Let us define $L_{\min}=\min\{L+GL_{Q_1}+G_{Q_1}Ll_{Q_1}, GL_{Q_2}+G_{Q_2}Ll_{Q_2}\}$, and  $\bz^t = (\bx^{t},\bc^{t})$. Then from \eqref{thm1:overall decrease}:
\begin{align*}
	F_{\lambda}\left(\bz^{t+1}\right) + \frac{L_{\min}}{2}(\|\bz^{t+1}-\bz^t\|^2) = F_{\lambda}\left(\bx^{t+1},\bc^{t+1}\right) + \frac{L_{\min}}{2}(\|\bx^{t+1}-\bx^{t}\|^2+\|\bc^{t+1}-\bc^{t}\|^2) \leq
	F_{\lambda}\left(\bx^{t},\bc^{t}\right) = 	F_{\lambda}(\bz^t)
\end{align*}

Telescoping the above bound for $t=0, \ldots, T-1,$ and dividing by $T$:
\begin{align} \label{proximity1}
	\frac{1}{T}\sum_{t=0}^{T-1}(\left\|\bz^{t+1}-\bz^{t}\right\|_{2}^{2})\leq \frac{2\left(F_{\lambda}\left(\bz^{0}\right)-F_{\lambda}\left(\bz^{T}\right)\right)}{L_{\min}T} 
\end{align}

\subsection{Bound on the Gradient}

%%\leq \frac{2\left(F_{\lambda}\left(\bz_{0}\right)-F_{\star}\right)}{L_{\min}}
%Therefore we have the proximity guarantee
%\begin{align}
%\min _{0 \leq t \leq T-1}(\left\|\bz_{t+1}-\bz_{t}\right\|_{2}^{2}) \leq \frac{2\left(F_{\lambda}\left(\bz_0\right)-F_{\star}\right)}{L_{F1} T}
%\end{align}

We now find the first order stationarity guarantee. Taking the derivative of \eqref{thm1:optimization prob for x} with respect to $\bx$ at $\bx=\bx^{t+1}$ and setting it to 0 gives us the first order optimality condition:
\begin{align} \label{thm1:foc1}
	\nabla f(\bx^{t})+\nabla_{\bx^{t}} f(\widetilde{Q}_{\bc^{t}}(\bx^{t}))+\frac{1}{\eta_1}\left(\bx^{t+1}-\bx^{t}\right)+\lambda \nabla_{\bx^{t+1}} R\left(\bx^{t+1},\bc^{t}\right)=0
\end{align}
Combining the above equality and Claim~\ref{claim: lqxsmooth}:

\begin{align*}
	\left\|\nabla_{\bx^{t+1}} F_{\lambda}(\bx^{t+1},\bc^{t})\right\|_{2} &= \left\|\nabla f(\bx^{t+1})+\nabla_{\bx^{t+1}} f(\widetilde{Q}_{\bc^{t}}(\bx^{t+1}))+\lambda \nabla_{\bx^{t+1}} R(\bx^{t+1},\bc^{t})\right\|_{2} \\
	&\stackrel{\text{(a)}}{=} \left\|\frac{1}{\eta}\left(\bx^{t}-\bx^{t+1}\right)+\nabla f(\bx^{t+1})-\nabla f(\bx^{t}) + \nabla_{\bx^{t+1}} f(\widetilde{Q}_{\bc^{t}}(\bx^{t+1})) - \nabla_{\bx^{t}} f(\widetilde{Q}_{\bc^{t}}(\bx^{t}))\right\|_{2} \\ 
	&\leq (\frac{1}{\eta_1}+L+GL_{Q_1}+G_{Q_1}Ll_{Q_1})\left\|\bx^{t+1}-\bx^{t}\right\|_{2} \\
	&\stackrel{\text{(b)}}{=} 3(L+GL_{Q_1}+G_{Q_1}Ll_{Q_1})\left\|\bx^{t+1}-\bx^{t}\right\|_{2} \\
	& \leq 3 (L+GL_{Q_1}+G_{Q_1}Ll_{Q_1})\left\|\bz_{t+1}-\bz_{t}\right\|_{2}
\end{align*}

where (a) is from \eqref{thm1:foc1} and (b) is because we chose $\eta_1 = \frac{1}{2(L+GL_{Q_1}+G_{Q_1}Ll_{Q_1})}$. First order optimality condition in \eqref{thm1:optimization prob for c} for $\bc^{t+1}$ gives:
\begin{align*}
	\nabla_{\bc^{t+1}} f(\widetilde{Q}_{\bc^{t+1}}(\bx^{t+1}))+\frac{1}{\eta_2}(\bc^{t+1}-\bc^{t})+\lambda \nabla_{\bc^{t+1}} R(\bx^{t+1},\bc^{t+1})=0
\end{align*}

Combining the above equality and Claim~\ref{claim: lqcclaim}:

\begin{align*}
	\left\|\nabla_{\bc^{t+1}} F_{\lambda}(\bx^{t+1},\bc^{t+1})\right\|_{2} &=\left\|\nabla_{\bc^{t+1}} f(\widetilde{Q}_{\bc^{t+1}}(\bx^{t+1}))+\lambda \nabla_{\bc^{t+1}} R(\bx^{t+1},\bc^{t+1})\right\|_{2} \\
	&= \left\|\frac{1}{\eta_2}\left(\bc^{t}-\bc^{t+1}\right)+ \nabla_{\bc^{t+1}} f(\widetilde{Q}_{\bc^{t+1}}(\bx^{t+1})) - \nabla_{\bc^{t}} f(\widetilde{Q}_{\bc^{t}}(\bx^{t+1}))\right\|_{2} \\ 
	&\leq (\frac{1}{\eta_2}+GL_{Q_2}+G_{Q_2}Ll_{Q_2})\left\|\bc^{t+1}-\bc^{t}\right\|_{2} \\
	&\stackrel{\text{(a)}}{=}3 (GL_{Q_2}+G_{Q_2}Ll_{Q_2})\left\|\bc^{t+1}-\bc^{t}\right\|_{2} \\
	& \leq 3 (GL_{Q_2}+G_{Q_2}Ll_{Q_2})\left\|\bz^{t+1}-\bz^{t}\right\|_{2}
\end{align*}

where (a) is because we set $\eta_2 = \frac{1}{2(GL_{Q_2}+G_{Q_2}Ll_{Q_2})}$. Then:
\begin{align*}
&\left\|[\nabla_{\bx^{t+1}} F_{\lambda}(\bx^{t+1},\bc^{t})^T,\nabla_{\bc^{t+1}} F_{\lambda}(\bx^{t+1},\bc^{t+1})^T]^T\right\|^2_{2} =  \left\|\nabla_\bx F_{\lambda}\left(\bx^{t+1},\bc^{t}\right)\right\|^2_{2} + \left\|\nabla_\bc F_{\lambda}\left(\bx^{t+1},\bc^{t+1}\right)\right\|^2_{2} \\ 
&\hspace{6.5cm}\leq 3^2(G(L_{Q_1}+L_{Q_2})+L(1+G_{Q_1}l_{Q_1}+G_{Q_2}l_{Q_2}))^2\|\bz^{t+1}-\bz^{t}\|^2
\end{align*}

Letting $L_{\max} = \max\{L+GL_{Q_1}+G_{Q_1}Ll_{Q_1},GL_{Q_2}+G_{Q_2}Ll_{Q_2}\}$ we have:

\begin{align*}
	\left\|[\nabla_{\bx^{t+1}} F_{\lambda}(\bx^{t+1},\bc^{t})^T,\nabla_{\bc^{t+1}} F_{\lambda}(\bx^{t+1},\bc^{t+1})^T]^T\right\|_{2}^2 \leq 9L_{\max}^2\left\|\bz^{t+1}-\bz^{t}\right\|^2_{2}
\end{align*}

Summing over all time points and dividing by $T$:

\begin{align*}
	\frac{1}{T}\sum_{t=0}^{T-1}\left\|[\nabla_{\bx^{t+1}} F_{\lambda}(\bx^{t+1},\bc^{t})^T,\nabla_{\bc^{t+1}} F_{\lambda}(\bx^{t+1},\bc^{t+1})^T]^T\right\|^2_{2} &\leq \frac{1}{T}\sum_{t=0}^{T-1}9L_{\max}^2(\left\|\bz^{t+1}-\bz^{t}\right\|_{2}) \\ & \leq \frac{18 L_{\max}^2\left(F_{\lambda}\left(\bz^{0}\right)-F_{\lambda}(\bz^{T})\right)}{L_{\min}T},
\end{align*}
where in the last inequality we use \eqref{proximity1}. This concludes the proof of Theorem~\ref{thm:centralized}.
%%-------

%%------------------------------------------------------------------
%% PROOF2
\section{Proof of Theorem \ref{thm:personalized}} \label{sec:proof_personalized}
\allowdisplaybreaks{
%In this section we give a proof outline for Theorem~\ref{thm:personalized}; a detailed proof can be found in Appendix~\ref{appendix:proof of theorem 2}.

In this part, different than Section~\ref{sec:proof_centralized}, we have an additional update due to local iterations. The key is to integrate the local iterations into our alternating update scheme. To do this, we use smoothness with respect to $\bw_i$ alongside with Assumption \textbf{A.6}. This proof again consists of two parts. First, we show the sufficient decrease property by sequentially using and combining Lipschitz properties for each update step in Algorithm~\ref{algo:personalized}. Then, we bound the norm of the gradient using optimality conditions of the proximal updates in Algorithm~\ref{algo:personalized}. Then, by combining the sufficient decrease results and bounds on partial gradients we will derive our result. We defer proofs of the claims and some derivation details to Appendix~\ref{appendix:proof of theorem 2}. In this analysis we take $\bw^t = \frac{1}{n}\sum_{i=1}^n \bw^t_{i}$, so that $\bw^t$ is defined for every time point.

\textbf{Alternating updates.} Let us first restate the alternating updates for $\bx_i$ and $\bc_i$:
\begin{align*}
	\bx^{t+1}_{i} &= \text{prox}_{\eta_1 \lambda R_{\bc^t_{i}}}(\bx^t_{i}-\eta_1 \nabla f_i(\bx^t_{i})- \eta_1 \nabla_{\bx^t_{i}} f_i(\widetilde{Q}_{\bc^t_{i}}(\bx^t_{i}))-\eta_1 \lambda_p(\bx^t_{i}-\bw^{t}_{i}) ) \\
	\bc^{t+1}_{i} &= \text{prox}_{\eta_2 \lambda R_{\bx^{t+1}_{i}}}(\bc^t_{i}-\eta_2 \nabla_{\bc^t_{i}} f_i(\widetilde{Q}_{\bc^t_{i}}(\bx^t_{i})))
\end{align*}
The alternating updates are equivalent to solving the following two optimization problems (see appendix for derivation).
\begin{align}
	\bx^{t+1}_{i}&= \underset{\bx \in \mathbb{R}^{d}}{\arg \min }\left\{\left\langle \bx-\bx^t_{i}, \nabla f_i\left(\bx^t_{i}\right)\right\rangle+\left\langle \bx-\bx^t_{i}, \nabla_{\bx^t_{i}} f_i(\widetilde{Q}_{\bc^t_{i}}(\bx^t_{i}))\right\rangle +\left\langle \bx-\bx^t_{i}, \lambda_p (\bx^t_{i}-\bw^{t}_{i})\right\rangle \right. \notag \\
	&\hspace{9cm} \left. +\frac{1}{2 \eta_1}\left\|\bx-\bx^t_{i}\right\|_{2}^{2}+\lambda R(\bx,\bc^t_{i})\right\} \label{thm2:lower-bounding-interim1}\\
	\bc^{t+1}_{i}&= \underset{\bc \in \mathbb{R}^{m}}{\arg \min }\left\{\left\langle \bc-\bc^t_{i}, \nabla_{\bc^t_{i}} f_i(\widetilde{Q}_{\bc^t_{i}}(\bx^{t+1}_{i}))\right\rangle  \right.  \left. +\frac{1}{2 \eta_2}\left\|\bc-\bc^t_{i}\right\|_{2}^{2}+\lambda R(\bx^{t+1}_{i},\bc)\right\} \label{thm2:lower-bounding-interim2}
\end{align}
%Let us state some useful facts.
%\begin{fact}
Note that the update on $\bw^t$ from Algorithm~\ref{algo:personalized} can be written as:
\begin{align*}
\bw^{t+1} = \bw^{t} - \eta_3 \bg^{t}, \quad \text{ where }\quad \bg^{t} =\frac{1}{n} \sum_{i=1}^n \nabla_{\bw^{t}_{i}} F_i(\bx^{t+1}_{i},\bc^{t+1}_{i},\bw^{t}_{i}).
%	\bw^{t+1} = \bw^{t} - \eta_3 \frac{1}{n} \sum_{i=1}^n \nabla_{\bw^{t}_{i}} F_i(\bx^{t+1}_{i},\bc^{t+1}_{i},\bw^{t}_{i})
\end{align*}
%\end{fact}
%We will denote $\bg^{t} =\frac{1}{n} \sum_{i=1}^n \nabla_{\bw^{t}_{i}} F_i(\bx^{t+1}_{i},\bc^{t+1}_{i},\bw^{t}_{i})$. Then $\bw^{t+1} = \bw^{t} - \eta_3 \bg^{t}$.
In the convergence analysis we will require smoothness of the local functions $F_i$ w.r.t.\ the global parameter $\bw$. Recall the definition of $F_i(\bx_i,\bc_i,\bw) = f_i(\bx_i)+f_i(\widetilde{Q}_{\bc_i}(\bx_i))+\lambda R(\bx_i,\bc_i) + \frac{\lambda_p}{2} \|\bx_i - \bw \|^2$ from \eqref{qpfl}. It follows that $F_i$ is $\lambda_p$-smooth with respect to $\bw$:
%\begin{fact}\label{thm2:lambda_p}
%Recall the definition of $F_i(\bx,\bc,\bw)$ from \eqref{qpfl} in Section~\ref{sec:problem}. Then, $F_i$ is $\lambda_p$-smooth with respect to $\bw$:
\begin{align}\label{lambda_p}
\Big\| \nabla_\bw F_i(\bx,\bc,\bw)-\nabla_{\bw'} F_i(\bx,\bc,\bw')\Big\| =\|\lambda_p(\bx-\bw)-\lambda_p(\bx-\bw')\| \leq \lambda_p \|\bw-\bw'\|, \quad \forall \bw,\bw' \in \mathbb{R}^d.
\end{align}
%\end{fact}
Now let us move on with the proof. %First we will show the sufficient decrease properties.
\subsection{Sufficient Decrease} \label{thm2:subsection sufficient decrease}
We will divide this part into three and obtain sufficient decrease properties for each variable: $\bx_i,\bc_i,\bw$.
\subsubsection{Sufficient Decrease Due to $\bx_i$}
We begin with a useful claim.
\begin{claim} \label{thm2:claim lqxsmoothness}
	$f_i(\bx)+ f_i(\widetilde{Q}_\bc(\bx))+\frac{\lambda_p}{2}\|\bx-\bw\|^2$ is $(\lambda_p+L+GL_{Q_1}+G_{Q_1}LL_{Q_1})$-smooth with respect to $\bx$.
\end{claim}
From Claim~\ref{thm2:claim lqxsmoothness} we have:

\begin{align}
	&F_i(\bx^{t+1}_{i},\bc^t_{i},\bw^{t}) +(\frac{1}{2\eta_1}-\frac{\lambda_p+L+GL_{Q_1}+G_{Q_1}LL_{Q_1}}{2})\|\bx^{t+1}_{i}-\bx^t_{i}\|^2 = f_i(\bx^{t+1}_{i})+f_i(\widetilde{Q}_{\bc^t_{i}}(\bx^{t+1}_{i})) \notag \\
	&\hspace{3cm} +\lambda R(\bx^{t+1}_{i},\bc^t_{i})+\frac{\lambda_p}{2}\|\bx^{t+1}_{i}-\bw^{t}\|^2 + (\frac{1}{2\eta_1}-\frac{\lambda_p+L+GL_{Q_1}+G_{Q_1}LL_{Q_1}}{2})\|\bx^{t+1}_{i}-\bx^t_{i}\|^2 \notag \\ 
	&\hspace{1cm}\leq f_i(\bx^t_{i})+f_i(\widetilde{Q}_{\bc^t_{i}}(\bx^t_{i}))+\frac{\lambda_p}{2}\|\bx^t_{i}-\bw^{t}\|^2+\lambda R(\bx^{t+1}_{i},\bc^t_{i})+\left\langle \nabla f_i(\bx^t_{i}), \bx^{t+1}_{i}-\bx^t_{i}\right\rangle \notag \\ 
	&\hspace{2cm} + \left\langle \nabla_{\bx^t_{i}} f_i(\widetilde{Q}_{\bc^t_{i}}(\bx^t_{i})), \bx^{t+1}_{i}-\bx^t_{i}\right\rangle+\left\langle \lambda_p(\bx^t_{i}-\bw^{t}_{i}), \bx^{t+1}_{i}-\bx^t_{i}\right\rangle + \left\langle\lambda_p(\bw^{t}_{i}-\bw^{t}), \bx^{t+1}_{i}-\bx^t_{i}\right\rangle \notag \\
	&\hspace{3cm} + \frac{1}{2\eta_1}\|\bx^{t+1}_{i}-\bx^t_{i}\|^2 \label{thm2:first-part-interim1}
\end{align}
\begin{claim}\label{thm2:claim:lower-bound1}
	Let 
	\begin{align*}
		A(\bx^{t+1}_{i}) &:= \lambda R(\bx^{t+1}_{i},\bc^t_{i})+\left\langle \nabla f_i(\bx^t_{i}), \bx^{t+1}_{i}-\bx^t_{i}\right\rangle 
		+ \left\langle \nabla_{\bx^t_{i}} f_i(\widetilde{Q}_{\bc^t_{i}}(\bx^t_{i})), \bx^{t+1}_{i}-\bx^t_{i}\right\rangle \notag \\
		&\hspace{3cm} +\left\langle \lambda_p(\bx^t_{i}-\bw^{t}_{i}), \bx^{t+1}_{i}-\bx^t_{i}\right\rangle + \frac{1}{2\eta_1}\|\bx^{t+1}_{i}-\bx^t_{i}\|^2 \notag \\
		A(\bx^t_{i}) &:= \lambda R(\bx^t_{i},\bc^t_{i}).
	\end{align*} 
	Then $A(\bx^{t+1}_{i})\leq A(\bx^t_{i})$.
	%\begin{align}
	%&\lambda R(\bx^{t+1}_{i},\bc^t_{i})+\left\langle \nabla f_i(\bx^t_{i}), \bx^{t+1}_{i}-\bx^t_{i}\right\rangle 
	%+ \left\langle \nabla_x f_i(\widetilde{Q}_{\bc^t_{i}}(\bx^t_{i})), \bx^{t+1}_{i}-\bx^t_{i}\right\rangle \notag \\
	%&\hspace{3cm} +\left\langle \lambda_p(\bx^t_{i}-\bw^{t}_{i}), \bx^{t+1}_{i}-\bx^t_{i}\right\rangle + \frac{1}{2\eta_1}\|\bx^{t+1}_{i}-\bx^t_{i}\|^2 \leq \lambda R(\bx^t_{i},\bc^t_{i})
	%\end{align}
\end{claim}
Using the inequality from Claim~\ref{thm2:claim:lower-bound1} in \eqref{thm2:first-part-interim1} gives
\begin{align}
	F_i(\bx^{t+1}_{i},\bc^t_{i},\bw^{t})\ +\ &(\frac{1}{2\eta_1}-\frac{\lambda_p+L+GL_{Q_1}+G_{Q_1}LL_{Q_1}}{2})\|\bx^{t+1}_{i}-\bx^t_{i}\|^2 \notag \\
	&\stackrel{\text{(a)}}{\leq} f_i(\bx^t_{i})+f_i(\widetilde{Q}_{\bc^t_{i}}(\bx^t_{i}))+\frac{\lambda_p}{2}\|\bx^t_{i}-\bw^{t}\|^2 + A(\bx^{t+1}_{i}) + \left\langle\lambda_p(\bw^{t}_{i}-\bw^{t}), \bx^{t+1}_{i}-\bx^t_{i}\right\rangle \notag \\
	%&\leq f_i(\bx^t_{i})+f_i(\widetilde{Q}_{\bc^t_{i}}(\bx^t_{i}))+\frac{\lambda_p}{2}\|\bx^t_{i}-\bw^{t}\|^2 + R(x,\bc^t_{i}) + \left\langle\lambda_p(\bw^{t}_{i}-\bw^{t}), \bx^{t+1}_{i}-\bx^t_{i}\right\rangle \notag \\
	&\stackrel{\text{(b)}}{\leq} f_i(\bx^t_{i})+f_i(\widetilde{Q}_{\bc^t_{i}}(\bx^t_{i}))+\frac{\lambda_p}{2}\|\bx^t_{i}-\bw^{t}\|^2 + R(\bx^t_{i},\bc^t_{i}) + \frac{\lambda_p}{2}\|\bw^{t}_{i}-\bw^{t}\|^2 + \frac{\lambda_p}{2}\|\bx^{t+1}_{i}-\bx^t_{i}\|^2 \notag \\
	&= F_i(\bx^t_{i},\bc^t_{i},\bw^{t}) + \frac{\lambda_p}{2}\|\bw^{t}_{i}-\bw^{t}\|^2+\frac{\lambda_p}{2}\|\bx^{t+1}_{i}-\bx^t_{i}\|^2 \label{thm2:first-part-interim3}
\end{align}
To obtain (a), we substituted the value of $A(\bx^{t+1}_{i})$ from Claim~\ref{thm2:claim:lower-bound1} into \eqref{thm2:first-part-interim1}. In (b), we used $A(\bx^{t+1}_{i})\leq \lambda R(\bx^t_{i},\bc^t_{i})$ and $\langle\lambda_p(\bw^{t}_{i}-\bw^{t}), \bx^{t+1}_{i}-\bx^t_{i}\rangle = \langle\sqrt{\lambda_p}(\bw^{t}_{i}-\bw^{t}), \sqrt{\lambda_p}(\bx^{t+1}_{i}-\bx^t_{i})\rangle \leq \frac{\lambda_p}{2}\|\bw^{t}_{i}-\bw^{t}\|^2 + \frac{\lambda_p}{2}\|\bx^{t+1}_{i}-\bx^t_{i}\|^2$.

Substituting $\eta_1=\frac{1}{2(2\lambda_p+L+GL_{Q_1}+G_{Q_1}LL_{Q_1})}$ in \eqref{thm2:first-part-interim3} gives:

\begin{align} \label{thm2:dec1}
	F_i(\bx^{t+1}_{i},\bc^t_{i},\bw^{t})
	+(\frac{2\lambda_p+L+GL_{Q_1}+G_{Q_1}LL_{Q_1}}{2})\|\bx^{t+1}_{i}-\bx^t_{i}\|^2 \leq F_i(\bx^t_{i},\bc^t_{i},\bw^{t})+ \frac{\lambda_p}{2}\|\bw^{t}_{i}-\bw^{t}\|^2
\end{align}
\subsubsection{Sufficient Decrease Due to $\bc_i$}
In parallel with Claim~\ref{thm2:claim lqxsmoothness}, we have following smoothness result for $\bc$:
\begin{claim} \label{thm2:claim lqcsmoothness}
	$f_i(\widetilde{Q}_\bc(\bx))$ is $(GL_{Q_2}+G_{Q_2}LL_{Q_2})$-smooth with respect to $\bc$.
\end{claim}
From Claim \ref{thm2:claim lqcsmoothness} we have:
\begin{align}
	&F_i(\bx^{t+1}_{i},\bc^{t+1}_{i},\bw^{t})+(\frac{1}{2\eta_2}-\frac{GL_{Q_2}+G_{Q_2}LL_{Q_2}}{2})\|\bc^{t+1}_{i}-\bc^t_{i}\|^2 = f_i(\bx^{t+1}_{i})+f_i(\widetilde{Q}_{\bc^{t+1}_{i}}(\bx^{t+1}_{i}))+\lambda R(\bx^{t+1}_{i},\bc^{t+1}_{i}) \notag \\
	&\hspace{7cm}+\frac{\lambda_p}{2}\|\bx^{t+1}_{i}-\bw^{t}\|^2 + (\frac{1}{2\eta_2}-\frac{GL_{Q_2}+G_{Q_2}LL_{Q_2}}{2})\|\bc^{t+1}_{i}-\bc^t_{i}\|^2 \notag \\ 
	&\hspace{3cm}\leq f_i(\bx^{t+1}_{i})+f_i(\widetilde{Q}_{\bc^t_{i}}(\bx^{t+1}_{i}))+\lambda R(\bx^{t+1}_{i},\bc^{t+1}_{i})+\frac{\lambda_p}{2}\|\bx^{t+1}_{i}-\bw^{t}\|^2 \notag \\
	&\hspace{7cm}+ \left\langle \nabla_{\bc^t_{i}} f_i(\widetilde{Q}_{\bc^t_{i}}(\bx^{t+1}_{i})), \bc^{t+1}_{i}-\bc^t_{i}\right\rangle + \frac{1}{2\eta_2}\|\bc^{t+1}_{i}-\bc^t_{i}\|^2 \label{thm2:second-part-interim1}
\end{align}

\begin{claim} \label{thm2:claim:lower-bound2}
	Let 
	\begin{align*}
		B(\bc^{t+1}_{i}) &:= \lambda R(\bx^{t+1}_{i},\bc^{t+1}_{i}) + \left\langle \nabla_{\bc^t_{i}} f_i(\widetilde{Q}_{\bc^t_{i}}(\bx^{t+1}_{i})), \bc^{t+1}_{i}-\bc^t_{i}\right\rangle \notag  + \frac{1}{2\eta_2}\|\bc^{t+1}_{i}-\bc^t_{i}\|^2 \notag \\
		B(\bc^t_{i}) &:= \lambda R(\bx^{t+1}_{i},\bc^t_{i}).
	\end{align*} 
	Then $B(\bc^{t+1}_{i})\leq B(\bc^t_{i})$.
\end{claim}
Substituting the bound from Claim~\ref{thm2:claim:lower-bound2} in \eqref{thm2:second-part-interim1}, 

\begin{align*}
	F_i(\bx^{t+1}_{i},\bc^{t+1}_{i},\bw^{t})+(\frac{1}{2\eta_2}-\frac{GL_{Q_2}+G_{Q_2}LL_{Q_2}}{2})\|\bc^{t+1}_{i}-\bc^t_{i}\|^2 & \leq f_i(\bx^{t+1}_{i})+f_i(\widetilde{Q}_{\bc^t_{i}}(\bx^{t+1}_{i}))+\lambda R(\bx^{t+1}_{i},\bc^t_{i}) \notag \\
	&\hspace{3cm}+\frac{\lambda_p}{2}\|\bx^{t+1}_{i}-\bw^{t}\|^2 \\
	& = F_i(\bx^{t+1}_i,\bc^{t}_i,\bw^t)
\end{align*}

Substituting $\eta_2 = \frac{1}{2(GL_{Q_2}+G_{Q_2}LL_{Q_2})}$ gives us: 
\begin{align} \label{thm2:dec2}
	F_i(\bx^{t+1}_{i},\bc^{t+1}_{i},\bw^{t})+\frac{GL_{Q_2}+G_{Q_2}LL_{Q_2}}{2}\|\bc^{t+1}_{i}-\bc^t_{i}\|^2  \leq	F_i(\bx^{t+1}_{i},\bc^t_{i},\bw^{t})
\end{align}

\subsubsection{Sufficient Decrease Due to $\bw$}
Now, we use $\lambda_p$-smoothness of $F_i(\bx,\bc,\bw)$ with respect to $\bw$:
\begin{align*}
	F_i(\bx^{t+1}_{i},\bc^{t+1}_{i},\bw^{t+1}) &\leq 	F_i(\bx^{t+1}_{i},\bc^{t+1}_{i},\bw^{t}) + \left\langle \nabla_{\bw^{t}} F_i(\bx^{t+1}_{i},\bc^{t+1}_{i},\bw^{t}), \bw^{t+1}-\bw^{t} \right\rangle + \frac{\lambda_p}{2}\|\bw^{t+1}-\bw^{t}\|^2 
\end{align*}
After some algebraic manipulations (see Appendix~\ref{appendix:proof of theorem 2}) we have,
\begin{align} \label{thm2:algebra for decrease of w}
		F_i(\bx^{t+1}_{i},\bc^{t+1}_{i},\bw^{t+1}) &\leq F_i(\bx^{t+1}_{i},\bc^{t+1}_{i},\bw^{t}) - (\frac{\eta_3}{2}-\lambda_p\eta_3^2)\Big\|\nabla_{\bw^{t}} F_i(\bx^{t+1}_{i},\bc^{t+1}_{i},\bw^{t})\Big\|^2 \notag \\ 
		& \quad + (\eta_3+2\lambda_p\eta_3^2)\Big\|\bg^{t} - \nabla_{\bw^{t}_{i}} F_i(\bx^{t+1}_{i},\bc^{t+1}_{i},\bw^{t}_{i})\Big\|^2 + (\eta_3+2\lambda_p\eta_3^2)\lambda_p^2\Big\|\bw^{t}_{i}-\bw^{t}\Big\|^2
\end{align}

Rearranging the terms, we have:
\begin{align} \label{thm2:dec3}
	&F_i(\bx^{t+1}_{i},\bc^{t+1}_{i},\bw^{t+1}) + (\frac{\eta_3}{2}-\lambda_p\eta_3^2)\Big\|\nabla_{\bw^{t}} F_i(\bx^{t+1}_{i},\bc^{t+1}_{i},\bw^{t})\Big\|^2 \leq
	 (\eta_3+2\lambda_p\eta_3^2)\Big\|\bg^{t} - \nabla_{\bw^{t}_{i}} F_i(\bx^{t+1}_{i},\bc^{t+1}_{i},\bw^{t}_{i})\Big\|^2 \notag \\
	 &\hspace{7cm}+ (\eta_3+2\lambda_p\eta_3^2)\lambda_p^2\Big\|\bw^{t}_{i}-\bw^{t}\Big\|^2 + F_i(\bx^{t+1}_{i},\bc^{t+1}_{i},\bw^{t}) 
\end{align}

\subsubsection{Overall Decrease} 
Define $L_{x}=2\lambda_p+L+GL_{Q_1}+G_{Q_1}LL_{Q_1}$ and $L_{c} = GL_{Q_2}+G_{Q_2}LL_{Q_2}$. Summing \eqref{thm2:dec1}, \eqref{thm2:dec2}, \eqref{thm2:dec3} we get the overall decrease property:

\begin{align} \label{thm2:dec4}
	F_i(\bx^{t+1}_{i},\bc^{t+1}_{i},\bw^{t+1}) + (\frac{\eta_3}{2}-\lambda_p\eta_3^2)\Big\|\nabla_{\bw^{t}} F_i(\bx^{t+1}_{i},\bc^{t+1}_{i},\bw^{t})\Big\|^2 \ + \frac{L_{x}}{2}\|\bx^{t+1}_{i}-\bx^t_{i}\|^2 + \frac{L_{c}}{2}\|\bc^{t+1}_{i}-\bc^t_{i}\|^2 \notag \\ \leq   (\eta_3+2\lambda_p\eta_3^2)\Big\|\bg^{t} - \nabla_{\bw^{t}_{i}} F_i(\bx^{t+1}_{i},\bc^{t+1}_{i},\bw^{t}_{i})\Big\|^2 + (\frac{1}{2}+\eta_3\lambda_p+2\lambda_p^2\eta_3^2)\lambda_p\|\bw^{t}_{i}-\bw^{t}\|^2 + F_i(\bx^t_{i},\bc^t_{i},\bw^{t})  
\end{align}
Let $L_{\min} = \min\{L_{x},L_{c},(\eta_3-2\lambda_p\eta_3^2)\} $. Then,
\begin{align} \label{thm2:dec5}
	&F_i(\bx^{t+1}_{i},\bc^{t+1}_{i},\bw^{t+1}) + \frac{L_{\min}}{2}\left(\Big\|\nabla_{\bw^{t}} F_i(\bx^{t+1}_{i},\bc^{t+1}_{i},\bw^{t})\Big\|^2 \ +\|\bx^{t+1}_{i}-\bx^t_{i}\|^2 +\|\bc^{t+1}_{i}-\bc^t_{i}\|^2\right) \notag \\
	&\leq (\eta_3+2\lambda_p\eta_3^2)\Big\|\bg^{t} - \nabla_{\bw^{t}_{i}} F_i(\bx^{t+1}_{i},\bc^{t+1}_{i},\bw^{t}_{i})\Big\|^2 + (\frac{1}{2}+\eta_3\lambda_p+2\lambda_p^2\eta_3^2)\lambda_p\|\bw^{t}_{i}-\bw^{t}\|^2 + F_i(\bx^t_{i},\bc^t_{i},\bw^{t})
\end{align}
\subsection{Bound on the Gradient}
Now, we will use the first order optimality conditions due to proximal updates and bound the partial gradients with respect to variables $\bx$ and $\bc$. After obtaining bounds for partial gradients we will bound the overall gradient and use our results from Section~\ref{thm2:subsection sufficient decrease} to arrive at the final bound.

\subsubsection{Bound on the Gradient w.r.t.\ $\bx_i$}
 Taking the derivative inside the minimization problem \eqref{thm2:lower-bounding-interim1} with respect to $\bx$ at $\bx=\bx^{t+1}_{i}$ and setting it to $0$ gives the following optimality condition:

\begin{align*}
	\nabla f_i(\bx^t_{i})+\nabla_{\bx^t_{i}} f_i(\widetilde{Q}_{\bc^t_{i}}(\bx^t_{i})) + \lambda_p(\bx^t_{i}-\bw^{t}_{i})+\frac{1}{\eta_1}(\bx^{t+1}_{i}-\bx^t_{i})+\lambda\nabla_{\bx^{t+1}_{i}} R(\bx^{t+1}_{i},\bc^t_{i}) = 0  \tag{$\star3$} \label{thm2:popt1}
\end{align*}

Then we have ,
\begin{align*}
	\Big\|\nabla_{\bx^{t+1}_{i}} F_i(\bx^{t+1}_{i},\bc^t_{i},\bw^{t})\Big\| &= \Big\|\nabla f_i(\bx^{t+1}_{i}) + \nabla_{\bx^{t+1}_{i}} f_i(\widetilde{Q}_{\bc^t_{i}}(\bx^{t+1}_{i}))+\lambda_p(\bx^{t+1}_{i}-\bw^{t}) + \lambda \nabla_{\bx^{t+1}_{i}} R(\bx^{t+1}_{i},\bc^t_{i})\Big\| \\
	&\stackrel{\text{(a)}}{=} \Big\|\nabla f_i(\bx^{t+1}_{i})-\nabla f_i(\bx^t_{i}) + \nabla_{\bx^{t+1}_{i}} f_i(\widetilde{Q}_{\bc^t_{i}}(\bx^{t+1}_{i}))-\nabla_{\bx^t_{i}} f_i(\widetilde{Q}_{\bc^t_{i}}(\bx^t_{i}))\\  
	&\hspace{5cm} +(\lambda_p-\frac{1}{\eta_1})(\bx^{t+1}_{i}-\bx^t_{i}) + \lambda_p(\bw^{t}_{i}-\bw^{t})\Big\| \\
	& \leq (\frac{1}{\eta_1} + \lambda_p+L+GL_{Q_1}+G_{Q_1}LL_{Q_1})\|\bx^{t+1}_{i}-\bx^t_{i}\| +\lambda_p\|\bw^{t}_{i}-\bw^{t}\|\\
\end{align*}
where (a) is from \eqref{thm2:popt1} and the last inequality is due to Lipschitz continuous gradients and triangle inequality. This implies: \\
\begin{align*}
	\Big\|\nabla_{\bx^{t+1}_{i}} F_i(\bx^{t+1}_{i},\bc^t_{i},\bw^{t})\Big\|^2 &\leq 2(\frac{1}{\eta_1} + L + \lambda_p+L+GL_{Q_1}+G_{Q_1}LL_{Q_1})^2\|\bx^{t+1}_{i}-\bx^t_{i}\|^2 +2\lambda_p^2\|\bw^{t}_{i}-\bw^{t}\|^2
\end{align*}
Substituting $\eta_1=\frac{1}{2(2\lambda_p+L+GL_{Q_1}+G_{Q_1}LL_{Q_1})}$ we have:
\begin{align} \label{thm2:partial grad1}
	\Big\|\nabla_{\bx^{t+1}_{i}} F_i(\bx^{t+1}_{i},\bc^t_{i},\bw^{t})\Big\|^2 &\leq 2(5\lambda_p+ 3(L+GL_{Q_1}+G_{Q_1}LL_{Q_1}))^2\|\bx^{t+1}_{i}-\bx^t_{i}\|^2 +2\lambda_p^2\|\bw^{t}_{i}-\bw^{t}\|^2 \\
	& = 18(\frac{5}{3}\lambda_p+L+GL_{Q_1}+G_{Q_1}LL_{Q_1})^2\|\bx^{t+1}_{i}-\bx^t_{i}\|^2 +2\lambda_p^2\|\bw^{t}_{i}-\bw^{t}\|^2
\end{align}
\subsubsection{Bound on the Gradient w.r.t.\ $\bc_i$}
Similarly. taking the derivative inside the minimization problem \eqref{thm2:lower-bounding-interim2} with respect to $\bc$ at $\bc^{t+1}_{i}$ and setting it to $0$ gives the following optimality condition:
\begin{align*}
	\nabla_{\bc^t_{i}} f_i(\widetilde{Q}_{\bc^t_{i}}(\bx^{t+1}_{i})) +\frac{1}{\eta_2}(\bc^{t+1}_{i}-\bc^t_{i})+\lambda\nabla_{\bc^{t+1}_{i}} R(\bx^{t+1}_{i},\bc^{t+1}_{i}) = 0 \tag{$\star 4$} \label{thm2:popt2}
\end{align*}
Then we have 
\begin{align*}
	\Big\|\nabla_{\bc^{t+1}_{i}} F_i(\bx^{t+1}_{i},\bc^{t+1}_{i},\bw^{t})\Big\| &= \Big\| \nabla_{\bc^{t+1}_{i}} f_i(\widetilde{Q}_{\bc^{t+1}_{i}}(\bx^{t+1}_{i})) + \lambda \nabla_{\bc^{t+1}_{i}} R(\bx^{t+1}_{i},\bc^{t+1}_{i})\Big\|
	\\ &= \Big\| \nabla_{\bc^{t+1}_{i}} f_i(\widetilde{Q}_{\bc^{t+1}_{i}}(\bx^{t+1}_{i})) - \nabla_{\bc^t_{i}} f_i(\widetilde{Q}_{\bc^t_{i}}(\bx^{t+1}_{i})) + \frac{1}{\eta_2}(\bc^t_{i}-\bc^{t+1}_{i})\Big\| \\
	& \leq (\frac{1}{\eta_2} +GL_{Q_2}+G_{Q_2}LL_{Q_2})\|\bc^{t+1}_{i}-\bc^t_{i}\| \\
\end{align*}

the first equality is due to \eqref{thm2:popt2} and the last inequality is due to Lipschitz continuous gradient and triangle inequality. As a result we have, 
\begin{align} \label{thm2:partial grad2}
	\Big\|\nabla_{\bc^{t+1}_{i}} F_i(\bx^{t+1}_{i},\bc^{t+1}_{i},\bw^{t})\Big\|^2 &\leq 2(\frac{1}{\eta_2} + GL_{Q_2}+G_{Q_2}LL_{Q_2})^2\|\bc^{t+1}_{i}-\bc^t_{i}\|^2 
\end{align}

substituting $\eta_2 = \frac{1}{2(GL_{Q_2}+G_{Q_2}LL_{Q_2})}$ we have:
\begin{align*}
	\Big\|\nabla_{\bc^{t+1}_{i}} F_i(\bx^{t+1}_{i},\bc^{t+1}_{i},\bw^{t})\Big\|^2 &\leq 18( GL_{Q_2}+G_{Q_2}LL_{Q_2})^2\|\bc^{t+1}_{i}-\bc^t_{i}\|^2 
\end{align*}
\subsubsection{Overall Bound}
Then, let us write $\|\bG^{t}_{i}\|^2:$
\begin{align*}
	\|\bG^{t}_{i}\|^2 & = \Big\|[\nabla_{\bx^{t+1}_{i}} F_i(\bx^{t+1}_{i},\bc^t_{i},\bw^{t})^T,\nabla_{\bc^{t+1}_{i}} F_i(\bx^{t+1}_{i},\bc^{t+1}_{i},\bw^{t})^T , \nabla_{\bw^{t}} F_i(\bx^{t+1}_{i},\bc^{t+1}_{i},\bw^{t})^T]^T\Big\|^2 \\
	& = \Big\|\nabla_{\bx^{t+1}_{i}} F_i(\bx^{t+1}_{i},\bc^t_{i},\bw^{t})\Big\|^2 + \Big\|\nabla_{\bc^{t+1}_{i}} F_i(\bx^{t+1}_{i},\bc^{t+1}_{i},\bw^{t})\Big\|^2 + \Big\|\nabla_{\bw^{t}} F_i(\bx^{t+1}_{i},\bc^{t+1}_{i},\bw^{t})\Big\|^2 \\
	& \leq 18(\frac{5}{3}\lambda_p+L+GL_{Q_1}+G_{Q_1}LL_{Q_1})^2\|\bx^{t+1}_{i}-\bx^t_{i}\|^2 +2\lambda_p^2\|\bw^{t}_{i}-\bw^{t}\|^2 \\ 
	&\hspace{4cm} + 18( GL_{Q_2}+G_{Q_2}LL_{Q_2})^2\|\bc^{t+1}_{i}-\bc^t_{i}\|^2 +\Big\| \nabla_{\bw^{t}} F_i(\bx^{t+1}_{i},\bc^{t+1}_{i},\bw^{t})\Big\|^2
\end{align*}

where the last inequality is due to \eqref{thm2:partial grad1} and \eqref{thm2:partial grad2}. Let $L_{\max} = \max\{1,GL_{Q_2}+G_{Q_2}LL_{Q_2},\frac{5}{3}\lambda_p+L+GL_{Q_1}+G_{Q_1}LL_{Q_1}\}$; then,
\begin{align} \label{thm2:bound over gradient}
	&\Big\|[\nabla_{\bx^{t+1}_{i}} F_i(\bx^{t+1}_{i},\bc^t_{i},\bw^{t})^T,\nabla_{\bc^{t+1}_{i}} F_i(\bx^{t+1}_{i},\bc^{t+1}_{i},\bw^{t})^T , \nabla_{\bw^{t}} F_i(\bx^{t+1}_{i},\bc^{t+1}_{i},\bw^{t})^T]^T\Big\|^2 \notag\\
	& \leq 18L_{\max}^2(\|\bx^{t+1}_{i}-\bx^t_{i}\|^2+\|\bc^{t+1}_{i}-\bc^t_{i}\|^2+\Big\|\nabla_{\bw^{t}} F_i(\bx^{t+1}_{i},\bc^{t+1}_{i},\bw^{t})\Big\|^2)+2\lambda_p^2\|\bw^{t}_{i}-\bw^{t}\|^2 \notag\\
	&\stackrel{\text{(a)}}{\leq}  36\frac{L_{\max}^2}{L_{\min}} \Big[(\eta_3+2\lambda_p\eta_3^2)\Big\|\bg^{t} - \nabla_{\bw^{t}_{i}} F_i(\bx^{t+1}_{i},\bc^{t+1}_{i},\bw^{t}_{i})\Big\|^2 + (\frac{1}{2}+\eta_3\lambda_p+2\lambda_p^2\eta_3^2)\lambda_p\|\bw^{t}_{i}-\bw^{t}\|^2 + F_i(\bx^t_{i},\bc^t_{i},\bw^{t}) \notag\\ 
	& \hspace{7cm} - F_i(\bx^{t+1}_{i},\bc^{t+1}_{i},\bw^{t+1}) \Big]+2\lambda_p^2\|\bw^{t}_{i}-\bw^{t}\|^2 
\end{align}
in (a) we use the bound from \eqref{thm2:dec5}. 

Now we state an useful lemma that enables us to relate local version of the global model , $\bw_i$, to global model itself , $\bw$.
\begin{lemma}\label{thm2:lemma1}
Let $\eta_3$ be chosen such that $\eta_3 \leq \sqrt{\frac{1}{12\tau^2\lambda_p^2}}$, then we have,
	\begin{align*} 
		\frac{1}{T}\sum_{t=0}^{T-1}\frac{1}{n}\sum_{i=1}^{n} \|\bw^{t}-\bw^{t}_{i}\| \leq 6 \tau^2\eta_3^2\kappa. 
	\end{align*}
\end{lemma}
As a corollary:
\begin{corollary}\label{thm2:corollary diversity}
	Recall, $\bg^{t} = \frac{1}{n} \sum_{i=1}^n \nabla_{\bw^{t}} F_i(\bx^{t+1}_{i},\bc^{t+1}_{i},\bw^{t}_{i})$. Then, we have:
	\begin{align*}
		\frac{1}{T}\sum_{t=0}^{T-1} \frac{1}{n} \sum_{i=1}^n \Big\|\bg^{t} - \nabla_{\bw^{t}_{i}} F_i(\bx^{t+1}_{i},\bc^{t+1}_{i},\bw^{t}_{i})\Big\|^2 \leq 36\lambda_p^2\tau^2\eta_3^2\kappa+3\kappa.
	\end{align*}
\end{corollary}
See appendix for the proofs. Using Lemma~\ref{thm2:lemma1} and Corollary~\ref{thm2:corollary diversity}, summing the bound in \eqref{thm2:bound over gradient} over time and clients, dividing by $T$ and $n$:
%1/(12L^lam2sqrtT)
\begin{align} \label{thm2:res1}
	\nonumber \frac{1}{T}\sum_{t=0}^{T-1}\frac{1}{n}\sum_{i=1}^{n} \|\bG^{t}_{i}\|^2 &\leq 36\frac{L_{\max}^2}{L_{\min}} \Big[(\eta_3+2\lambda_p\eta_3^2)(36\lambda_p^2\tau^2\eta_3^2\kappa+3\kappa)+(\frac{1}{2}+\eta_3\lambda_p+2\lambda_p^2\eta_3^2)\lambda_p6\tau^2\eta_3^2\kappa\\ \nonumber
	& \hspace{4cm} +\frac{\sum_{i=1}^{n}\left(F_i(\bx^{0}_{i},\bc^{0}_{i},\bw^{0}_{i})-F_i(\bx^t_{i},\bc^t_{i},\bw^{t}_{i})\right)}{nT}\Big] + 2\lambda_p^26\tau^2\eta_3^2\kappa \\
	& = 36\frac{L_{\max}^2}{L_{\min}} \Big[6\tau^2\eta_3^2\kappa(\frac{\lambda_p}{2}+7\eta_3\lambda_p^2+14\eta_3^2\lambda_p^3)+3\eta_3\kappa+6\lambda_p\eta_3^2+ \frac{\Delta_F}{T}\Big] + 12\lambda_p^2\tau^2\eta_3^2\kappa
\end{align}

where $\Delta_F=\frac{\sum_{i=1}^{n}\left(F_i(\bx^{0}_{i},\bc^{0}_{i},\bw^{0}_{i})-F_i(\bx^T_{i},\bc^T_{i},\bw^{T}_{i})\right)}{n}$.

\textbf{Choice of $\eta_3$.} Assuming $\tau \leq \sqrt{T}$ we can take $\eta_3 = \frac{1}{4\lambda_p\sqrt{T}}$, details of this choice is discussed in Appendix~\ref{appendix:proof of theorem 2}, and after some algebra we have the end result:
\begin{align*} 
\frac{1}{T}\sum_{t=0}^{T-1}\frac{1}{n}\sum_{i=1}^{n} \|\bG^{t}_{i}\|^2 & \leq  \frac{54L_{\max}^2\tau\kappa+108L_{\max}^2+288L_{\max}^2\lambda_p \Delta_F}{\sqrt{T}} + \frac{189L_{\max}^2\tau^2\kappa+\frac{3}{4}\tau^2\kappa^2}{T} + \frac{378L_{\max}^2\tau^2\kappa}{T^{\frac{3}{2}}} \\
& \hspace{1cm} + 216L_{\max}^2\kappa.
\end{align*}
This concludes the proof.

}
%%------------------------------------------------------------------

%%------------------------------------------------------------------
%% EXPERIMENTS
\section{Experiments} \label{sec:experiments}
\allowdisplaybreaks{
\subsection{Proximal Updates} \label{sec:experiments:proximal updates}
In the experiments we consider $\ell_1$-loss for the distance function $R(\bx,\bc)$. In other words, $R(\bx,\bc)=\min\{ \frac{1}{2} \|\bz-\bx\|_1:z_i \in \{c_1,\cdots,c_m\}, \forall i \}$. For simplicity, we define $\calC = \{\bz:z_i \in \{c_1,\cdots,c_m\}, \forall i\}$.
For the first type of update (update of $\bx$) we have:
\begin{align}
	\text{prox}_{\eta_1 \lambda R(\cdot,\bc)}(\by) &= {\argmin_{\bx \in \mathbb{R}^{d}} }\left\{\frac{1}{2 \eta_1}\left\|\bx-\by\right\|_{2}^{2}+\lambda R(\bx,\bc) \right\}\nonumber \\
	&= {\argmin_{\bx \in \mathbb{R}^{d}} }\left\{\frac{1}{2 \eta_1}\left\|\bx-\by\right\|_{2}^{2}+\frac{\lambda}{2} \min_{\bz \in \calC} \|\bz-\bx\|_1 \right\}\nonumber \\
	&= {\argmin_{\bx \in \mathbb{R}^{d}} } \min_{\bz \in \calC} \left\{\frac{1}{2 \eta_1}\left\|\bx-\by\right\|_{2}^{2}+\frac{\lambda}{2} \|\bz-\bx\|_1 \right\}\nonumber \\
\end{align}
This corresponds to solving:
\begin{align}
	\min_{\bz \in \calC} {\min_{\bx \in \mathbb{R}^{d}} }  \left\{\frac{1}{\eta_1}\left\|\bx-\by\right\|_{2}^{2}+\lambda \|\bz-\bx\|_1 \right\}\nonumber
\end{align}
Since both $\ell_1$ and squared $\ell_2$ norms are decomposable; if we fix $\bz$, for the inner problem we have the following solution to soft thresholding:
\begin{align}
	x^\star(\bz)_i=
	\begin{cases}
		y_i - \frac{\lambda\eta_1}{2}, \quad \text{if } y_i - \frac{\lambda\eta_1}{2} > z_i \\
		y_i + \frac{\lambda\eta_1}{2}, \quad \text{if } y_i + \frac{\lambda\eta_1}{2} < z_i \\
		z_i, \quad \text{otherwise}
	\end{cases}
\end{align}

As a result we have:

\begin{align}
	\min_{\bz \in \calC}  \left\{\frac{1}{\eta_1}\left\|x^\star(\bz)-\by\right\|_{2}^{2}+\lambda \|\bz-x^\star(\bz)\|_1 \right\}\nonumber
\end{align}
This problem is separable, in other words we have:
\begin{align*}
	\bz^\star_i = \argmin_{z_i \in \{c_1, \cdots, c_m\}} \left\{ \frac{1}{\eta_1}(x^\star(\bz)_i-y_i)^2+\lambda |z_i-x^\star(\bz)_i | \right\} \ \forall i
\end{align*}
Substituting $x^\star(\bz)_i$ and solving for $z_i$ gives us:
\begin{align*}
	\bz^\star_i = \argmin_{z_i \in \{c_1, \cdots, c_m\}} \left\{ |z_i - y_i | \right\} \ \forall i
\end{align*}
Or equivalently we have,
\begin{align}
	\bz^\star = \argmin_{\bz \in \calC} \|\bz-\by\|_1 = Q_\bc(\by)
\end{align}
As a result, $\text{prox}_{\eta_1 \lambda R(\cdot,\bc)}(\cdot)$ becomes the soft thresholding operator:
\begin{align}
	\text{prox}_{\eta_1 \lambda R(\cdot,\bc)}(\by)_i = 
	\begin{cases}
		y_i-\frac{\lambda\eta_1}{2}, \quad \text{if } y_i \geq Q_\bc(\by)_i+\frac{\lambda\eta_1}{2} \\
		y_i+\frac{\lambda\eta_1}{2}, \quad \text{if } y_i \leq Q_\bc(\by)_i-\frac{\lambda\eta_1}{2} \\
		Q_\bc(\by)_i, \quad \text{otherwise} 
	\end{cases}
\end{align}
And for the second type of update we have $\text{prox}_{\eta_2 \lambda R(\bx,\cdot)}(\cdot)$ becomes:
\begin{align}
	\text{prox}_{\eta_2 \lambda R(\bx,\cdot)}(\boldsymbol{\mu}
	) &= {\argmin_{\bc \in \mathbb{R}^{m}} }\left\{\frac{1}{2 \eta_2}\left\|\bc-\boldsymbol{\mu}
	\right\|_{2}^{2}+\lambda R(\bx,\bc) \right\}\nonumber \\
	&= {\argmin_{\bc \in \mathbb{R}^{m}} }\left\{\frac{1}{2 \eta_2}\left\|\bc-\boldsymbol{\mu}
	\right\|_{2}^{2}+\frac{\lambda}{2} \min_{\bz \in \cal C} \|\bz-\bx\|_1 \right\}\nonumber \\
	&= {\argmin_{\bc \in \mathbb{R}^{m}} } \left\{\frac{1}{2 \eta_2}\left\|\bc-\boldsymbol{\mu}
	\right\|_{2}^{2}+\frac{\lambda}{2} \|Q_\bc(\bx)-\bx\|_1 \right\} \label{app:prox2}
\end{align}
Then,
\begin{align*}
	\text{prox}_{\eta_2 \lambda R(\bx,\cdot)}(\boldsymbol{\mu}
	)_j &= {\argmin_{\bc_j \in \mathbb{R}^{m}} } \left\{\frac{1}{2 \eta_2}(c_j-\mu_j)^{2}+\frac{\lambda}{2} \sum_{i=1}^d |Q_\bc(\bx)_i-x_i| \right\} \\
	&= {\argmin_{\bc_j \in \mathbb{R}^{m}} } \left\{\frac{1}{2 \eta_2}(c_j-\mu_j)^{2}+\frac{\lambda}{2} \sum_{i=1}^d \mathbbm{1}(Q_{\bc}(\bx)_i=c_{j}) |c_j-x_i| \right\}
\end{align*}
Note that second part of the optimization problem is hard to solve; in particular we need to know the assignments of $x_i$ to $c_j$. In the algorithm, at each time point $t$ we are given the previous epoch's assignments. We can utilize that and approximate the optimization problem by assuming $\bc^{t+1}$ will be in a neighborhood of $\bc^t$. We can take the gradient of $R(\bx^{t+1},\bc)$ at $\bc=\bc^t$ while finding the optimal point. This is also equivalent to optimizing the first order Taylor approximation around $\bc=\bc^t$.
As a result we have the following optimization problem:
\begin{align*}
	\text{prox}_{\eta_2 \lambda R(\bx^{t+1},\cdot)}(\boldsymbol{\mu}
	)_j \approx {\argmin_{\bc_j \in \mathbb{R}^{m}} } \left\{\frac{1}{2 \eta_2}(c_j-\mu_j)^{2}+\frac{\lambda}{2} \sum_{i=1}^d \mathbbm{1}(Q_{\bc^t}(\bx^{t+1})_i=c^t_{j}) |c^t_j-x^{t+1}_i| \right. \\
	\left. +(c_j-c^t_j) \frac{\lambda}{2} \sum_{i=1}^d \mathbbm{1}(Q_{\bc^t}(\bx^{t+1})_i=c^t_{j})\frac{\partial |c^t_j-x^{t+1}_i|}{ \partial c^t_j}  \right\}
\end{align*}
In the implementation we take $\frac{\partial |c^t_j-x^{t+1}_i|}{ \partial c^t_j}$ as $1$ if $c^t_j > x^{t+1}_i$,  $-1$ if $c^t_j < x^{t+1}_i$ and $0$ otherwise. Now taking the derivative with respect to $c_j$ and setting it to 0 gives us:
\begin{align*}
	\text{prox}_{\eta_2 \lambda R(\bx^{t+1},\cdot)}(\boldsymbol{\mu})_j&\approx \mu_j -\frac{\lambda\eta_2}{2}(\sum_{i=1}^d \mathbbm{1}(Q_{\bc^t}(\bx^{t+1})_i=c^t_{j}) \mathbbm{1}(x^{t+1}_i>c^{t}_{j})-\sum_{i=1}^d \mathbbm{1}(Q_{\bc^t}(\bx^{t+1})_i=c^t_{j}) \mathbbm{1}(x^{t+1}_i<c^{t}_{j}))
\end{align*}

Proximal map pulls the updated centers toward the median of the weights that are assigned to them.  

\textbf{Using $P \rightarrow \infty$.} In the experiments we observed that using $P \rightarrow \infty$, i.e. using hard quantization function produces good results and also simplifies the implementation. The implications of $P \rightarrow \infty$ are as follows:
\begin{itemize}
	\item We take $\nabla_\bx f(\widetilde{Q}_\bc(\bx))$=0.
	
	\item We take $\nabla_\bc f(\widetilde{Q}_\bc(\bx)) = \nabla_\bc f(Q_\bc(\bx))$, where $\nabla_\bc f(Q_\bc(\bx)) = \begin{bmatrix} \sum_{i=1}^d \frac{\partial f(Q_\bc(\bx)_i)}{\partial Q_\bc(\bx)_i} \mathbbm{1}(Q_\bc(\bx)_i = c_1) \\ \vdots \\ \sum_{i=1}^d \frac{\partial f(Q_\bc(\bx)_i)}{\partial Q_\bc(\bx)_i} \mathbbm{1}(Q_\bc(\bx)_i = c_m) \end{bmatrix}$.
\end{itemize}

\subsection{Details about HyperParameters} \label{appendix_E:Hyperparams}
\textbf{Models and hyperparameters.}
In the centralized case we use ResNet-20 and ResNet-32 \cite{he2015deep}, following \citep{BinaryRelax} and \cite{bai2018proxquant}. In the federated setting we use a 5 layer CNN that was used in \cite{mcmahan2017communicationefficient}.

$\bullet$ For centralized case we use ResNet-20, ResNet-32 \footnote{For the implementation of ResNet models we used the toolbox from \url{https://github.com/akamaster/pytorch_
	resnet_cifar10}.} and employ the learning schedule from \cite{bai2018proxquant}. We use ADAM optimizer with a learning rate of $\eta_1 = 0.01$ and no weight decay.We set a minibatch size of 128, and train for a total of 300 epochs. At the end of 200 epochs we hard quantize the weights and do 100 epochs for fine tuning as in \cite{bai2018proxquant}. As accustomed in quantized network training, we don't quantize the first and last layers, as well as the bias and batch normalization layers. For $\eta_2$ we start with $\eta_2 = 0.0001$ and divide it by 10 at 80'th and 140'th epochs. For $\lambda$ we start with $\lambda = 0.0001$ and increase it every epoch using $\lambda(t) = 0.0001t$. Following, \cite{Yang_2019_CVPR}, \cite{gong2019differentiable}, \cite{zhu2017trained} and many other works; we do layer-wise quantization i.e. each layer has its own set of quantization values. 

$\bullet$ For 5 layer CNN in Federated Setting we use SGD with $\eta_1 = 0.1$, $1e-4$ weight decay, and 0.99 learning rate decay at each epoch. We use batch size of 50, $\tau = 10$ and 350 epochs in total. We do fine tuning after epoch 300. We set $\eta_2 = 0.0001$ and divide it by 10 at epochs 120 and 180, for $\lambda$ we use $\lambda (t) = 1e-6 \cdot t$ and we set $\lambda_p = 0.025$. We use $\eta_3 = 5$.\footnote{To simulate a federated setting we used pytorch.distributed (\url{https://pytorch.org/tutorials/intermediate/dist_tuto.html}) package. }

\textbf{Fine tuning.} We employ a fine tuning procedure similar to \cite{bai2018proxquant}. At the end of the regular training procedure, model weights are hard-quantized. After the hard-quantization, during the fine tuning epochs we let the unquantized parts of the network to continue training (e.g. batch normalization layers) and different from \cite{bai2018proxquant} we also continue to train quantization levels. 

Accompanying Remark \ref{remark:centralized_compare}, we first provide a performance comparison of our centralized training algorithm (Algorithm \ref{algo:centralized}) to related quantized training schemes \citep{BinaryRelax,bai2018proxquant} in Section \ref{subsec:centralized_expts} for classification task on the CIFAR-10 \cite{cifar10} dataset. In the rest of this section, we then validate the performance of our personalized quantization training algorithm QuPeL when learning over heterogeneous client data distributions. We compare QuPeL with FedAvg \cite{mcmahan2017communicationefficient} and local client training (no collaboration between clients) when learning over CIFAR-10 dataset.

%We provide experiments for our centralized training algorithm and its comparison to related quantized training schemes \citep{BinaryRelax,bai2018proxquant} in Appendix \TODO{REF}.

\subsection{Centralized Model Compression Training} \label{subsec:centralized_expts}
%In the centralized case our aim is to compare our centralized method (see Algorithm~\ref{algo:centralized}) to other model compression methods where the centers are not optimized over.
We use ResNet-20 and ResNet-32 \cite{he2015deep} to compare Algorithm \ref{algo:centralized} to BinaryRelax \cite{BinaryRelax} and ProxQuant \cite{bai2018proxquant}.
\begin{table}[H] 
	\centering
	\begin{tabular}{lccl} \toprule 
		& ResNet-20 & ResNet-32 \\ \midrule
		Full Precision & $92.05 $ & $92.95$ \\ 
		ProxQuant (1 bit)  & $90.69 $ & $91.55$\\
		BinaryRelax (1 bit) & $87.82 $ & $90.65$\\
		Our method (1 bit) & $91.17$ & $92.10$\\ 
		Our method (2 bits) & $91.45$ & $92.47$
	\end{tabular} 
	\caption{Test accuracy (in \%) of ResNet-20 and ResNet-32 on Cifar-10 using different methods. We implemented ProxQuant, and put test accuracy of BinaryRelax as reported. } 
	\label{tab:Table centralized}
\end{table}
%In the centralized case we use ResNet-20 and ResNet-32 \cite{he2015deep} to make comparisons, following \citep{BinaryRelax} and \cite{bai2018proxquant}.
%\textbf{Discussion.} 

 Note that when we let $P \rightarrow \infty$ we can see ProxQuant as a special case of our method where the optimization is only over the model parameters. As a result, from Table~\ref{tab:Table centralized} we can infer how much test accuracy we gain from optimizing over the quantization levels. Compared to ProxQuant we observe, $0.38\%$ increase for ResNet-20 and $0.55\%$ increase for ResNet-32. This indicates that, indeed, optimizing over centers provides a gain. As expected, increasing the number of bits further improves the performance of our method.

\subsection{Personalized Quantization for Federated Learning} 
\textbf{Heterogeneity model.} We simulate a heterogeneous setting using pathological non-IID setup similar to recent works \cite{zhang2021personalized,dinh2020personalized}. In particular, we randomly assign only $4$ classes (out of 10) to each client and sample both training and test data from assigned classes while ensuring all clients have same amount of data. For all the plots, we average over three runs for each algorithm, choosing different class assignments in each run.
\begin{table*}[h] 
	\centering
	\begin{threeparttable}
		\begin{tabular}{lccccccl}\toprule
			Avg. No. of Bits	& F. P. (32 bits) & 3 bits &2.75 bits & 2.5 bits & 2.25 bits & 2 bits \\ \midrule
			FedAvg  & $64.31 \pm 0.48$ & - & - & - & - & -\\
			Local Training & $78.42 \pm 0.26$ \tnote{*}& $77.55 \pm 0.24$ &  $77.21 \pm 0.13$ & $77.15  \pm 0.16$  & $76.77 \pm 0.32$ & $76.46 \pm 0.02$ \\
			QuPeL & $79.94 \pm 0.12$ \tnote{*} & $79.50 \pm 0.50$ & $79.10 \pm 0.40$ & $78.74 \pm 0.03$  & $78.57 \pm 0.51$  & $77.88 \pm 0.16$
		\end{tabular}
		\begin{tablenotes}
			\item[*] %This corresponds to doing a full precision training without the distance term $R(\bx,c)$ but with only personalization term for regularization.
			see Footnote \ref{foot:qupel_FP}
		\end{tablenotes}
	\end{threeparttable}
	\caption{Test accuracy (in \%) of different methods under non-IID setting with 20 clients. }
	\label{tab:Table pers}
\end{table*}

\begin{figure*}[h]
	\centering
	\includegraphics[scale=0.33]{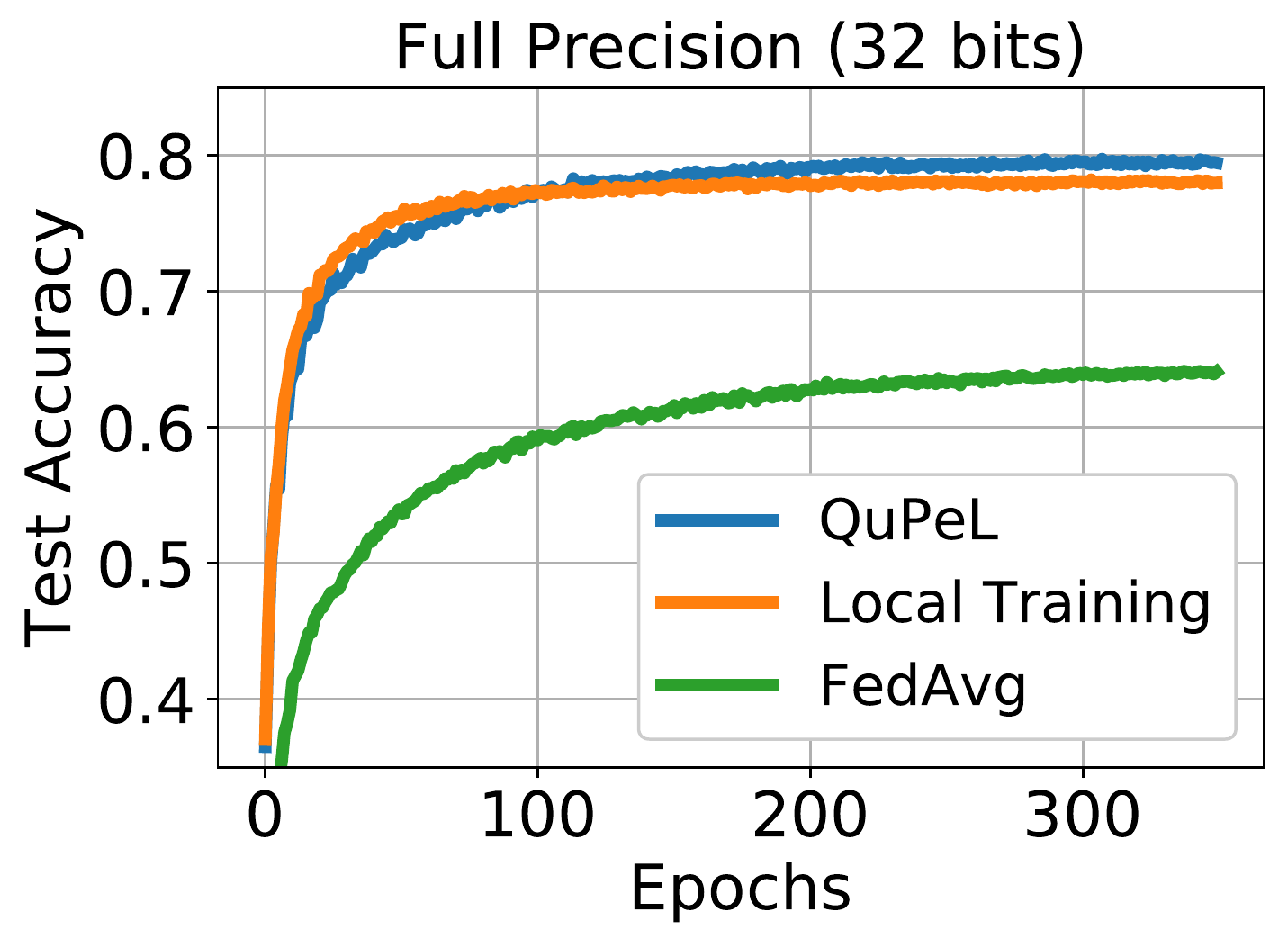}
	\includegraphics[scale=0.33]{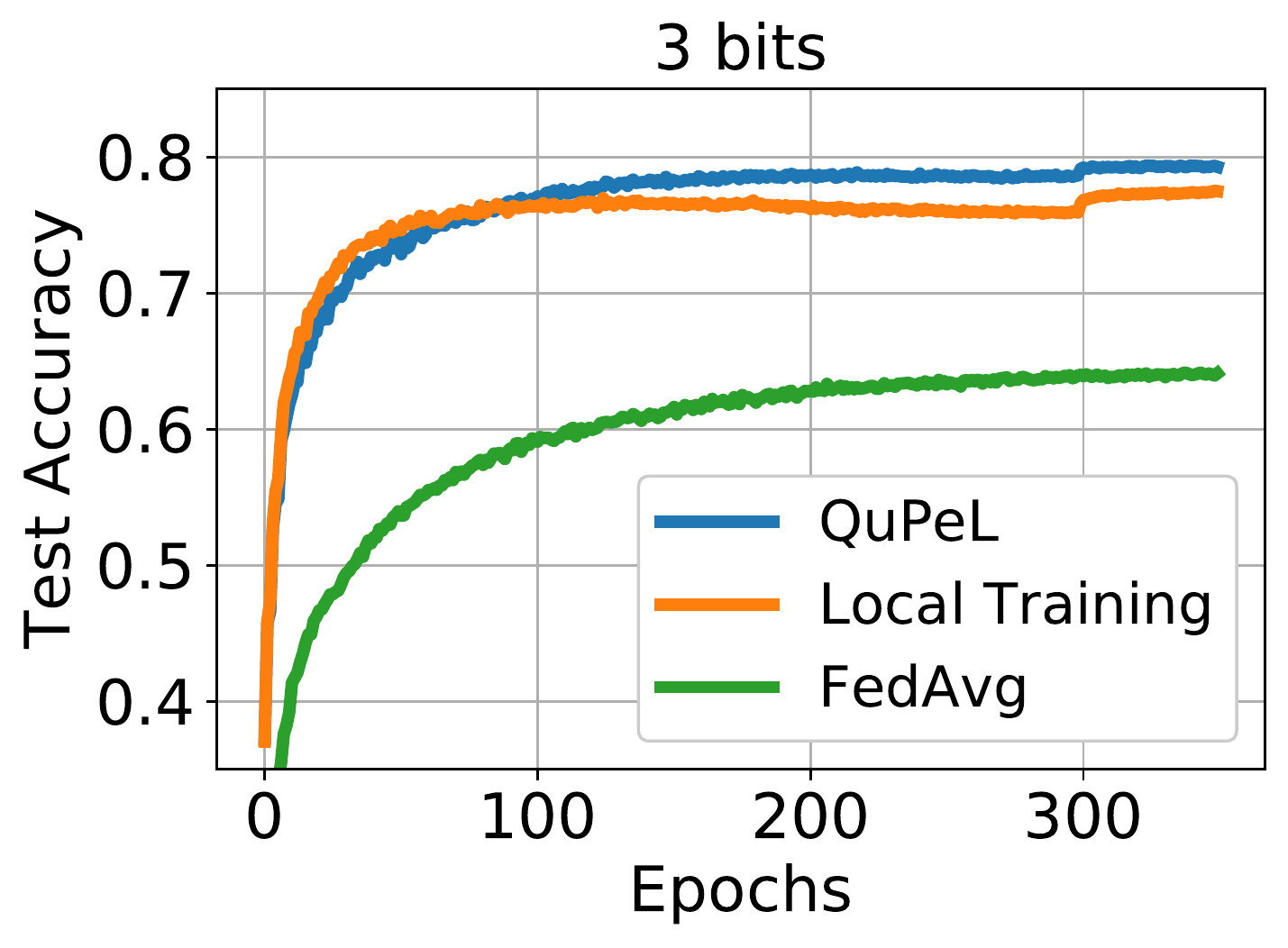}
	\includegraphics[scale=0.33]{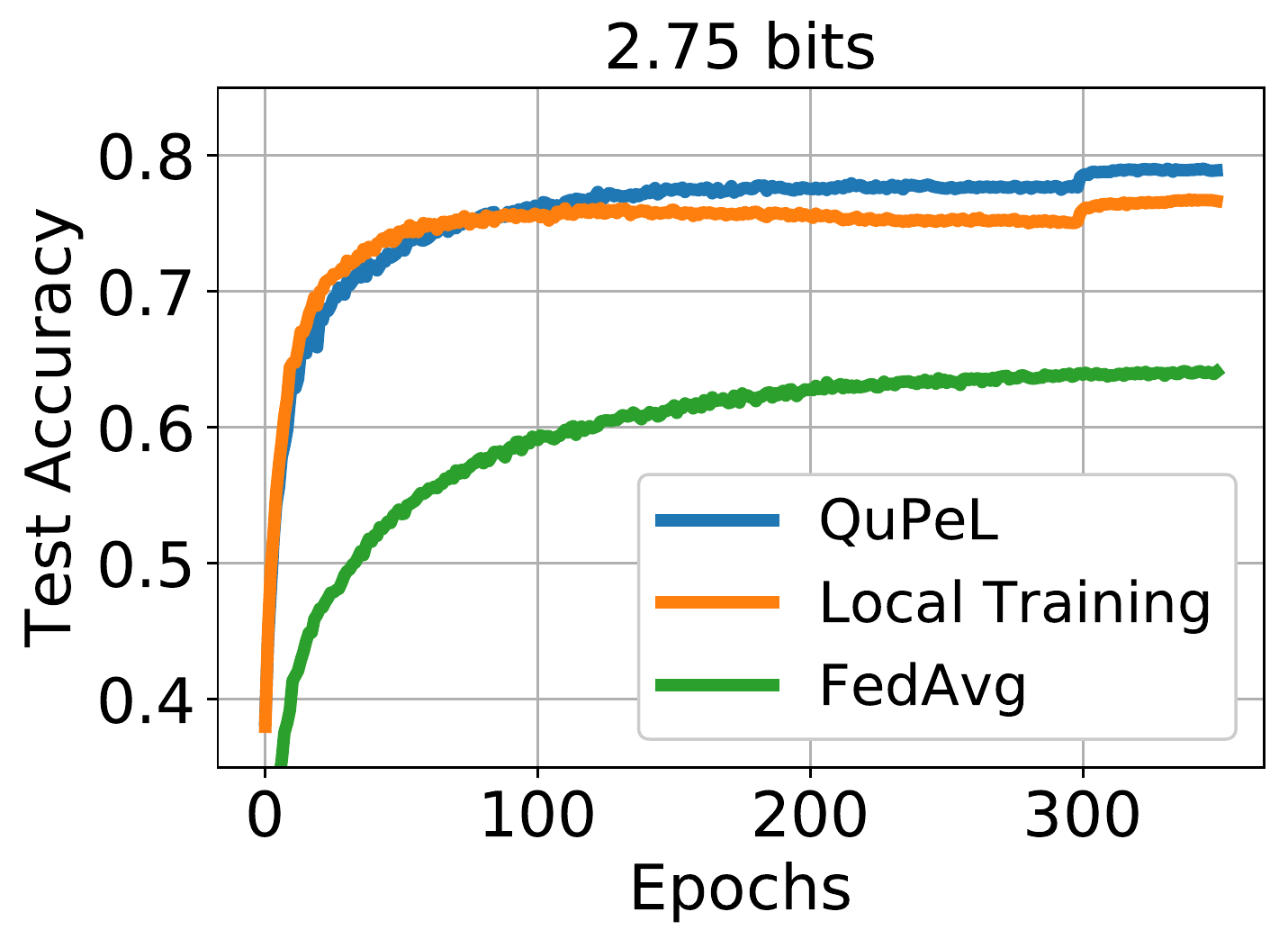}
	\includegraphics[scale=0.33]{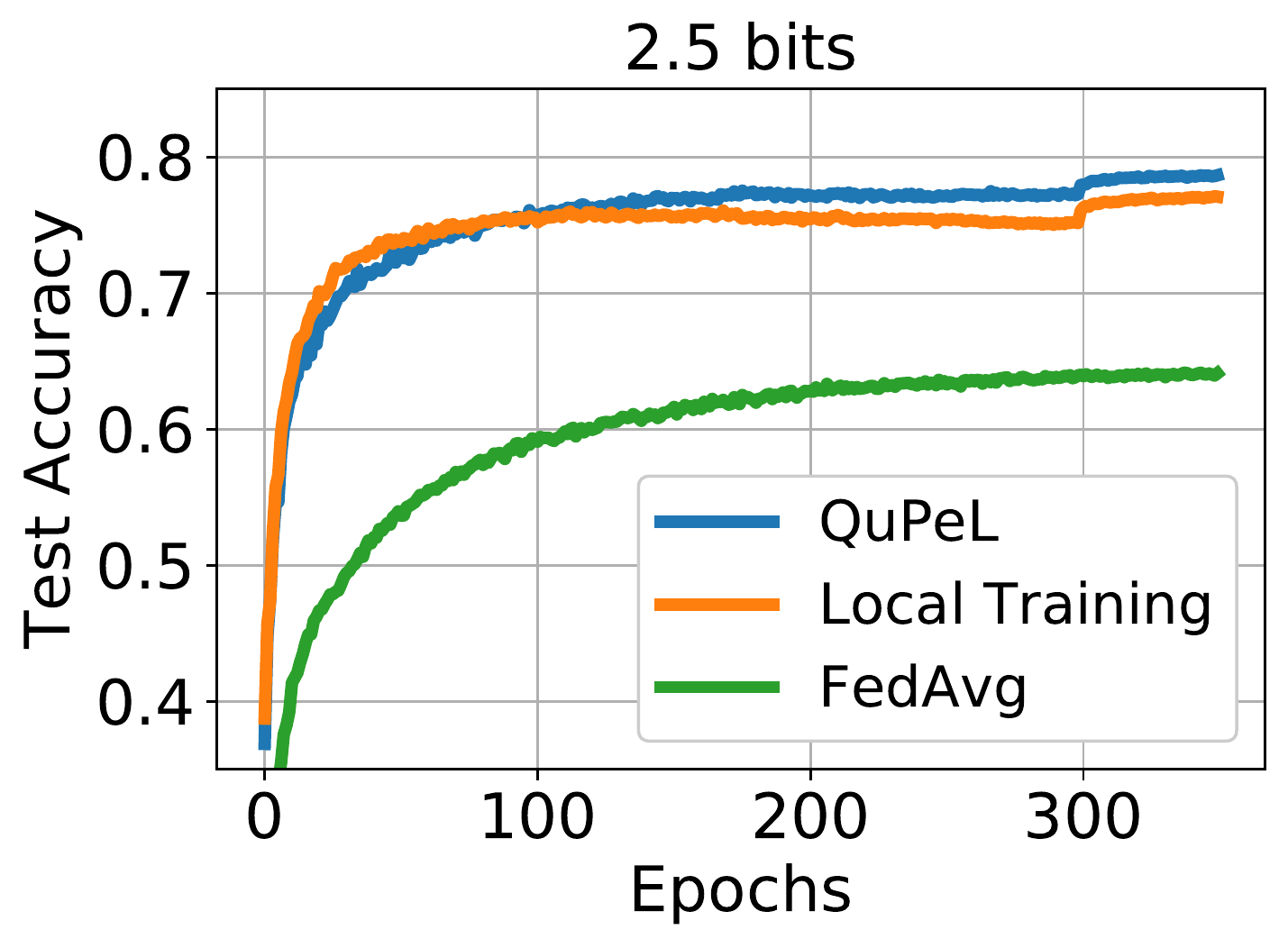}
	\includegraphics[scale=0.33]{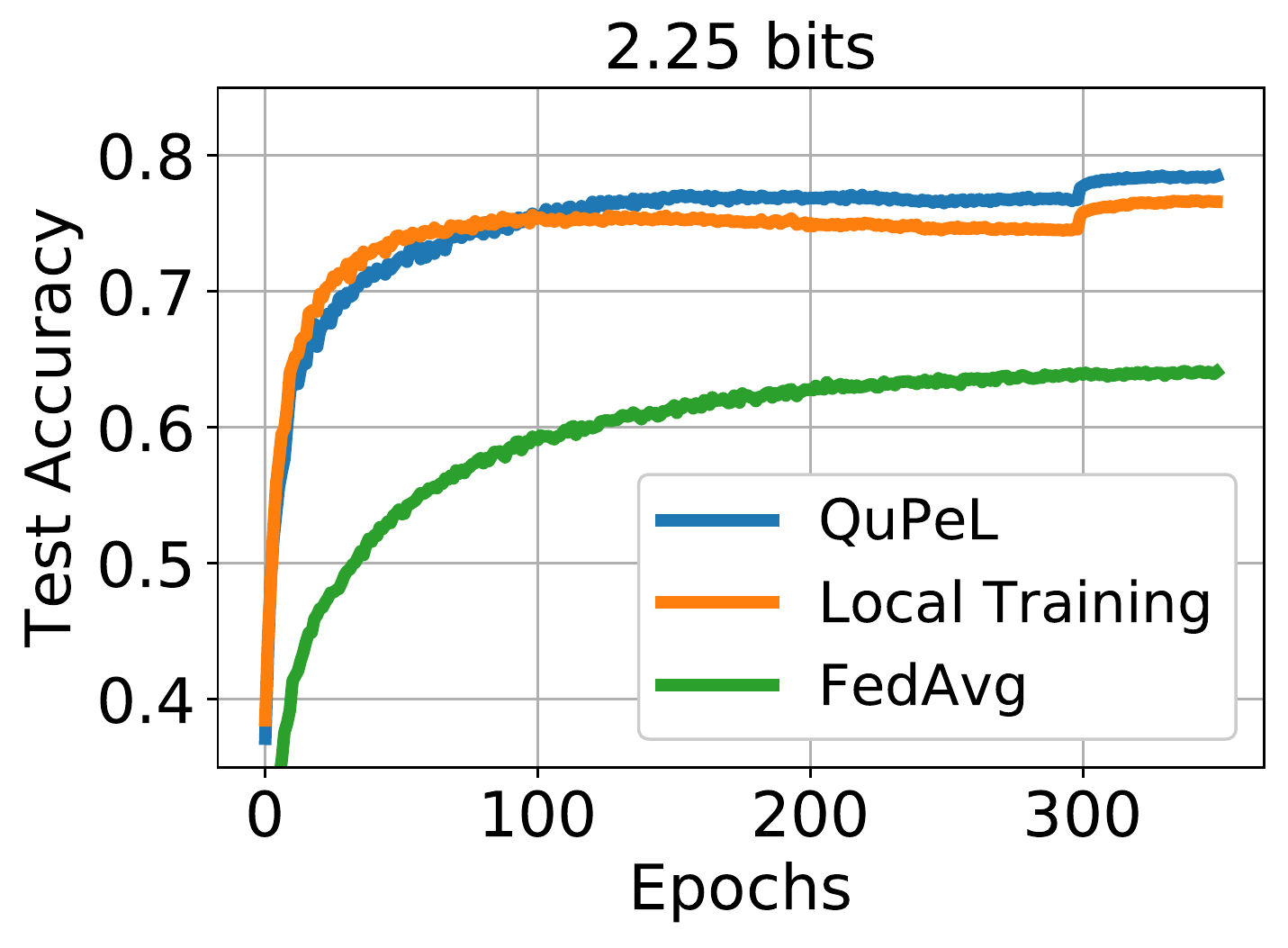}
	\includegraphics[scale=0.33]{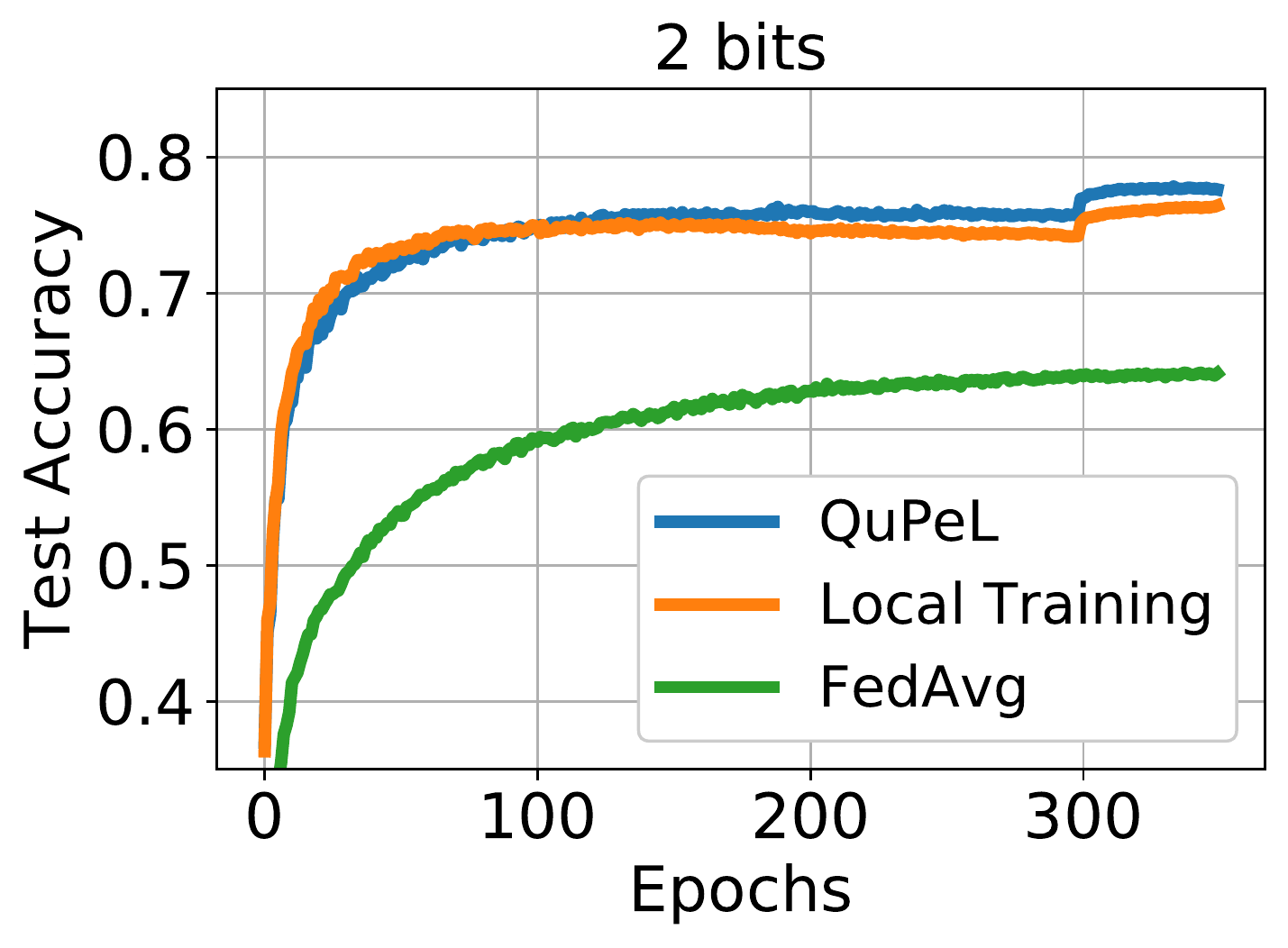}
	\caption{Average test accuracy across clients compared for each scheme for different precision of learned quantization. For QuPeL and Local Training, results are reported for quantized models (except the first plot). For FedAvg we always use full precision global model.}
	\label{fig:figure pers}
\end{figure*}
% We set $\eta_2 = 0.0001$ and divide it by 10 at epochs 120 and 180, for $\lambda$ we use $\lambda (t) = 10^{-6} \cdot t$ and we set $\lambda_p = 0.025$. We use $\eta_3 = 5$. We initialize the network randomly, and initialize the centers uniformly according to the weights. Following \cite{Yang_2019_CVPR,gong2019differentiable,zhu2017trained}, we do layer-wise quantization i.e. each layer has its own set of quantization values. As accustomed in quantized network training, we don't quantize the first and last layers, as well as the biases. 

%\subsection{Results}\label{sec:numerics}

%
%In this part, first we briefly explain our findings for centralized case that illustrate the effectiveness of our quantization algorithm. And then, we move on to personalized case and show quantized and personalized scheme indeed works well in a heterogeneous setting.
%
%\textbf{Centralized Quantized Training} In the centralized case we compare our results with \citep{BinaryRelax} and \cite{bai2018proxquant} for ResNet-20 and ResNet-32 on Cifar-10. As noted in Section~\ref{sec:centralized}, ProxQuant can be seen as a specific case of our centralized method where the centers do not get updated. We observe that updating centers significantly improves the performance. We give the detailed results in Appendix~\ref{appendix_E:Addtl Results Centralized} due to lack of space.
%
%

\textbf{Results.} 
We state our results in Table~\ref{tab:Table pers}, and provide comparison of average test accuracy across clients in Figure~\ref{fig:figure pers} for the different schemes. We compare the schemes for four different cases based on allowed precision for the quantized models: {\sf (i)} Full precision (32 bit),\footnote{\label{foot:qupel_FP}Full Precision refers to the setting where we discard the terms in the objective function related to Model Compression, and optimize the full precision personalized/local models with SGD.} %optimizing the distance term $R(\bx,\bc)$ but with only the penalty term against the global model.}}
{\sf (ii)} (3 bits): each client learns a 3 bit quantization, {\sf (iii)} (2.75 bits): fifteen clients learn a 3 bits other five learn a 2 bits quantization {\sf (iv)} (2.5 bits): ten clients learn a 2 bits and the other ten learn 3 bits quantization, {\sf (v)} (2.25 bits): five clients learn a 3 bits and the rest fifteen learn a 2 bits quantization, {\sf (vi)} (2 bits): each client learns a 2 bit quantization. 
In cases {\sf (ii)}-{\sf (vi)} for local training, we provide accuracies for quantized models learned using our centralized training scheme (Algorithm \ref{algo:centralized}) without collaboration for a fair comparison.
For full precision (32 bit) training, `QuPeL' (see Footnote \ref{foot:qupel_FP}) has 1.5\% more accuracy than local training on account of collaboration between clients, while having 15\% more accuracy than FedAvg on account of learning personalized models for heterogeneous clients. For learning quantized models (in 3 bits, 2.75 bits, 2.5 bits, 2.25 bits 2 bits cases), our proposed algorithm QuPeL outperforms local training around $1.5-2\%$ on account of collaboration between clients.\footnote{Note that as reported in \cite{zhang2021personalized}, in many scenarios local training outperforms personalized training.}

%%%%% TODO %%
%\TODO{Appendix: In heterogeneous settings the disadvantage for FedAvg is that it tries to combine models that possibly converged to very distinct points. Consequently, taking average of models results in a model that performs poorly. On the contrary, in heterogeneous settings Local Training of clients can be powerful especially given enough data samples. This is because a client needs to be able to distinguish between less number of classes. As a result, there is no guarantee that collaboration will help in every heterogeneous setting.}

\textbf{Discussion.}
The above comparison results give insight into two salient features of our training algorithm QuPeL: 

\textit{Model compression}: When learning personalized models, comparison of QuPeL for full precision (32bits) and quantized training shows that using a quantized model for inference does not significantly affect performance. This means that QuPeL can provide personalized, compressed models for deployment which effectively match full precision models in test performance. 

\textit{Collaboration}: Comparison of QuPeL with local training (where clients do not collaborate) shows that collaboration between participating clients (as in QuPeL) can increase the test performance of the learned models. 
Thus, QuPeL allows clients to leverage information from each other in cases when local data is not sufficient, which can improve the test performance of the learned personalized models.

%Firstly, collaboration among clients should help in generalization. And secondly, quantization should not hurt the learning process drastically. In Table \ref{tab:Table pers}, we exactly observe those two points. We see that for any choice of number of bits our method consistently outperforms Local training and FedAvg. Moreover, the loss due to more aggressive quantization is insignificant. In Figure \ref{fig:figure pers} we give the resulted test accuracy training loss plots for quantized models comparing aforementioned methods with different quantization levels. In Figure \ref{fig:figure pers} we see that the local training achieves a loss that is less than QuPeL, however, we observe that QuPeL achieves a higher test accuracy. This implies that collaboration helps clients to have a model that has better generalization.
%
\begin{table}[h] 
	\centering
	\begin{tabular}{lcl} \toprule 
		&Test Accuracy  \\ \midrule
		%QuPeL 2 Bits & $\%77.88 \pm 0.16 $  \\ 
		%QuPeL 2.5 Bits  & $\%78.74 \pm 0.03 $\\
		Collaboration with 2 Bit clients & $78.17 \pm 0.11 $\\
		Collaboration with 3 Bit clients & $78.79 \pm 0.12$ 
	\end{tabular} 
	\caption{Average Test accuracy (in \%) of Control Clients. } 
	\label{tab:Table collab}
\end{table} 
\textbf{Collaboration with resource rich clients.} We demonstrate that for clients with scarce resources, it is advantageous to collaborate with the clients with more resources in terms of having finer quantized models. To investigate this we choose a subset of 10 out of 20 clients and constraint them to have 2 bits; we call these 10 clients `Control Clients'. Then we examine their average test accuracy in two cases, depending on whether the other 10 clients are restricted to have 2-bit or 3-bit quantizers. In Table \ref{tab:Table collab} we observe an increase of $0.62\%$ in test accuracy for the Control Clients. In case of collaboration with full precision clients there are further small gains. This implies, clients with scarce resources can take advantage of clients with rich resources.

%\subsection{Towards Experiments}
%
%In the experiments we consider $l_1$-loss for the distance function $R(\bx,\bc)$. In other words, $R(\bx,\bc)=\{\min_\bz \frac{1}{2} \|\bz-\bx\|_1: \forall j \ z_j \in \calC\}$. As a result, $prox_{\eta_1 \lambda R(\cdot,\bc)}(\cdot)$ becomes the soft thresholding operator:
%\begin{align}
%	prox_{\eta_1 \lambda R(\cdot,\bc)}(\by)_i = 
%	\begin{cases}
%		y_i-\frac{\lambda\eta_1}{2}, \quad \text{if } y_i \geq Q_\bc(\by)_i+\frac{\lambda\eta_1}{2} \\
%		y_i+\frac{\lambda\eta_1}{2}, \quad \text{if } y_i \leq Q_\bc(\by)_i-\frac{\lambda\eta_1}{2} \\
%		Q_\bc(\by)_i, \quad \text{otherwise} 
%	\end{cases}
%\end{align}
%For the second type of update for our choice of $R$ the proximal mapping cannot be exactly computed. Hence, we approximate $prox_{\eta_2 \lambda R(\bx,\cdot)}(\cdot)$ assuming $\bc^{t+1}$ will be in a neighborhood of $\bc^t$. We take the gradient of $R(\bx_{t+1},\bc)$ at $\bc=\bc_t$ while finding the optimal point. As a result we have the following equation:
%\begin{align*}
%	prox_{\eta_2 \lambda R_{\bx}}( \mathbf{\mu})_j&\approx \mu_j \\
%	& -\frac{\lambda\eta_2}{2}(\sum_{i=1}^d \mathbbm{1}(Q_{\bc_t}(\bx)_i=c^t_{j}) \mathbbm{1}(x_i>c^t_{j})\\
%	& -\sum_{i=1}^d \mathbbm{1}(Q_{\bc_t}(\bx)_i=c^t_{j}) \mathbbm{1}(x_i<c^t_{j}))
%\end{align*}
%Proximal map pulls the updated centers toward the median of the weights that are assigned to them. This update corresponds to optimizing the first order Taylor approximation around $\bc = \bc^t $. See \cred{Appendix} for the derivation and more comments on the setup. 
%}
}
%%------------------------------------------------------------------

%%
%% CONCLUSION
\section{Conclusion} \label{sec:conclusion}
In this paper,  we propose a new problem formulation to enable personalized model compression in a  federated setting, where clients also learn different personalized models. For this, we propose and analyze QuPeL, a relaxed optimization framework to learn personalized quantization levels as well as the model parameters. We give first convergence analysis on the QuPeL through this relaxed optimization problem. We observe that QuPeL recovers the convergence rates in the related works on personalized learning, which utilize full precision personalized models. Numerically, we show that model compression does not significantly degrade the performance and collaboration increases the test performance compared to local training. An intriguing preliminary experimental observation is that resource-starved clients can benefit from collaborating with resource rich clients, perhaps promoting "equity". To conclude, our method demonstrates a good numerical performance together with convergence guarantees while accounting for data and resource heterogeneity among clients.

%\newpage
\bibliography{bibliography}
\bibliographystyle{icml2021}
%\newpage

%%% Appendix starts
\onecolumn
\renewcommand{\thesection}{\Alph{section}}
\appendix
%\setcounter{equation}{0}
%%------------------------------------------------------------------

\section{Preliminaries}\label{appendix:Preliminaries}
\allowdisplaybreaks{
%\textbf{Notation. } 
\subsection{Notation}
\begin{itemize}
\item Given a composite function $g(\bx,\by)$ we will denote $\nabla g(\bx,\by)$  or $\nabla_{(\bx,\by)} g(\bx,\by)$   as the gradient; $\nabla_\bx g(\bx,\by)$ and $\nabla_\by g(\bx,\by)$ as the partial gradients with respect to $\bx$ and $\by$. 
\item For a vector $\bu$, $\|\bu\|$ denotes the $\ell_2$-norm $\|\bu\|_2$. For a matrix $\bA$, $\|\bA\|_F$ denotes the Frobenius norm.
\item Unless otherwise stated, for a given vector $\bx$, $x_i$ denotes the $i'th$ element in vector $\bx$; and $\bx_i$ denotes that the vector belongs to client $i$. Furthermore, $\bx^t_i$ denotes a vector that belongs to client $i$ at time $t$.
\end{itemize}

\subsection{Alternating Proximal Steps}
We define the following functions: $f:\mathbb{R}^d \rightarrow \mathbb{R}, \widetilde{Q}: \mathbb{R}^{d+m} \rightarrow \mathbb{R}^d$, and we also define $h(\bx,\bc) = f(\widetilde{Q}_\bc(\bx)), h:\mathbb{R}^{d+m} \rightarrow \mathbb{R}$ where $\bx \in \mathbb{R}^d \text{ and } \bc \in \mathbb{R}^m$. Note that here $\widetilde{Q}_\bc(\bx)$ denotes  $\widetilde{Q}(\bx,\bc)$. Throughout our paper we will use $\widetilde{Q}_\bc(\bx)$ to denote  $\widetilde{Q}(\bx,\bc)$, in other words both $\bc$ and $\bx$ are inputs to the function $\widetilde{Q}_\bc(\bx)$. We propose an alternating proximal gradient algorithm. Our updates are as follows:
%Now we define:
%\begin{align}
%    &F_\lambda(\bx^{t},\bc^{t}) = f(\bx^{t}) + f(\widetilde{Q}_{\bc^{t}}(\bx^{t})) + \lambda R(\bx^{t},\bc^{t})    
%\end{align}
%
%For the centralized training part of the paper $\bx^{t}$ will denote the model parameters at time $t$ and $\bc^{t}$ will denote the vector of centers at time $t$ according to the updates \eqref{updates}, note, $\bx^{t} \in \mathbb{R}^d, \bc^{t} \in \mathbb{R}^m, \forall t $. Here for a vector $a$, $(a)_i$ denotes the $i$'th element of the vector. 
%To solve the empirical minimization problem $\min_{\bx,\bc} F_{\lambda}(\bx,\bc)$ 

\begin{align} \label{updates}
    &\bx^{t+1} = \text{prox}_{\eta_1\lambda R(\cdot,\bc^{t})}(\bx^{t} - \eta_1 \nabla f(\bx^{t})-\eta_1 \nabla_{\bx^{t}} f(\widetilde{Q}_{\bc^{t}}(\bx^{t})) )\\
    &\bc^{t+1} = \text{prox}_{\eta_2\lambda R(\bx^{t+1},\cdot)}(\bc^{t} - \eta_2 \nabla_{\bc^{t}} f(\widetilde{Q}_{\bc^{t}}(\bx^{t+1}))) \notag
\end{align}

For simplicity we assume the functions in the objective function are differentiable, however, our analysis could also be done using subdifferentials.

Our method is inspired by \citep{Bolte13} where the authors introduce an alternating proximal minimization algorithm to solve a broad class of non-convex problems as an alternative to coordinate descent methods. In this work we construct another optimization problem that can be used as a surrogate in learning quantized networks where both model parameters and quantization levels are subject to optimization. In particular, \cite{Bolte13} considers a general objective function of the form $F(\bx,\by) = f(\bx)+g(\by)+\lambda H(\bx,\by)$, whereas, our objective function is tailored for learning quantized networks: $F_\lambda(\bx,\bc) = f(\bx)+f(\widetilde{Q}_{\bc}(\bx))+\lambda R(\bx,\bc)$. Furthermore, they consider updates where the proximal mappings are with respect to functions $f,g$, whereas in our case the proximal mappings are with respect to the distance function $R(\bx,\bc)$ to capture the soft projection. 

\subsection{Lipschitz Relations}
In this section we will use the assumptions \textbf{A.1-5} and show useful relations for partial gradients derived from the assumptions. We have the following gradient for the composite function:

\begin{align}
	\nabla_{(\bx,\bc)} h(\bx,\bc) = \nabla_{(\bx,\bc)} f(\widetilde{Q}_\bc(\bx)) = \nabla_{(\bx,\bc)} \widetilde{Q}_\bc(\bx) \nabla_{\widetilde{Q}_\bc(\bx)} f(\widetilde{Q}_\bc(\bx)) 
\end{align}

where $dim(\nabla h(\bx,\bc)) = (d+m) \times 1, \ dim(\nabla_{\widetilde{Q}_\bc(\bx)} f(\widetilde{Q}_\bc(\bx))) = d \times 1, \ dim(\nabla \widetilde{Q}_\bc(\bx)) = (d+m) \times d $. Note that the soft quantization functions of our interest are elementwise which implies $\frac{\partial \widetilde{Q}_\bc(\bx)_i}{\partial x_j}=0$ if $i \neq j$. In particular,
for the gradient of the quantization function we have,

\begin{align}\label{lipschitz_relations2}
	\nabla_{(\bx,\bc)} \widetilde{Q}_\bc(\bx) = \begin{bmatrix} \frac{\partial \widetilde{Q}_\bc(\bx)_1}{\partial x_1} & 0 & \hdots  \\ 0 & \frac{\partial \widetilde{Q}_\bc(\bx)_2}{\partial x_2} & 0\hdots \\ \vdots & \vdots & \vdots \\ \frac{\partial \widetilde{Q}_\bc(\bx)_1}{\partial c_1} & \frac{\partial \widetilde{Q}_\bc(\bx)_2}{\partial c_1} & \hdots \\ \vdots & \vdots & \vdots \\ \frac{\partial \widetilde{Q}_\bc(\bx)_1}{\partial c_m} &  \frac{\partial \widetilde{Q}_\bc(\bx)_2}{\partial c_m} & \hdots \end{bmatrix}
\end{align}

Moreover for the composite function we have,

\begin{align} \label{lipschitz_relations1}
	\nabla h(\bx,\bc) = \begin{bmatrix} &  \frac{\partial f}{\partial \widetilde{Q}_\bc(\bx)_1} \frac{\partial \widetilde{Q}_\bc(\bx)_1}{\partial  x_1} &+  0 & +  \hdots  \\ & 0 &+  \frac{\partial f}{\partial \widetilde{Q}_\bc(\bx)_2} \frac{\partial \widetilde{Q}_\bc(\bx)_2}{\partial  x_2} &+  \hdots \\ &\vdots \\& \frac{\partial f}{\partial \widetilde{Q}_\bc(\bx)_1} \frac{\partial \widetilde{Q}_\bc(\bx)_1}{\partial  c_1} &+ \frac{\partial f}{\partial \widetilde{Q}_\bc(\bx)_2} \frac{\partial \widetilde{Q}_\bc(\bx)_2}{\partial  c_1} &+ \hdots \\ &\vdots \\ &\frac{\partial f}{\partial \widetilde{Q}_\bc(\bx)_1} \frac{\partial \widetilde{Q}_\bc(\bx)_1}{\partial  c_m} &+ \frac{\partial f}{\partial \widetilde{Q}_\bc(\bx)_2} \frac{\partial \widetilde{Q}_\bc(\bx)_2}{\partial  c_m} &+ \hdots \end{bmatrix} 
\end{align}

In \eqref{lipschitz_relations1} and \eqref{lipschitz_relations2}  we use $x_i$, $c_j$ to denote $(\bx)_i$ and $(\bc)_j$ (i'th, j'th element respectively) for notational simplicity. Now, we prove two claims that will be useful in the main analysis.
\begin{claim} \label{claim: lqxclaim}
\begin{align*}
\| \nabla_\bx f(\widetilde{Q}_\bc(\bx))-\nabla_\by f(\widetilde{Q}_\bc(\by)) \| =\| \nabla h(\bx,\bc)_{1:d}-\nabla h(\by,\bc)_{1:d} \| \leq  (GL_{Q_1} + G_{Q_1}Ll_{Q_1})\|\bx-\by\|  
\end{align*}
\end{claim}
\begin{proof}
\begin{align*}
	\| \nabla h(\bx,\bc)_{1:d}-\nabla h(\by,\bc)_{1:d} \| &=\|\nabla_\bx f(\widetilde{Q}_{\bc}(\bx))-\nabla_\by f(\widetilde{Q}_{\bc}(\by))\| \\  &= \| \nabla_{\widetilde{Q}_\bc(\bx)} f(\widetilde{Q}_\bc(\bx)) \nabla_\bx \widetilde{Q}_\bc(\bx) - \nabla_{\widetilde{Q}_\bc(\by)} f(\widetilde{Q}_\bc(\by)) \nabla_\by \widetilde{Q}_\bc(\by)\| \\
	&=\| \nabla_{\widetilde{Q}_\bc(\bx)} f(\widetilde{Q}_\bc(\bx)) \nabla_\bx \widetilde{Q}_\bc(\bx)- \nabla_{\widetilde{Q}_\bc(\bx)} f(\widetilde{Q}_\bc(\bx))  \nabla_\by \widetilde{Q}_\bc(\by) \\ & \quad + \nabla_{\widetilde{Q}_\bc(\bx)} f(\widetilde{Q}_\bc(\bx)) \nabla_\by \widetilde{Q}_\bc(\by) -  \nabla_{\widetilde{Q}_\bc(\by)} f(\widetilde{Q}_\bc(\by))  \nabla_\by \widetilde{Q}_\bc(\by)\| \\ &\stackrel{\text{(a)}}{\leq} \| \nabla_{\widetilde{Q}_\bc(\bx)} f(\widetilde{Q}_\bc(\bx)) \| \| \nabla_\bx \widetilde{Q}_\bc(\bx) - \nabla_\by \widetilde{Q}_\bc(\by) \|_F \\ & \quad + \| \nabla_\by \widetilde{Q}_\bc(\by) \|_F \| \nabla_{\widetilde{Q}_\bc(\bx)} f(\widetilde{Q}_\bc(\bx)) - \nabla_{\widetilde{Q}_\bc(\by)} f(\widetilde{Q}_\bc(\by)) \|  \\ & \leq GL_{Q_1}\|\bx-\by\| + G_{Q_1}L\|\widetilde{Q}_\bc(\bx)-\widetilde{Q}_\bc(\by)\| \\ &\leq GL_{Q_1}\|\bx-\by\| + G_{Q_1}Ll_{Q_1}\|\bx-\by\| = (GL_{Q_1} + G_{Q_1}Ll_{Q_1})\|\bx-\by\|    
\end{align*} 
To obtain (a) we have used the fact $\|\bA\bx\|_2 \leq \|\bA\|_F\|\bx\|_2$.
\end{proof}

\begin{claim} \label{claim: lqcclaim}
\begin{align*}
\| \nabla_\bc f(\widetilde{Q}_\bc(\bx))-\nabla_\bd f(\widetilde{Q}_\bd(\bx)) \| = \| \nabla h(\bx,\bc)_{d+1:m}-\nabla h(\bx,\bd)_{d+1:m} \| \leq (GL_{Q_2} + G_{Q_2}Ll_{Q_2})\|\bc-\bd\| 
\end{align*}
\end{claim}
\begin{proof} We can follow similar steps,
\begin{align*}
	\| \nabla h(\bx,\bc)_{d+1:m}-\nabla h(\bx,\bd)_{d+1:m} \|&=\|\nabla_\bc f(\widetilde{Q}_{\bc}(\bx))-\nabla_\bd f(\widetilde{Q}_{\bd}(\bx))\| \\ &= \| \nabla_{\widetilde{Q}_\bc(\bx)} f(\widetilde{Q}_\bc(\bx)) \nabla_\bc \widetilde{Q}_\bc(\bx) - \nabla_{\widetilde{Q}_\bd(\by)} f(\widetilde{Q}_\bd(\by)) \nabla_\bd \widetilde{Q}_\bd(\by)\| \\
	&=\| \nabla_{\widetilde{Q}_\bc(\bx)} f(\widetilde{Q}_\bc(\bx)) \nabla_\bc \widetilde{Q}_\bc(\bx)- \nabla_{\widetilde{Q}_\bc(\bx)} f(\widetilde{Q}_\bc(\bx)) \nabla_\bd \widetilde{Q}_\bd(\bx) \\ & \quad + \nabla_{\widetilde{Q}_\bc(\bx)} f(\widetilde{Q}_\bc(\bx)) \nabla_\bd \widetilde{Q}_\bd(\by) -  \nabla_{\widetilde{Q}_\bd(\bx)} f(\widetilde{Q}_\bd(\by)) \nabla_\bd \widetilde{Q}_\bd(\by)\| \\ &\leq \| \nabla_{\widetilde{Q}_\bc(\bx)} f(\widetilde{Q}_\bc(\bx)) \| \| \nabla_\bc \widetilde{Q}_\bc(\bx) - \nabla_\bd \widetilde{Q}_\bd(\by) \|_F \\ & \quad + \| \nabla_\bd \widetilde{Q}_\bd(\by) \|_F \| \nabla_{\widetilde{Q}_\bc(\bx)} f(\widetilde{Q}_\bc(\bx)) - \nabla_{\widetilde{Q}_\bd(\by)} f(\widetilde{Q}_\bd(\by)) \|  \\ & \leq GL_{Q_2}\|\bc-\bd\| + G_{Q_2}L\|\widetilde{Q}_\bc(\bx)-\widetilde{Q}_\bd(\bx)\| \\ &\leq GL_{Q_2}\|\bc-\bd\| + G_{Q_2}Ll_{Q_2}\|\bc-\bd\| = (GL_{Q_2} + G_{Q_2}Ll_{Q_2})\|\bc-\bd\|     
\end{align*} 
where $\nabla_\bc \widetilde{Q}_\bc(\bx) = \nabla \widetilde{Q}_\bc(\bx)_{(d+1:d+m,:)}$. 
\end{proof}
\subsection{Proof of the Claims Regarding Soft Quantization Function}\label{appendix:Soft Quantization}

{\begin{claim*}[Restating Claim~\ref{claimlQ_1}]
		$\widetilde{Q}_\bc(\bx)$ is $l_{Q_1}$-Lipschitz continuous and $L_{Q_1}$-smooth with respect to $\bx$.
\end{claim*}}
\begin{proof}
First we prove Lipschitz continuity. Note,
		\begin{align} \label{grad1}
		\frac{\partial \widetilde{Q}_\bc(\bx)_i}{\partial x_j} = 
			\begin{cases}
				0, \text{ if } i \neq j\\ 
				P\sum_{j=2}^m(c_j-c_{j-1})\sigma(P(x_i-\frac{c_j + c_{j-1}}{2}))(1-\sigma(P(x_i-\frac{c_j + c_{j-1}}{2}))),  \text{ if } i = j\\ 
			\end{cases}
		\end{align}
		As a result, $\| \frac{\partial \widetilde{Q}_\bc(\bx)_i}{\partial x_j} \| \leq \frac{P}{4}(c_m-c_1)\leq \frac{P}{2}c_{max}$. The norm of the gradient of $\widetilde{Q}_\bc(\bx)_i$ with respect to $x$ is bounded which implies there exists $l^{(i)}_{Q_1}$ such that $\|\widetilde{Q}_\bc(\bx)_i-\widetilde{Q}_\bc(\bx')_i\|\leq l^{(i)}_{Q_1}\|\bx-\bx'\|$; using Fact~\ref{sq:fact5} and the fact that $i$ was arbitrary, there exists $l_{Q_1}$ such that $\|\widetilde{Q}_\bc(\bx)-\widetilde{Q}_\bc(\bx')\|\leq l_{Q_1}\|\bx-\bx'\|$. In other words, $\widetilde{Q}_\bc(\bx)$ is Lipschitz continuous.
		
		For smoothness note that, $\nabla_\bx \widetilde{Q}_\bc(\bx) = \nabla \widetilde{Q}_\bc(\bx)_{1:d,:}$. Now we focus on an arbitrary term of $\nabla_\bx \widetilde{Q}_\bc(\bx)_{j,i}$. From \eqref{grad1} we know that this term is 0 if $i \neq j$, and a weighted sum of product of sigmoid functions if $i = j$. Then, using the Facts~\ref{sq:fact1}-\ref{sq:fact4} the function $\nabla_\bx \widetilde{Q}_\bc(\bx)_{j,i}$ is Lipschitz continuous. Since $i,j$ were arbitrarily chosen, $\nabla_\bx \widetilde{Q}_\bc(\bx)_{j,i}$ is Lipschitz continuous for all $i,j$. Then, by Fact~\ref{sq:fact5},  $\nabla_\bx \widetilde{Q}_\bc(\bx)$ is Lipschitz continuous, which implies that $\widetilde{Q}_\bc(\bx)$ is $L_{Q_1}$-smooth for some coefficient $L_{Q_1} < \infty$.
\end{proof}
{\begin{claim*}[Restating Claim~\ref{claimlQ_2}]
		 $\widetilde{Q}_\bc(\bx)$ is $l_{Q_2}$-Lipschitz continuous and $L_{Q_2}$-smooth with respect to $\bc$.
\end{claim*}}
\begin{proof}
	For Lipschitz continuity we have,
		\begin{align*}
		\frac{\partial \widetilde{Q}_\bc(\bx)_i}{\partial c_j} & = 
				\sigma(P(x_i-\frac{c_j + c_{j-1}}{2}))-\sigma(P(x_i-\frac{c_j + c_{j+1}}{2}))\\
				& \quad +(c_j-c_{j+1})\frac{P}{2}\sigma(P(x_i-\frac{c_j + c_{j+1}}{2}))(1-\sigma(P(x_i-\frac{c_j + c_{j+1}}{2}))) \\
				& \quad - (c_j-c_{j-1})\frac{P}{2}\sigma(P(x_i-\frac{c_j + c_{j-1}}{2}))(1-\sigma(P(x_i-\frac{c_j + c_{j-1}}{2})))
		\end{align*}
		As a result, $\| \frac{\partial \widetilde{Q}_\bc(\bx)_i}{\partial c_j} \| \leq 2 + c_{max}\frac{P}{2}$. Similar to Claim~\ref{claimlQ_1} using the facts that $i$ is arbitrary and the Fact~\ref{sq:fact5}, we find there exists $l_{Q_2}$ such that $\|\widetilde{Q}_\bc(\bx)-\widetilde{Q}_\bd(\bx)\|\leq l_{Q_1}\|\bc-\bd\|$. In other words, $\widetilde{Q}_\bc(\bx)$ is Lipschitz continuous. And for the smoothness, following the same idea from the proof of Claim \ref{claimlQ_2} we find $\widetilde{Q}_\bc(\bx)$ is $L_{Q_2}$-smooth with respect to $\bc$.
\end{proof}
}
%%--------------------------------------------- ---------------------

%%------------------------------------------------------------------

%%------------------------------------------------------------------

%%------------------------------------------------------------------

\section{Omitted Details from Section~\ref{sec:proof_centralized} -- Proof of Theorem~\ref{thm:centralized}}\label{appendix:proof of theorem 1}

First we derive the optimization problems that the alternating updates correspond to. Remember we had the following alternating updates:
\begin{align*}
    &\bx^{t+1} = \text{prox}_{\eta_1\lambda R_{\bc^{t}}}(\bx^{t} - \eta_1 \nabla f(\bx^{t})-\eta_1 \nabla_{\bx^{t}} f(\widetilde{Q}_{\bc^{t}}(\bx^{t})) )\\
    &\bc^{t+1} = \text{prox}_{\eta_2\lambda R_{\bx^{t+1}}}(\bc^{t} - \eta_2 \nabla_{\bc^{t}} f(\widetilde{Q}_{\bc^{t}}(\bx^{t+1})))
\end{align*}
For $\bx^{t+1}$, from the definition of proximal mapping we have:
\begin{align}\label{app:quantized argmin1}
\bx^{t+1}&=\underset{\bx \in \mathbb{R}^{d}}{\arg \min }\left\{\frac{1}{2 \eta_1}\left\|\bx-\bx^{t}+\eta_1 \nabla_{\bx^{t}} f\left(\bx^{t}\right)+\eta_1 \nabla_{\bx^{t}} f(\widetilde{Q}_{\bc^{t}}(\bx^{t}))\right\|_{2}^{2}+\lambda R(\bx,\bc^{t})\right\}\nonumber \\\nonumber
&=\underset{\bx \in \mathbb{R}^{d}}{\arg \min } \Big\{ \left\langle \bx-\bx^{t}, \nabla_{\bx^{t}} f\left(\bx^{t}\right)\right\rangle+\left\langle \bx-\bx^{t}, \nabla_{\bx^{t}} f(\widetilde{Q}_{\bc^{t}}(\bx^{t}))\right\rangle+\frac{1}{2 \eta_1}\left\|\bx-\bx^{t}\right\|_{2}^{2}+\frac{\eta_1}{2}\|\nabla_{\bx^{t}} f(\bx^{t})+\nabla_{\bx^{t}} f(\widetilde{Q}_{\bc^{t}}(\bx^{t}))\|^2 \\\nonumber
&+\lambda R(\bx,\bc^{t})\Big\} \\
&=\underset{\bx \in \mathbb{R}^{d}}{\arg \min }\left\{\left\langle \bx-\bx^{t}, \nabla_{\bx^{t}} f\left(\bx^{t}\right)\right\rangle+\left\langle \bx-\bx^{t}, \nabla_{\bx^{t}} f(\widetilde{Q}_{\bc^{t}}(\bx^{t}))\right\rangle+\frac{1}{2 \eta_1}\left\|\bx-\bx^{t}\right\|_{2}^{2}+\lambda R(\bx,\bc^{t})\right\}
\end{align}
Note, in the third equality we remove the terms that do not depend on $\bx$. Similarly, for $\bc^{t+1}$ we have:
\begin{align} \label{app:quantized argmin2}
	\bc^{t+1}&=\underset{\bc \in \mathbb{R}^{m}}{\arg \min }\left\{\frac{1}{2 \eta_2}\left\|\bc-\bc^{t}+\eta_2 \nabla_{\bc^{t}} f(\widetilde{Q}_{\bc^{t}}(\bx^{t+1}))\right\|_{2}^{2}+\lambda R(\bx^{t+1},\bc)\right\}\nonumber \\ \nonumber
	&= \underset{\bc \in \mathbb{R}^{m}}{\arg \min } \left\{ \left\langle \bc-\bc^{t}, \nabla_{\bc^{t}} f(\widetilde{Q}_{\bc^{t}}(\bx^{t+1}))\right\rangle+\frac{1}{2 \eta_2}\left\|\bc-\bc^{t}\right\|_{2}^{2}+\frac{\eta_2}{2}\|\nabla_{\bc^{t}} f(\widetilde{Q}_{\bc^{t}}(\bx^{t+1}))\|^2 + \lambda R(\bx^{t+1},\bc)  \right\}\\
	&= \underset{\bc \in \mathbb{R}^{m}}{\arg \min }\left\{\left\langle \bc-\bc^{t}, \nabla_{\bc^{t}} f(\widetilde{Q}_{\bc^{t}}(\bx^{t+1}))\right\rangle+\frac{1}{2 \eta_2}\left\|\bc-\bc^{t}\right\|_{2}^{2}+\lambda R(\bx^{t+1},\bc)\right\}
\end{align}
Minimization problems in \eqref{app:quantized argmin1} and \eqref{app:quantized argmin2} are the main problems to characterize the update rules and we use them in multiple places throughout the section.
\subsection{Proof of the Claims}
{\begin{claim*}[Restating Claim \ref{claim: lqxsmooth}]
$f(\bx)+ f(\widetilde{Q}_\bc(\bx))$ is $(L+GL_{Q_1}+G_{Q_1}LL_{Q_1})$-smooth with respect to $\bx$.
\end{claim*}}
\begin{proof}
	From our assumptions, we have $f$ is $L$-smooth. And from Claim \ref{claim: lqxclaim} we have $f(\widetilde{Q}_\bc(\bx))$ is $(GL_{Q_1}+G_{Q_1}LL_{Q_1})$-smooth. Using the fact that if two functions $g_1$ and $g_2$ are $L_1$ and $L_2$ smooth respectively, then $g_1+g_2$ is $(L_1+L_2)$-smooth concludes the proof.
\end{proof}

{\begin{claim*}[Restating Claim \ref{claim:quantization lower bound 1}]
Let
\begin{align*}
A(\bx^{t+1}) &:= \lambda R(\bx^{t+1},\bc^{t})+\left\langle \nabla f(\bx^{t}), \bx^{t+1}-\bx^{t}\right\rangle 
+ \left\langle \nabla_{\bx^{t}} f(\widetilde{Q}_{\bc^{t}}(\bx^{t})), \bx^{t+1}-\bx^{t}\right\rangle \notag + \frac{1}{2\eta_1}\|\bx^{t+1}-\bx^{t}\|^2 \notag \\
A(\bx^{t}) &:= \lambda R(\bx^{t},\bc^{t}).
\end{align*} 
Then $A(\bx^{t+1})\leq A(\bx^{t})$.
\end{claim*}}
\begin{proof}
	Let $A(\bx)$ denote the expression inside the $\arg\min$ in \eqref{app:quantized argmin1} and we know that \eqref{app:quantized argmin1} is minimized when $\bx=\bx^{t+1}$. So we have $A(\bx^{t+1})\leq A(\bx^t)$. This proves the claim.
\end{proof}
{\begin{claim*}[Restating Claim \ref{claim:quantization lower bound 2}]
	Let
	\begin{align*}
		B(\bc^{t+1}) &:= \lambda R(\bx^{t+1},\bc^{t+1}) + \left\langle \nabla_{\bc^{t}} f(\widetilde{Q}_{\bc^{t}}(\bx^{t+1})), \bc^{t+1}-\bc^{t}\right\rangle \notag + \frac{1}{2\eta_1}\|\bc^{t+1}-\bc^{t}\|^2 \notag \\
		B(\bc^{t}) &:= \lambda R(\bx^{t+1},\bc^{t}).
	\end{align*} 
	Then $B(\bc^{t+1})\leq B(\bc^{t})$.	
\end{claim*}}
\begin{proof}
	Let $B(\bc)$ denote the expression inside the $\arg\min$ in \eqref{app:quantized argmin2} and we know that \eqref{app:quantized argmin2} is minimized when $\bc=\bc^{t+1}$. So we have $B(\bc^{t+1})\leq B(\bc^t)$. This proves the claim.
\end{proof}

%%------------------------------------------------------------------

%%------------------------------------------------------------------

\section{Omitted Details from Section~\ref{sec:proof_personalized} -- Proof of Theorem~\ref{thm:personalized}}\label{appendix:proof of theorem 2}
\allowdisplaybreaks{
Again we begin with deriving the optimization problems that alternating proximal updates correspond to.
The update rule for $\bx^t_{i}$ is
\begin{align}
	\bx^{t+1}_{i} &= \text{prox}_{\eta_1 \lambda R_{\bc^t_{i}}}(\bx^t_{i}-\eta_1 \nabla f_i(\bx^t_{i})- \eta_1 \nabla_{\bx^t_{i}} f_i(\widetilde{Q}_{\bc^t_{i}}(\bx^t_{i}))-\eta_1 \lambda_p(\bx^t_{i}-\bw^{t}_{i}) ) \notag \\
	&= \underset{\bx \in \mathbb{R}^{d}}{\arg \min } \left\{\frac{1}{2\eta_1}\left\|\bx-\bx^t_{i}+\eta_1 \nabla f_i(\bx^t_{i}) + \eta_1 \nabla_{\bx^t_{i}} f_i(\widetilde{Q}_{\bc^t_{i}}(\bx^t_{i})) + \eta_1 \lambda_p(\bx^t_{i}-\bw^{t}_{i}) \right\|^2  + \lambda R(\bx,\bc^t_{i})\right\} \notag \\
	&= \underset{\bx \in \mathbb{R}^{d}}{\arg \min }\left\{\left\langle \bx-\bx^t_{i}, \nabla f_i\left(\bx^t_{i}\right)\right\rangle+\left\langle \bx-\bx^t_{i}, \nabla_{\bx^t_{i}} f_i(\widetilde{Q}_{\bc^t_{i}}(\bx^t_{i}))\right\rangle +\left\langle \bx-\bx^t_{i}, \lambda_p (\bx^t_{i}-\bw^{t}_{i})\right\rangle \right. \notag \\
	&\hspace{9cm} \left. +\frac{1}{2 \eta_1}\left\|\bx-\bx^t_{i}\right\|_{2}^{2}+\lambda R(\bx,\bc^t_{i})\right\} \label{app:lower-bounding-interim1}
\end{align}
and the update rule for $\bc^t_{i}$ is
\begin{align}
	\bc^{t+1}_{i} &= \text{prox}_{\eta_2 \lambda R_{\bx^{t+1}_{i}}}(\bc^t_{i}-\eta_2 \nabla_{\bc^t_{i}} f_i(\widetilde{Q}_{\bc^t_{i}}(\bx^t_{i}))) \notag \\
	&= \underset{\bc \in \mathbb{R}^{m}}{\arg \min } \left\{\frac{1}{2\eta_2}\left\|\bc-\bc^t_{i}+ \eta_2 \nabla_{\bc^t_{i}} f_i(\widetilde{Q}_{\bc^t_{i}}(\bx^{t+1}_{i})) \right\|^2  + \lambda R(\bx^{t+1}_{i},\bc)\right\} \notag \\
	&= \underset{\bc \in \mathbb{R}^{m}}{\arg \min }\left\{\left\langle \bc-\bc^t_{i}, \nabla_{\bc^t_{i}} f_i(\widetilde{Q}_{\bc^t_{i}}(\bx^{t+1}_{i}))\right\rangle  \right.  \left. +\frac{1}{2 \eta_2}\left\|\bc-\bc^t_{i}\right\|_{2}^{2}+\lambda R(\bx^{t+1}_{i},\bc)\right\} \label{app:lower-bounding-interim2}
\end{align}
\subsection{Proof of the Claims}
{\begin{claim*}[Restating Claim~\ref{thm2:claim lqxsmoothness}]
	$f_i(\bx)+ f_i(\widetilde{Q}_\bc(\bx))+\frac{\lambda_p}{2}\|\bx-\bw\|^2$ is $(\lambda_p+L+GL_{Q_1}+G_{Q_1}LL_{Q_1})$-smooth with respect to $\bx$.
\end{claim*}}
\begin{proof}
	From our assumptions, we have $f_i$ is $L$-smooth. We know $\frac{\lambda_p}{2}\|\bx-\bw\|^2$ is $\lambda_p$-smooth with respect to $\bx$. And applying the Claim \ref{claim: lqxclaim} to each client separately gives that $f_i(\widetilde{Q}_\bc(\bx))$ is $(GL_{Q_1}+G_{Q_1}LL_{Q_1})$-smooth. Using the fact that if two functions $g_1$ and $g_2$ are $L_1$ and $L_2$ smooth respectively, then $g_1+g_2$ is $(L_1+L_2)$-smooth concludes the proof.
\end{proof}
{\begin{claim*}[Restating Claim~\ref{thm2:claim:lower-bound1}]
Let
	\begin{align*}
		A(\bx^{t+1}_{i}) &:= \lambda R(\bx^{t+1}_{i},\bc^t_{i})+\left\langle \nabla f_i(\bx^t_{i}), \bx^{t+1}_{i}-\bx^t_{i}\right\rangle 
		+ \left\langle \nabla_{\bx^t_{i}} f_i(\widetilde{Q}_{\bc^t_{i}}(\bx^t_{i})), \bx^{t+1}_{i}-\bx^t_{i}\right\rangle \notag \\
		&\hspace{3cm} +\left\langle \lambda_p(\bx^t_{i}-\bw^{t}_{i}), \bx^{t+1}_{i}-\bx^t_{i}\right\rangle + \frac{1}{2\eta_1}\|\bx^{t+1}_{i}-\bx^t_{i}\|^2 \notag \\
		A(\bx^t_{i}) &:= \lambda R(\bx^t_{i},\bc^t_{i}).
	\end{align*} 
	Then $A(\bx^{t+1}_{i})\leq A(\bx^t_{i})$.
\end{claim*}}
\begin{proof}
	Let $A(\bx)$ denote the expression inside the $\arg\min$ in \eqref{app:lower-bounding-interim1} and we know that \eqref{app:lower-bounding-interim1} is minimized when $\bx=\bx^{t+1}_{i}$. So we have $A(\bx^{t+1}_{i})\leq A(\bx^t_{i})$. This proves the claim.
\end{proof}
\begin{claim*}[Restating Claim~\ref{thm2:claim lqcsmoothness}]
 $(GL_{Q_2}+G_{Q_2}LL_{Q_2})$-smooth with respect to $\bc$.
\end{claim*}
\begin{proof}
	Applying the Claim \ref{claim: lqcclaim} to each client separately gives that $f_i(\widetilde{Q}_\bc(\bx))$ is $(GL_{Q_2}+G_{Q_2}LL_{Q_2})$-smooth.
\end{proof}
\begin{claim*}[Restating Claim~\ref{thm2:claim:lower-bound2}]
Let
	\begin{align*}
		B(\bc^{t+1}_{i}) &:= \lambda R(\bx^{t+1}_{i},\bc^{t+1}_{i}) + \left\langle \nabla_{\bc^t_{i}} f_i(\widetilde{Q}_{\bc^t_{i}}(\bx^{t+1}_{i})), \bc^{t+1}_{i}-\bc^t_{i}\right\rangle \notag  + \frac{1}{2\eta_2}\|\bc^{t+1}_{i}-\bc^t_{i}\|^2 \notag \\
		B(\bc^t_{i}) &:= \lambda R(\bx^{t+1}_{i},\bc^t_{i}).
	\end{align*} 
	Then $B(\bc^{t+1}_{i})\leq B(\bc^t_{i})$.
\end{claim*}
\begin{proof}
	Note that the update rule for $\bc$ is:	
	Let $B(\bc)$ denote the expression inside the $\arg\min$ in \eqref{app:lower-bounding-interim2} and we know that \eqref{app:lower-bounding-interim2} is minimized when $\bc=\bc^{t+1}_{i}$. So we have $B(\bc^{t+1}_{i})\leq B(\bc^t_{i})$. This proves the claim.
\end{proof}
\textit{Obtaining \eqref{thm2:algebra for decrease of w}. Note, we have $F_i(\bx.\bc,\bw)$ is $\lambda_p$-smooth w.r.t $\bw$. Then,} 
\begin{align*}
	F_i(\bx^{t+1}_{i},\bc^{t+1}_{i},\bw^{t+1}) &\leq 	F_i(\bx^{t+1}_{i},\bc^{t+1}_{i},\bw^{t}) + \left\langle \nabla_{\bw^{t}} F_i(\bx^{t+1}_{i},\bc^{t+1}_{i},\bw^{t}), \bw^{t+1}-\bw^{t} \right\rangle + \frac{\lambda_p}{2}\|\bw^{t+1}-\bw^{t}\|^2 \\
	&= F_i(\bx^{t+1}_{i},\bc^{t+1}_{i},\bw^{t}) + \left\langle \nabla_{\bw^{t}} F_i(\bx^{t+1}_{i},\bc^{t+1}_{i},\bw^{t}), \eta_3 \bg^{t} \right\rangle + \frac{\lambda_p}{2}\|\eta_3 \bg^{t}\|^2 \\
	& =	F_i(\bx^{t+1}_{i},\bc^{t+1}_{i},\bw^{t}) \\& \quad -  \eta_3 \left\langle \nabla_{\bw^{t}} F_i(\bx^{t+1}_{i},\bc^{t+1}_{i},\bw^{t}), \bg^{t} - \nabla_{\bw^{t}} F_i(\bx^{t+1}_{i},\bc^{t+1}_{i},\bw^{t}) + \nabla_{\bw^{t}} F_i(\bx^{t+1}_{i},\bc^{t+1}_{i},\bw^{t}) \right\rangle  \\& \quad + \frac{\lambda_p}{2}\eta_3^2\|\bg^{t} - \nabla_{\bw^{t}} F_i(\bx^{t+1}_{i},\bc^{t+1}_{i},\bw^{t}) + \nabla_{\bw^{t}} F_i(\bx^{t+1}_{i},\bc^{t+1}_{i},\bw^{t})\|^2 \\ 
	&=	F_i(\bx^{t+1}_{i},\bc^{t+1}_{i},\bw^{t}) - \eta_3\Big\|\nabla_{\bw^{t}} F_i(\bx^{t+1}_{i},\bc^{t+1}_{i},\bw^{t})\Big\|^2 \\& \quad -  \eta_3 \left\langle \nabla_{\bw^{t}} F_i(\bx^{t+1}_{i},\bc^{t+1}_{i},\bw^{t}), \bg^{t} - \nabla_{\bw^{t}} F_i(\bx^{t+1}_{i},\bc^{t+1}_{i},\bw^{t}) \right\rangle  \\& \quad + \frac{\lambda_p}{2}\eta_3^2\|\bg^{t} - \nabla_{\bw^{t}} F_i(\bx^{t+1}_{i},\bc^{t+1}_{i},\bw^{t}) + \nabla_{\bw^{t}} F_i(\bx^{t+1}_{i},\bc^{t+1}_{i},\bw^{t})\|^2 \\ 
	& \leq F_i(\bx^{t+1}_{i},\bc^{t+1}_{i},\bw^{t}) - \eta_3\Big\|\nabla_{\bw^{t}} F_i(\bx^{t+1}_{i},\bc^{t+1}_{i},\bw^{t})\Big\|^2 + \frac{\eta_3}{2}\Big\|\nabla_{\bw^{t}} F_i(\bx^{t+1}_{i},\bc^{t+1}_{i},\bw^{t})\Big\|^2 \\& \quad + \frac{\eta_3}{2}\Big\|\bg^{t} - \nabla_{\bw^{t}} F_i(\bx^{t+1}_{i},\bc^{t+1}_{i},\bw^{t})\Big\|^2 + \lambda_p \eta_3^2 \Big\|\bg^{t} - \nabla_{\bw^{t}} F_i(\bx^{t+1}_{i},\bc^{t+1}_{i},\bw^{t})\Big\|^2 \\& \quad + \lambda_p \eta_3^2 \Big\|\nabla_{\bw^{t}} F_i(\bx^{t+1}_{i},\bc^{t+1}_{i},\bw^{t})\Big\|^2 
	\\ & = F_i(\bx^{t+1}_{i},\bc^{t+1}_{i},\bw^{t}) - (\frac{\eta_3}{2}-\lambda_p\eta_3^2)\Big\|\nabla_{\bw^{t}} F_i(\bx^{t+1}_{i},\bc^{t+1}_{i},\bw^{t})\Big\|^2 \\& \quad + (\frac{\eta_3}{2}+\lambda_p\eta_3^2)\Big\|\bg^{t} - \nabla_{\bw^{t}_{i}} F_i(\bx^{t+1}_{i},\bc^{t+1}_{i},\bw^{t}_{i}) + \nabla_{\bw^{t}_{i}} F_i(\bx^{t+1}_{i},\bc^{t+1}_{i},\bw^{t}_{i}) \\ & \quad - \nabla_{\bw^{t}} F_i(\bx^{t+1}_{i},\bc^{t+1}_{i},\bw^{t})\Big\|^2  \\
	& \leq F_i(\bx^{t+1}_{i},\bc^{t+1}_{i},\bw^{t}) - (\frac{\eta_3}{2}-\lambda_p\eta_3^2)\Big\|\nabla_{\bw^{t}} F_i(\bx^{t+1}_{i},\bc^{t+1}_{i},\bw^{t})\Big\|^2 \\& \quad + (\eta_3+2\lambda_p\eta_3^2)\Big\|\bg^{t} - \nabla_{\bw^{t}_{i}} F_i(\bx^{t+1}_{i},\bc^{t+1}_{i},\bw^{t}_{i})\Big\|^2 + (\eta_3+2\lambda_p\eta_3^2)\Big\|\nabla_{\bw^{t}_{i}} F_i(\bx^{t+1}_{i},\bc^{t+1}_{i},\bw^{t}_{i}) \\ & \quad - \nabla_{\bw^{t}} F_i(\bx^{t+1}_{i},\bc^{t+1}_{i},\bw^{t})\Big\|^2 \\
	& \leq F_i(\bx^{t+1}_{i},\bc^{t+1}_{i},\bw^{t}) - (\frac{\eta_3}{2}-\lambda_p\eta_3^2)\Big\|\nabla_{\bw^{t}} F_i(\bx^{t+1}_{i},\bc^{t+1}_{i},\bw^{t})\Big\|^2 \\& \quad + (\eta_3+2\lambda_p\eta_3^2)\Big\|\bg^{t} - \nabla_{\bw^{t}_{i}} F_i(\bx^{t+1}_{i},\bc^{t+1}_{i},\bw^{t}_{i})\Big\|^2 + (\eta_3+2\lambda_p\eta_3^2)\lambda_p^2\Big\|\bw^{t}_{i}-\bw^{t}\Big\|^2
\end{align*}
Hence, we obtain the bound in \eqref{thm2:algebra for decrease of w}.

\subsection{Proof of Lemma~\ref{thm2:lemma1} and Corollary~\ref{thm2:corollary diversity}}

Let us restate and prove Lemma~\ref{thm2:lemma1},
\begin{lemma*}[Restating  Lemma~\ref{thm2:lemma1}]
	\begin{align} \label{app:lemma1}
		\frac{1}{T}\sum_{t=0}^{T-1}\frac{1}{n}\sum_{i=1}^{n} \|\bw^{t}-\bw^{t}_{i}\| \leq 6 \tau^2\eta_3^2\kappa 
	\end{align}
\end{lemma*}
\begin{proof}
Let $t_c$ be the latest synchronization time before t. Define $\gamma_t = \frac{1}{n} \sum_{i=1}^{n}\|\bw^{t} - \bw^{t}_{i}\|^2$. Then:
\begin{align*}
\gamma_t &= \frac{1}{n} \sum_{i=1}^{n} \Big\|\bw^{tc}-\frac{\eta_3}{n} \sum_{j=t_c}^{t} \sum_{k=1}^{n} \nabla_{\bw^{j}_{k}} F_k(\bx^{j+1}_{k},\bc^{j+1}_{k},\bw^{j}_{k}-(\bw^{tc}-\eta_3 \sum_{j=t_c}^{t} \nabla_{\bw^{j}_{i}} F_i(\bx^{j+1}_{i},\bc^{j+1}_{i},\bw^{j}_{i})\Big\|^2 \\ 
&\stackrel{\text{(a)}}{\leq} \tau  \sum_{j=t_c}^t \frac{\eta_3^2}{n} \sum_{i=1}^n \Big\|\frac{1}{n} \sum_{k=1}^{n} \nabla_{\bw^{j}_{k}} F_k(\bx^{j+1}_{k},\bc^{j+1}_{k},\bw^{j}_{k}- \nabla_{\bw^{j}_{i}} F_i(\bx^{j+1}_{i},\bc^{j+1}_{i},\bw^{j}_{i}\Big\|^2 \tag{$\star1$} \label{app:lm1star1}\\
&= \tau  \sum_{j=t_c}^t \frac{\eta_3^2}{n} \sum_{i=1}^n \Big[  \Big\|\frac{1}{n} \sum_{k=1}^{n} \left(\nabla_{\bw^{j}_{k}} F_k(\bx^{j+1}_{k},\bc^{j+1}_{k},\bw^{j}_{k})-\nabla_{\bw^{j}} F_k(\bx^{j+1}_{k},\bc^{j+1}_{k},\bw^{j}) + \nabla_{\bw^{j}} F_k(\bx^{j+1}_{k},\bc^{j+1}_{k},\bw^{j}) \right) \\ 
&\hspace{2cm} -\nabla_{\bw^{j}} F_i(\bx^{j+1}_{k},\bc^{j+1}_{k},\bw^{j})+\nabla_{\bw^{j}} F_i(\bx^{j+1}_{k},\bc^{j+1}_{k},\bw^{j})- \nabla_{\bw^{j}_{i}} F_i(\bx^{j+1}_{k},\bc^{j+1}_{k},\bw^{j}_{i}\Big\|^2 \Big] \\ 
& \leq \tau  \sum_{j=t_c}^{t_c+\tau} 3\frac{\eta_3^2}{n}  \sum_{i=1}^n \Big[\Big\|\frac{1}{n} \sum_{k=1}^{n} \left( \nabla_{\bw^{j}_{k}} F_k(\bx^{j+1}_{k},\bc^{j+1}_{k},\bw^{j}_{k})-\nabla_{\bw^{j}} F_k(\bx^{j+1}_{k},\bc^{j+1}_{k},\bw^{j})\right)\Big\|^2\\ 
&\hspace{2cm} + \Big\|\frac{1}{n}\sum_{k=1}^{n} \nabla_{\bw^{j}} F_k(\bx^{j+1}_{k},\bc^{j+1}_{k},\bw^{j}) -\nabla_{\bw^{j}} F_i(\bx^{j+1}_{i},\bc^{j+1}_{i},\bw^{j})\Big\|^2 \\ 
&\hspace{3cm} + \Big\|\nabla_{\bw^{j}} F_i(\bx^{j+1}_{i},\bc^{j+1}_{i},\bw^{j})- \nabla_{\bw^{j}_{i}} F_i(\bx^{j+1}_{i},\bc^{j+1}_{i},\bw^{j}_{i}\Big\|^2\Big] \\ 
& \leq \tau  \sum_{j=t_c}^{t_c+\tau} 3\frac{\eta_3^2}{n}  \sum_{i=1}^n \Big[\frac{\lambda_p^2 n}{n^2}\sum_{k=1}^n\| \bw^{j}_{k}-\bw^{j}\|^2 + \kappa_i + \lambda_p^2\|\bw^j-\bw^{j}_{i}\|^2\Big] \\ &=
\tau  \sum_{j=t_c}^{t_c+\tau} 3\eta_3^2(\lambda_p^2\gamma_j+\kappa+\lambda_p^2 \gamma_j) \tag{$\star2$} \label{app:lm1star2}
\end{align*}
in (a) we use the facts that  $\|\sum_{i=1}^K a_i\|^2 \leq K \sum_{i=1}^K \|a_i\|^2$, $t\leq \tau + t_c$ and that we are summing over non-negative terms. As a result, we have:
\begin{align*}
\gamma_t &\leq \tau  \sum_{j=t_c}^{t_c+\tau} 3\eta_3^2(2\lambda_p^2\gamma_j+\kappa)\\
\Longrightarrow \sum_{t=t_c}^{t_c+\tau} \gamma_t &\leq \sum_{t=t_c}^{t_c+\tau}  \sum_{j=t_c}^{t_c+\tau} 3\tau\eta_3^2(2\lambda_p^2\gamma_j+\kappa) = 6\tau^2\eta_3^2\lambda_p^2 \sum_{j=t_c}^{t_c+\tau} \gamma_j + 3\tau^3 \eta_3^2\kappa
\end{align*}
Let us choose $\eta_3$ such that $6\tau^2\eta_3^2\lambda_p^2 \leq \frac{1}{2} \Leftrightarrow \eta_3 \leq \sqrt{\frac{1}{12\tau^2\lambda_p^2}} $, sum over all syncronization times, and divide both sides by $T$:
\begin{align*} 
&\frac{1}{T} \sum_{t=0}^{T-1} \gamma_t \leq \frac{1}{2} \sum_{j=0}^{T-1} \gamma_j + 3\tau^3 \eta_3^2\kappa \\ \Longrightarrow &\frac{1}{T} \sum_{t=0}^{T-1} \gamma_t \leq  6\tau^2\eta_3^2\kappa
\end{align*}
\end{proof}
Let us restate and prove Corollary~\ref{thm2:corollary diversity}.
\begin{corollary*}[Restating Corollary~\ref{thm2:corollary diversity}]
Recall, $\bg^{t} = \frac{1}{n} \sum_{i=1}^n \nabla_{\bw^{t}} F_i(\bx^{t+1}_{i},\bc^{t+1}_{i},\bw^{t}_{i})$. Then, we have:
\begin{align*}
	\frac{1}{T}\sum_{t=0}^{T-1} \frac{1}{n} \sum_{i=1}^n \Big\|\bg^{t} - \nabla_{\bw^{t}_{i}} F_i(\bx^{t+1}_{i},\bc^{t+1}_{i},\bw^{t}_{i})\Big\|^2 \leq 36\lambda_p^2\tau^2\eta_3^2\kappa+3\kappa
\end{align*}
\end{corollary*}
\begin{proof}

 From \eqref{app:lm1star1} $\leq$ \eqref{app:lm1star2} in the proof of \eqref{thm2:lemma1}, we have:

\begin{align*}
	\sum_{t=t_c}^{t_c+\tau} \frac{1}{n} \sum_{i=1}^n \Big\|\bg^{t} - \nabla_{\bw^{t}_{i}} F_i(\bx^{t+1}_{i},\bc^{t+1}_{i},\bw^{t}_{i})\Big\|^2 \leq \sum_{j=t_c}^{t_c+\tau} 3(2\lambda_p^2\gamma_j+\kappa)
\end{align*}

Summing over all $t_c$ and dividing by $T$:

\begin{align*}
\frac{1}{T}\sum_{t=0}^{T-1} \frac{1}{n} \sum_{i=1}^n \Big\|\bg^{t} - \nabla_{\bw^{t}_{i}} F_i(\bx^{t+1}_{i},\bc^{t+1}_{i},\bw^{t}_{i})\Big\|^2 & \leq 6\lambda_p^2 \frac{1}{T} \sum_{t=0}^{T-1} \gamma_t + 3\kappa\\
& \stackrel{\text{(a)}}{\leq} 36\lambda_p^2 \tau^2 \eta_3^2 \kappa + 3 \kappa
\end{align*}
where (a) is from \eqref{thm2:lemma1}.
\end{proof}

\subsection{Choice of $\eta_3$}

Note that, in Lemma~\ref{thm2:lemma1} we chose $\eta_3$ such that $6\tau^2\eta_3^2\lambda_p^2 \leq \frac{1}{2} \Leftrightarrow \eta_3 \leq \frac{1}{\sqrt{12}\lambda_p\tau}$. Now, we further introduce upper bounds on $\eta_3$. 
\begin{itemize}
	\item We can choose $\eta_3$ small enough so that $L_{\min}=\eta_3-2\lambda_p\eta_3^2$.
	\item We can choose $\eta_3$ small enough so that $\eta_3-2\lambda_p\eta_3^2 \geq \frac{\eta_3}{2}$. This is equivalent to choosing $ \eta_3 \leq \frac{1}{4\lambda_p}$.
\end{itemize}

These two choices implies $L_{\min} \geq \frac{\eta_3}{2}$. 

In the end, we have 2 critical constraints on $\eta_3, \{\eta_3:\eta_3 \leq \frac{1}{\sqrt{12}\lambda_p\tau},\eta_3 \leq \frac{1}{4\lambda_p}\} $ . Then, let $\{\eta_3:\eta_3 \leq \frac{1}{4\lambda_p\tau}\}$. Moreover, assuming $\tau \leq \sqrt{T}$ we can take $\eta_3 = \frac{1}{4\lambda_p\sqrt{T}} $ this choice clearly satisfies the above constraints.

From \eqref{thm2:res1} we have,
\begin{align*} 
	\nonumber \frac{1}{T}\sum_{t=0}^{T-1}\frac{1}{n}\sum_{i=1}^{n} \Big\|\bG^{t}_{i}\Big\|^2 & \leq 36\frac{L_{\max}^2}{L_{\min}} \Big[6\tau^2\eta_3^2\kappa(\frac{\lambda_p}{2}+7\eta_3\lambda_p^2+14\eta_3^2\lambda_p^3)+3\eta_3\kappa+6\lambda_p\eta_3^2+ \frac{\Delta_F}{T}\Big] + 12\lambda_p^2\tau^2\eta_3^2\kappa\\
	& \leq 72\frac{L_{\max}^2}{\eta_3} \Big[6\tau^2\eta_3^2\kappa(\frac{\lambda_p}{2}+7\eta_3\lambda_p^2+14\eta_3^2\lambda_p^3)+3\eta_3\kappa+6\lambda_p\eta_3^2+ \frac{\Delta_F}{T}\Big] + 12\lambda_p^2\tau^2\eta_3^2\kappa \\
	& =   72L_{\max}^2 \Big[6\tau^2\eta_3\kappa(\frac{\lambda_p}{2}+7\eta_3\lambda_p^2+14\eta_3^2\lambda_p^3)+3\kappa+6\lambda_p\eta_3+ \frac{\Delta_F}{\eta_3 T}\Big] + 12\lambda_p^2\tau^2\eta_3^2\kappa \\
	\text{Now, we plug in $\eta_3 = \frac{1}{4\lambda_p\sqrt{T}}$:} \\
	& = 72(L_{\max})^2 \Big[ \frac{3}{4}\tau\kappa\frac{1}{\sqrt{T}}+\frac{21}{8}
	\tau^2\kappa\frac{1}{T} + \frac{21}{4}\tau^2\kappa\frac{1}{T^{\frac{3}{2}}}+3\kappa + \frac{3}{2}\frac{1}{\sqrt{T}}+\frac{4\lambda_p\Delta_F}{\sqrt{T}}\Big]+\frac{3}{4}\tau^2\kappa^2\frac{1}{T}\\
	& =  \frac{54L_{\max}^2\tau\kappa+108L_{\max}^2+288L_{\max}^2\lambda_p \Delta_F}{\sqrt{T}} + \frac{189L_{\max}^2\tau^2\kappa+\frac{3}{4}\tau^2\kappa^2}{T} + \frac{378L_{\max}^2\tau^2\kappa}{T^{\frac{3}{2}}} \\
	& \quad + 216L_{\max}^2\kappa
\end{align*}
This gives us the bound in Theorem~\ref{thm:personalized}.
}
%%------------------------------------------------------------------

\label{submission}

\newpage
\onecolumn

\end{document}